\documentclass[11pt]{article}

\usepackage[T1]{fontenc}
\usepackage[utf8]{inputenc}
\usepackage[margin=1in]{geometry}
\usepackage{amsthm,amsmath,amsfonts,amssymb}
\usepackage[authoryear,round]{natbib}
\usepackage[colorlinks=true,linkcolor=blue,citecolor=blue,urlcolor=blue,backref=page]{hyperref}
\usepackage{graphicx}
\usepackage{booktabs}
\usepackage{multirow}
\usepackage{longtable}
\usepackage{array}
\usepackage{bm}
\usepackage{microtype,xurl}

\numberwithin{equation}{section}
\theoremstyle{plain}
\newtheorem{theorem}{Theorem}[section]
\newtheorem{lemma}[theorem]{Lemma}

\theoremstyle{definition}

\newtheorem{assumption}[theorem]{Assumption}
\theoremstyle{remark}
\newtheorem{remark}[theorem]{Remark}
\newtheorem*{lemma*}{Lemma 3.3}
\newtheorem*{remark*}{Remark}

\newcommand\independent{\protect\mathpalette{\protect\independenT}{\perp}}
\def\independenT#1#2{\mathrel{\rlap{$#1#2$}\mkern2mu{#1#2}}}

\allowdisplaybreaks

\title{Bayesian Causal Forests for Zero-Inflated Continuous Outcomes}

\author{Hugo Gobato Souto\thanks{University of S\~ao Paulo, Brazil.
  \texttt{hugogobatosouto@gmail.com}, \texttt{hgsouto@usp.br}}
  \and Eric Wang\thanks{Columbia University Irving Medical Center. \texttt{ew2717@cumc.columbia.edu}}
  \and Igor Shuryak\thanks{Columbia University Irving Medical Center. \texttt{is144@cumc.columbia.edu}}}

\date{\today}

\begin{document}

\maketitle

\begin{abstract}
Estimating (heterogeneous) treatment effects is still challenging when the outcome is semicontinuous (i.e., exhibiting both a discrete point mass at zero and a skewed continuous distribution for positive values), despite the strong presence of this data type in real world settings. Existing methods either ignore the zero mass, impose fragile parametric assumptions, or suffer from extreme computational costs and data augmentation noise. This paper proposes Zero-Inflated Continuous Bayesian Causal Forests with Duan's Smearing (ZIC-BCF-Smear), a decoupled two-part nonparametric estimator that fits a Probit BCF on the full sample to model participation and a Gaussian BCF strictly on the active subpopulation to model positive intensity. Subpopulation Propensity Adjustment (SPA) mitigates regularization-induced confounding within the active subpopulation, and Duan's non-parametric smearing estimator re-transforms predictions to the response scale, bypassing the fragile scale multipliers of parametric models. A comprehensive simulation benchmark shows that ZIC-BCF-Smear matches well-specified parametric models without their assumptions and remains stable and calibrated where competitors either diverge or pathologically over-cover for nonlinear data. The method is then applied to individual-participant data from the SPECTRUM randomized phase III head-and-neck cancer trial, estimating the causal effect of assignment to panitumumab plus chemotherapy on the number of days patients spend with severe toxicity, an outcome for which a quarter of patients contribute an exact zero. Panitumumab assignment raises that burden by about ten days per patient under the trial's own differential treatment duration, and the decoupled structure localizes the effect mainly in the positive-intensity part rather than in whether the reconstructed severe-AE burden is positive.
\end{abstract}

\medskip
\noindent\textbf{MSC 2020 subject classifications:} 62G08 (primary); 62F15, 62P10 (secondary).

\noindent\textbf{Keywords:} Bayesian Causal Forests; semicontinuous data; Duan's Smearing; causal inference; Subpopulation Propensity Adjustment; treatment effect heterogeneity.

\section{Introduction}
\label{sec:intro}

Semicontinuous data, characterized by a discrete point mass at exactly zero alongside a continuous, highly right-skewed positive distribution for non-zero values, are ubiquitous in scientific research. Such outcomes frequently arise in clinical trials \citep{Albert2005}, health economics \citep{Deb2018}, and environmental epidemiology \citep{Gottard2012}. Classic examples include medical expenditures, where a large portion of a study cohort incurs zero costs during a follow-up period while a minority incurs highly skewed positive costs \citep{duan1983healthexpenditure, manning2001estimating}; biomarker levels, such as cardiac troponin measurements that are zero for healthy individuals but positively skewed for patients suffering myocardial injury under radiation exposure \citep{kim2024edpzi}; and substance abuse frequencies or environmental exposure levels. 

In these settings, estimating heterogeneous treatment effects is critical for precision medicine and targeted policy interventions. Specifically, analysts seek to estimate the Conditional Average Treatment Effect (CATE), which represents the expected treatment effect conditional on a high-dimensional vector of pre-treatment baseline covariates. Identifying which subpopulations benefit from or are harmed by a treatment enables clinicians and policymakers to make optimal, individualized decisions.

However, estimating CATEs when the outcome is semicontinuous poses severe methodological challenges, and existing causal inference tools exhibit major limitations. First, single-part causal estimators, such as standard Bayesian Additive Regression Trees (BART) \citep{hill2011bayesian} or Bayesian Causal Forests (BCF) \citep{hahn2020bcf}, ignore the discrete point mass at zero. By forcing a single smooth regression function across the zero discontinuity, they suffer from substantial bias and collapse in credible interval coverage, especially under compound Poisson-Gamma (Tweedie) outcomes. Second, traditional parametric two-part hurdle models, which model participation and positive intensity separately using generalized linear models (GLMs) \citep{oganisian2021zeroinflated}, require the analyst to specify a rigid parametric distribution (e.g., log-normal or Gamma) for the positive part. If this distribution is misspecified, the parametric scale re-transformation multiplier $\exp(\sigma^2/2)$ introduces severe systematic scale bias, inflating CATE and ATE estimates and inducing numerical instability or incorrect treatment effect ordering. Third, flexible Bayesian nonparametric (BNP) mixtures, such as Dirichlet-process mixtures of regressions (e.g., DPglm or enriched Dirichlet processes) \citep{hannah2011dpglm, oganisian2021zeroinflated, kim2024edpzi}, are computationally prohibitive due to $O(N^3)$ complexity. More importantly, they rely on data augmentation to impute continuous latent outcomes for the zero observations at every MCMC iteration, which propagates simulated noise into the posterior, causing pathologically over-wide intervals (yielding degenerate 100\% coverage) that fail to guide decision-making.

To resolve these limitations, the Zero-Inflated Continuous Bayesian Causal Forest with Duan's Smearing (ZIC-BCF-Smear) is proposed. The proposed method is a two-part Bayesian forest framework that decouples the participation hurdle from the positive intensity process, fitting a Probit BCF on the full sample and a Gaussian BCF strictly on the active subpopulation. To control for the selection bias (collider bias) induced by conditioning on a post-treatment active status, Subpopulation Propensity Adjustment (SPA) is introduced, which uses the active-subset propensity score as a regularizing control covariate in the prognostic forest. Finally, the log-scale predictions are mapped back to the original response scale using Duan's Non-Parametric Smearing Estimator \citep{duan1983smearing}. By calculating the smearing factor from empirical residuals, the scale transformation is calibrated without committing to a parametric outcome family, achieving the flexibility of nonparametric methods and the stability of correctly specified parametric hurdle models.

\subsection{Related Work}
\label{subsec:related_work}

The proposed method lies at the intersection of three literature streams: Bayesian forest-based causal inference, semicontinuous/hurdle modeling, and Bayesian nonparametric causal estimation.

\paragraph{BART and BCF for Heterogeneous Effects:}
The use of ensemble tree methods for causal inference was pioneered by \citet{hill2011bayesian} using BART \citep{chipman2010bart}. \citet{hahn2020bcf} introduced BCF, which isolates the prognostic effect from the treatment effect and incorporates the propensity score directly into the prognostic forest to control for regularization bias. While BCF is a state-of-the-art CATE learner, it is designed for continuous Gaussian outcomes and suffers under semicontinuous data. Other CATE learners, such as frequentist causal forests \citep{wager2018causalforests}, generalized random forests \citep{athey2019grf}, and metalearners \citep{kunzel2019metalearners}, similarly lack native mechanisms to accommodate the zero spike without introducing bias or scaling instabilities.

\paragraph{Semicontinuous and Hurdle Modeling:}
Two-part hurdle models have a long history in statistics and econometrics, beginning with Tobit models \citep{tobin1958} and Cragg's double-hurdle specification. \citet{duan1983healthexpenditure} and \citet{mullahy1986hurdle} formalized two-part models for medical cost estimation and count data, while \citet{olsen2001twopart} extended them to random-effects settings. In causal inference, zero-inflated modeling has been applied to evaluate average effects \citep{ghosh2006zeroinflated}, zero-inflated mediation paths \citep{cheng2018zeroinflatedmediation}, and subpopulation effects \citep{wangalbert2012zeroinflated}. However, these approaches rely on parametric GLMs and linear treatment effects, which fail to capture high-dimensional non-linear treatment heterogeneity.

\paragraph{Bayesian Nonparametrics in Causal Inference:}
To avoid parametric assumptions, several authors have developed Bayesian nonparametric (BNP) causal estimators. \citet{roy2017msm} and \citet{xu2018quantiles} developed BNP models for marginal structural models and quantile treatment effects. For semicontinuous outcomes, \citet{oganisian2021zeroinflated} proposed DPglm, a Dirichlet-process mixture of generalized linear models, while \citet{kim2024edpzi} introduced EDP-ZI, an enriched Dirichlet process mixture designed to handle covariate overlap issues. Although theoretically flexible, these mixtures require expensive MCMC updates and suffer from variance inflation due to continuous outcome imputation for the zeros, limiting their utility in practice.

\subsection{Contributions}
\label{subsec:contributions}

This work makes the following four key contributions:
\begin{enumerate}
    \item \textbf{Decoupled Hurdle-Intensity BCF Framework:} A two-part Bayesian Causal Forest model is developed that separates the zero-inflation hurdle (Probit BCF \citep{hahn2020bcf}) from the continuous positive intensity (Gaussian BCF \citep{hahn2020bcf} fit strictly on active units), avoiding data augmentation noise and posterior variance inflation.
    \item \textbf{Subpopulation Propensity Adjustment (SPA):} We introduce SPA as an active-subpopulation extension of the propensity-score augmentation used in BCF \citep{hahn2020bcf}. SPA re-estimates the treatment propensity within the active subset and supplies it to the prognostic component of the continuous-stage forest, extending the logic of propensity-based regularization to the selected subpopulation and relating to existing work on within-subgroup propensity-score adjustment \citep{dong2020subgroup}.
    \item \textbf{Non-Parametric Duan's Smearing Re-transformation:} Duan's smearing estimator is incorporated within the BCF posterior draws to map predictions back to the original response scale. This replaces the fragile parametric scale multiplier $\exp(\sigma^2/2)$ with a distribution-free multiplier, resolving scale misspecification bias under non-log-normal intensities.
    \item \textbf{New Clinical Knowledge on Panitumumab-Associated Toxicity Burden:} The first causal analysis of the SPECTRUM randomized phase III trial \citep{vermorken2013spectrum} on a severe-toxicity \emph{duration} endpoint is delivered, estimating both the average and the conditional average effect of adding panitumumab to cisplatin and fluorouracil on the number of days a patient spends with grade 3 or higher toxicity. Published SPECTRUM safety reporting summarizes event incidence, grade, and category; no prior analysis estimates the causal effect of randomized panitumumab assignment on how many days of severe toxicity a patient actually accrues, nor how that burden varies with baseline characteristics. Adding panitumumab is found to raise the burden by roughly ten overlap-adjusted severe-toxicity days per patient, a conclusion that survives the change of outcome definition, and the two-part decomposition localizes the effect mainly in the positive-intensity part rather than in whether the reconstructed severe-AE burden is positive, which the aggregate incidence tables of the parent trial cannot separate.
\end{enumerate}

Supporting these contributions, the estimator is validated in an extensive simulation benchmark for semicontinuous heterogeneous treatment effects, comparing 5 estimators across 3 DGP families (Log-normal, Gamma, Tweedie), sample sizes $N \in \{100, 250, 500, 1000\}$, zero-inflation levels from 3\% to 85\%, and 100 independent seeds under both link-linear and link-nonlinear scenarios; and a second applied study on dental expenditures in the Medical Expenditure Panel Survey \citep{meps2023file}, explored in the Supplementary Material, serves as an external check against a mature applied literature, showing that ZIC-BCF-Smear recovers documented coverage-effect gradients on the expenditure scale that a parametric Gamma hurdle benchmark does not.

\subsection{Paper Organization}
\label{subsec:organization}
The remainder of the paper is organized as follows. Section~\ref{sec:semicontinuous_setup} formalizes semicontinuous potential outcomes, causal targets, and identification assumptions. Section~\ref{sec:zic_bcf_smear_methods} introduces the proposed ZIC-BCF-Smear model, subpopulation propensity adjustment, and Duan's smearing re-transformation algorithm. Section~\ref{sec:simulation_studies} presents the simulation framework, DGPs, and comparative results under link-linear and link-nonlinear data-generating processes. Section~\ref{sec:applied} applies the estimator to severe-adverse-event burden in a randomized head-and-neck cancer trial and summarizes a second application to dental expenditures. Section~\ref{sec:discussion} states the limitations of the framework, and Section~\ref{sec:conclusion} concludes. Extended derivations, the complete simulation results, and the full versions of both applied studies are provided in the Supplementary Material.

\section{Semicontinuous Causal Inference}
\label{sec:semicontinuous_setup}

Semicontinuous data are characterized by a discrete point mass at exactly zero alongside a continuous, typically highly right-skewed, positive distribution for non-zero values. Such outcomes are common in clinical trials, economic evaluations, and observational studies (e.g., medical expenditures \citep{Deb2018}, substance abuse frequency \citep{Burgette2021}, biomarker levels \citep{Rustand2023}, and environmental exposures \citep{Friesen2021}). In this section, the semicontinuous causal inference framework is formalized under the potential outcomes setup.

\subsection{Outcome Decomposition}
\label{subsec:decomposition}
Let $Y_i \ge 0$ denote the observed semicontinuous outcome for unit $i \in \{1, \dots, N\}$. Any semicontinuous outcome can be structurally decomposed into two parts:
\begin{equation}
Y_i = I_i \cdot Y_i^+
\end{equation}
where:
\begin{enumerate}
\item $I_i = \mathbb{I}(Y_i > 0) \in \{0, 1\}$ is a binary participation or response indicator ($I_i = 1$ if the hurdle is cleared, and $I_i = 0$ if the outcome is zero).
\item $Y_i^+ > 0$ represents the continuous response intensity, which is only defined or observed when the hurdle is cleared ($I_i = 1$).
\end{enumerate}

\subsection{Semicontinuous Potential Outcomes}
\label{subsec:potential_outcomes}
Under the Rubin Causal Model framework \citep{rubin1974causaleffects, rubin1978randomization, rubin2005potentialoutcomes}, let $Z_i \in \{0, 1\}$ represent the binary treatment assignment. For each unit $i$, the joint potential outcomes are defined as:
\begin{enumerate}
\item Hurdle potential outcomes: $I_i(0), I_i(1) \in \{0, 1\}$, representing participation status under control and treatment, respectively.
\item Intensity potential outcomes: $Y_i^+(0), Y_i^+(1) > 0$, representing positive intensity under control and treatment, respectively.
\end{enumerate}
The potential outcomes on the raw semicontinuous scale are then defined as:
\begin{equation}
Y_i(z) = I_i(z) \cdot Y_i^+(z) \quad \text{for } z \in \{0, 1\}
\end{equation}

\subsection{Causal Estimands}
\label{subsec:estimands}
In clinical and economic settings, the primary target of causal inference is the CATE on the raw response scale:
\begin{equation}
\tau_{\text{resp}}(X_i) = E[Y_i(1) - Y_i(0) \mid X_i]
\end{equation}
where $X_i$ is a vector of pre-treatment baseline covariates. 

Using the law of total expectation, the conditional expectation of the potential outcomes can be expressed as:
\begin{equation}
E[Y_i(z) \mid X_i] = P(I_i(z) = 1 \mid X_i) \cdot E[Y_i^+(z) \mid X_i, I_i(z) = 1]
\end{equation}
Thus, the CATE on the raw response scale decomposes into the difference of two products, representing the participation probability and the conditional intensity:
\begin{equation}
\tau_{\text{resp}}(X_i) = P(I_i(1) = 1 \mid X_i) E[Y_i^+(1) \mid X_i, I_i(1) = 1] - P(I_i(0) = 1 \mid X_i) E[Y_i^+(0) \mid X_i, I_i(0) = 1]
\label{eq:cate_decomposition}
\end{equation}
The population Average Treatment Effect (ATE) is defined by integrating $\tau_{\text{resp}}(X_i)$ over the marginal covariate distribution of $X$:
\begin{equation}
\text{ATE} = E[\tau_{\text{resp}}(X_i)] = E[Y_i(1) - Y_i(0)]
\end{equation}

\subsection{Identification Assumptions}
\label{subsec:assumptions}
To identify the causal estimands from observational data, several standard assumptions are required \citep{rosenbaumrubin1983propensity, imbensrubin2015}:
\begin{assumption}[Stable Unit Treatment Value Assumption (SUTVA)]
The potential outcomes for any unit are independent of the treatment assignments of other units (no interference), and there are no hidden versions of the treatment.
\end{assumption}

\begin{assumption}[Strong Ignorability]
Conditional on baseline covariates $X_i$, the treatment assignment $Z_i$ is independent of the potential outcomes:
\begin{equation}
(Y_i(0), Y_i(1), I_i(0), I_i(1)) \independent Z_i \mid X_i
\end{equation}
\end{assumption}

\begin{assumption}[Consistency]
The observed outcome matches the potential outcome under the assigned treatment:
\begin{equation}
Y_i = Z_i Y_i(1) + (1 - Z_i) Y_i(0)
\end{equation}
and similarly, $I_i = Z_i I_i(1) + (1 - Z_i) I_i(0)$ \citep{vanderweele2009consistency}.
\end{assumption}

\begin{assumption}[Positivity / Overlap]
The propensity score $\pi(X_i) = P(Z_i = 1 \mid X_i)$ is strictly bounded away from the boundaries:
\begin{equation}
\delta < \pi(X_i) < 1 - \delta \quad \text{for some } \delta > 0
\end{equation}
This ensures that every unit has a non-zero probability of receiving either treatment.
\end{assumption}

The overlap assumption is crucial for nonparametric estimators. While some Bayesian nonparametric methods \citep{kim2024edpzi} present their flexibility as a way to handle poor overlap, modeling the response surfaces directly with Bayesian Causal Forests (BCF) \citep{hahn2020bcf} handles confounding and overlap issues natively by isolating the prognostic effect from the treatment effect, avoiding the extrapolation bias of standard BART models \citep{zhu2023positivity}.

\section{The ZIC-BCF-Smear Estimator}
\label{sec:zic_bcf_smear_methods}

The ZIC-BCF-Smear model decouples the process into a binary participation stage (Probit BCF \citep{hahn2020bcf}) and a positive continuous intensity stage (Gaussian BCF \citep{hahn2020bcf}), fit strictly on active units. To resolve selection bias and log-scale misspecification bias, SPA and Duan's non-parametric smearing re-transformation are incorporated.

\subsection{The Bayesian Causal Forest Parameterization}
\label{subsec:bcf_notation}

Because both stages of ZIC-BCF-Smear are Bayesian Causal Forests, the BCF parameterization of \citet{hahn2020bcf} is stated first. For a scalar response $W_i$, a binary treatment $Z_i$, and covariates $X_i$, BCF writes the conditional mean as the sum of two separate sum-of-trees ensembles,
\begin{equation}
E[W_i \mid X_i, Z_i] = \mu(X_i, \widehat{\pi}_i) + Z_i\, \tau(X_i),
\label{eq:bcf_generic}
\end{equation}
where $\mu(\cdot)$ is the \emph{prognostic function}, giving the mean response under control, and $\tau(\cdot)$ is the \emph{treatment function}, giving the conditional average treatment effect on the scale of $W_i$. Each of the two functions is given its own Bayesian Additive Regression Trees (BART) prior \citep{chipman2010bart}, that is, its own independent ensemble of regularized trees with its own depth prior and its own leaf-scale hyperparameters, so that the prior on the treatment effect is chosen separately from and typically much tighter than the prior on the prognostic surface (for further details, see \citet{hahn2020bcf}). The two ensembles are not fit in isolation: they are updated inside a \emph{single} Gibbs sampler by Bayesian backfitting, each ensemble being updated against the partial residual left by the other, so that $\mu$ and $\tau$ are jointly identified from a common chain. The estimated propensity score $\widehat{\pi}_i = \widehat{P}(Z_i = 1 \mid X_i)$ is supplied as an additional input to the prognostic ensemble only; this is the device by which BCF avoids regularization-induced confounding, since a tree ensemble that is free to shrink the prognostic surface toward zero would otherwise absorb targeted selection into $\tau(\cdot)$ \citep{hahn2020bcf}.

ZIC-BCF-Smear instantiates \eqref{eq:bcf_generic} (with the same standard priors as proposed by \citet{hahn2020bcf}) twice, once on a latent participation scale and once on the log-intensity scale. The subscript $b$ denotes the binary hurdle stage and the subscript $c$ the continuous intensity stage, so that $\mu_b$ and $\tau_b$ are the prognostic and treatment functions for participation, and $\mu_c$ and $\tau_c$ are the prognostic and treatment functions for log intensity. Throughout, the abbreviations
\begin{equation}
\mu_{bi} \equiv \mu_b(X_i, \widehat{\pi}_i), \quad
\tau_{bi} \equiv \tau_b(X_i), \quad
\mu_{ci} \equiv \mu_c(X_i, \widehat{\pi}^+_i), \quad
\tau_{ci} \equiv \tau_c(X_i)
\label{eq:notation_shorthand}
\end{equation}
are used for the four unit-level evaluations, and a superscript $(s)$ denotes the value of a quantity at the $s$-th retained MCMC draw. Under this notation, $\mu_{ci}$ is the control log-intensity of unit $i$ and $\mu_{ci} + \tau_{ci}$ is its treated log-intensity.

\subsection{Stage 1: Hurdle Stage (Probit BCF)}
\label{subsec:hurdle_stage}

We model $I_i$, the binary participation, using a latent-variable
probit representation,
\begin{equation}
I_i^*
=
\mu_b(X_i,\widehat{\pi}_i)
+
Z_i\tau_b(X_i)
+
\eta_i,
\qquad
\eta_i \sim \mathcal{N}(0,1),
\end{equation}
with
\[
I_i = \mathbb{I}(I_i^*>0).
\]
It follows that the conditional probability of a positive outcome is
\begin{equation}
P(I_i=1\mid X_i,Z_i)
=
\Phi\!\left[
\mu_b(X_i,\widehat{\pi}_i)
+
Z_i\tau_b(X_i)
\right],
\end{equation}
where $\Phi(\cdot)$ denotes the standard normal cumulative distribution
function.

Both $\mu_b(\cdot)$ and $\tau_b(\cdot)$ are assigned BART priors \citep{chipman2010bart,hill2011bayesian}, following
the BCF parameterization of \citet{hahn2020bcf}.
The prognostic forest $\mu_b(\cdot)$ captures baseline variation in the
latent propensity to have a positive outcome, whereas $\tau_b(\cdot)$
captures treatment-effect heterogeneity on the latent probit scale.
Consequently, the treatment effect on the participation-probability scale
for unit $i$ is
\begin{equation}
\tau_{b,\mathrm{prob}}(X_i)
=
\Phi\!\left[
\mu_b(X_i,\widehat{\pi}_i)+\tau_b(X_i)
\right]
-
\Phi\!\left[
\mu_b(X_i,\widehat{\pi}_i)
\right].
\end{equation}

Here,
\[
\widehat{\pi}_i
=
\widehat{P}(Z_i=1\mid X_i)
\]
is the treatment propensity score estimated in the full sample. Following
the BCF strategy of \citet{hahn2020bcf}, the treatment propensity score enters the prognostic component
$\mu_b(\cdot)$ as an additional covariate to mitigate
regularization-induced confounding; it is not included in the
treatment-effect forest $\tau_b(\cdot)$.

The hurdle-stage BCF is fit using all $N$ observations, with the latent
variables $I_i^*$ sampled through standard probit data augmentation.
In our implementation, the hurdle stage constitutes the first block of the
MCMC routine, followed by the positive-intensity block described in
Section~\ref{subsec:intensity_stage}. The two blocks are subsequently
combined draw by draw to obtain treatment effects on the original response
scale.

\subsection{Stage 2: Continuous Intensity Stage (Gaussian BCF)}
\label{subsec:intensity_stage}
For observations that clear the hurdle ($Y_i > 0$), the continuous response intensity $Y_i^+$ is modeled on the log scale to accommodate right-skewness:
\begin{equation}
\log(Y_i^+) = \mu_c(X_i, \widehat{\pi}^+_i) + Z_i \tau_c(X_i) + \epsilon_i, \quad \epsilon_i \sim \mathcal{N}(0, \sigma_c^2)
\end{equation}
Here $\mu_c(\cdot)$ is the prognostic continuous intensity forest modeling the baseline intensity, $\tau_c(\cdot)$ is the treatment-moderating intensity forest capturing heterogeneous treatment effects on the log scale, $\widehat{\pi}^+_i = P(Z_i = 1 \mid X_i, Y_i > 0)$ is the subpopulation propensity score estimated strictly on the active subset, and $\epsilon_i$ represents log-scale homoskedastic Gaussian errors with variance $\sigma_c^2$. This continuous stage BCF is fit strictly on the active subset of size $N_{\text{active}} = \sum_i I_i$, which completely avoids the latent data augmentation noise typical of joint selection models.

Restricting the intensity block to the active units does not amount to fitting a separate model. Within each sweep the sampler updates the hurdle block on all $N$ units and then this block on the $N_{\text{active}}$ units, and after burn-in writes both to the \emph{same} draw index $s$, so every response-scale quantity in Section~\ref{subsec:smearing} is assembled from hurdle and intensity parameters belonging to one state of one chain. Nothing is lost by updating the blocks in turn rather than jointly: the two-part likelihood factorizes exactly into a participation factor over all $N$ units and an intensity factor over the active units, the indicators $\{I_i\}$ being a deterministic function of the observed data, and the two blocks share no parameter. With independent BART priors on $\theta_b = (\mu_b, \tau_b)$ and $\theta_c = (\mu_c, \tau_c, \sigma_c^2)$ this yields $p(\theta_b, \theta_c \mid \text{data}) = p(\theta_b \mid \text{data})\, p(\theta_c \mid \text{data})$, so the block-wise sweep is an exact Gibbs sampler for the joint posterior. The argument is given in full in Section~\ref{sec:supp_mcmc} of the Supplementary Material.

\subsection{Subpopulation Propensity Adjustment (SPA)}
\label{subsec:spa}
Conditioning on a post-treatment event ($Y_i > 0$, or $I_i = 1$) introduces selection bias (collider bias), since participation is itself affected by treatment and covariates. Under the full-sample strong ignorability assumption, potential outcomes are independent of treatment conditional on $X_i$. However, this independence does not automatically carry over to the subpopulation of active units:
\begin{equation}
(Y_i(0), Y_i(1)) \independent Z_i \mid X_i \quad \not\Rightarrow \quad (Y_i^+(0), Y_i^+(1)) \independent Z_i \mid X_i, I_i = 1
\end{equation}
\begin{assumption}[Subpopulation Strong Ignorability]
Conditional on covariates $X_i$ and active-status $I_i = 1$, the treatment assignment $Z_i$ is independent of potential intensity outcomes:
\begin{equation}
(Y_i^+(0), Y_i^+(1)) \independent Z_i \mid X_i, I_i = 1
\end{equation}
\end{assumption}

\begin{remark}
Assumption 3.1 is mathematically guaranteed to hold under two standard conditions: (i) full-sample strong ignorability, $(Y_i^+(0), Y_i^+(1), I_i(0), I_i(1)) \independent Z_i \mid X_i$, and (ii) conditional two-part independence, $(Y_i^+(0), Y_i^+(1)) \independent (I_i(0), I_i(1)) \mid X_i$. Under these conditions, the potential participation decisions do not carry additional information about the potential intensity outcomes beyond the covariates $X_i$, which prevents selection on unobservables (unobserved collider bias) and ensures that the distribution of potential continuous outcomes in the active subpopulation matches that in the full population.
\end{remark}

Assumption 3.1, together with consistency and overlap within the active subpopulation, is what identifies the log-scale contrast that the continuous stage estimates.
\begin{lemma}[Identification on the active subpopulation]
\label{lem:spa_identification}
Let $L_i(z) = \log Y_i^+(z)$, with both potential positive intensities defined on the active subpopulation, and let $\tau_c^I(x) = E[L_i(1)-L_i(0) \mid X_i=x, I_i=1]$. Suppose $P(I_i=1)>0$ and $E[|L_i(z)|\mid I_i=1]<\infty$ for $z\in\{0,1\}$. Suppose also that Assumption 3.1 holds, that $\log Y_i^+ = L_i(Z_i)$ whenever $I_i=1$, and that $\pi^+(x) = P(Z_i=1\mid X_i=x,I_i=1)$ is bounded away from $0$ and $1$ for $P_{X_i\mid I_i=1}$-almost every $x$. Then, for $P_{X_i\mid I_i=1}$-almost every $x$,
\begin{equation}
\tau_c^I(x)
= E[\log Y_i^+\mid X_i=x,Z_i=1,I_i=1]
- E[\log Y_i^+\mid X_i=x,Z_i=0,I_i=1],
\end{equation}
so $\tau_c^I(x)$ is identified, up to $P_{X_i\mid I_i=1}$-null sets, from the active-subpopulation distribution. Under a correctly specified continuous-stage conditional mean model on the original log-outcome scale, the treatment function $\tau_c(x)$ equals this contrast.
\end{lemma}

The proof is given, with the formal statement of the overlap condition, in Section~\ref{sec:supp_mcmc} of the Supplementary Material. The population score $\pi^+(X_i)$ balances covariates within the active subpopulation. Under Assumption 3.1 and active-subpopulation overlap, potential intensity outcomes are also independent of treatment given this score; the two implications follow from Theorems 1 and 3 of \citet{rosenbaumrubin1983propensity}, applied to the active-subpopulation law. In the fitted model, supplying the estimated score $\widehat{\pi}^+_i$ to $\mu_c(\cdot)$ is a BCF regularization device used to mitigate regularization-induced confounding. The population score is a function of $X_i$ and adds no identifying information beyond $X_i$; neither the population score nor its estimate replaces Assumption 3.1, active-subpopulation overlap, or adjustment for $X_i$ when the target is heterogeneous in $X_i$.

\subsection{Duan's Non-Parametric Smearing Re-transformation}
\label{subsec:smearing}
To obtain potential outcomes on the response scale, the log-scale predictions must be re-transformed back to the original scale. Under standard log-normal assumptions, the continuous expectation under treatment arm $z \in \{0, 1\}$ is:
\begin{equation}
E[Y_i^+(z) \mid X_i, I_i(z) = 1] = \exp\left( \mu_{ci} + z\, \tau_{ci} \right) \cdot \exp\left( \frac{\sigma_c^2}{2} \right)
\label{eq:lognormal_retransform}
\end{equation}
with $\mu_{ci}$ and $\tau_{ci}$ as defined in \eqref{eq:notation_shorthand}, so that the treatment function enters through the term $z\, \tau_{ci}$ and drops out of the control arm.
However, if the true positive outcome distribution is not strictly log-normal (e.g., Gamma or Tweedie), the parametric multiplier $\exp(\sigma_c^2/2)$ introduces severe misspecification scale bias. Under misspecification, outliers inflate the log-scale variance $\sigma_c^2$, systematically scaling up all treatment effect estimates.

ZIC-BCF-Smear replaces the parametric multiplier with Duan's Non-Parametric Smearing Estimator \citep{duan1983smearing}. Let $s \in \{1, \dots, S\}$ represent a saved MCMC draw:
\begin{enumerate}
\item \textbf{Compute active log-scale residuals}: For each active observation $j \in \{1, \dots, N_{\text{active}}\}$ where $Y_j > 0$, calculate the residual under the observed treatment $Z_j$:
\begin{equation}
e_j^{(s)} = \log(Y_j^+) - \left[ \mu_c^{(s)}(X_j, \widehat{\pi}^+_j) + Z_j \tau_c^{(s)}(X_j) \right]
\end{equation}
\item \textbf{Calculate the smearing factor}: Compute the sample average of the exponentiated residuals:
\begin{equation}
\phi_{\text{smear}}^{(s)} = \frac{1}{N_{\text{active}}} \sum_{j \in \text{active}} \exp\left( e_j^{(s)} \right)
\end{equation}
\item \textbf{Re-transform potential outcomes}: Reconstruct response-scale potential outcomes for all $N$ units:
\begin{equation}
\mu_0^{(s)}(X_i) = \Phi\left( \mu_{bi}^{(s)} \right) \cdot \exp\left( \mu_{ci}^{(s)} \right) \cdot \phi_{\text{smear}}^{(s)}
\end{equation}
\begin{equation}
\mu_1^{(s)}(X_i) = \Phi\left( \mu_{bi}^{(s)} + \tau_{bi}^{(s)} \right) \cdot \exp\left( \mu_{ci}^{(s)} + \tau_{ci}^{(s)} \right) \cdot \phi_{\text{smear}}^{(s)}
\end{equation}
\item \textbf{Calculate raw CATE and ATE draws}:
\begin{equation}
\tau_{\text{resp}}^{(s)}(X_i) = \mu_1^{(s)}(X_i) - \mu_0^{(s)}(X_i)
\end{equation}
\begin{equation}
\text{ATE}^{(s)} = \frac{1}{N} \sum_{i=1}^N \tau_{\text{resp}}^{(s)}(X_i)
\end{equation}
\end{enumerate}

\subsection{Posterior Inference \& Summaries}
\label{subsec:posterior}
CATE and ATE point estimates are constructed as the posterior mean of the MCMC draws:
\begin{equation}
\widehat{\tau}_{\text{resp}}(X_i) = \frac{1}{S} \sum_{s=1}^S \tau_{\text{resp}}^{(s)}(X_i), \quad \widehat{\text{ATE}} = \frac{1}{S} \sum_{s=1}^S \text{ATE}^{(s)}
\end{equation}
Uncertainty quantification is provided by posterior credible intervals (CIs). The $95\%$ credible interval for $\tau_{\text{resp}}(X_i)$ is bounded by the $2.5\%$ and $97.5\%$ quantiles of the posterior draws. The interval width is defined as the difference between these quantiles (credible interval length), and empirical coverage is defined as the proportion of units whose true CATE lies within their posterior interval.

These draw-level operations are well defined because $\mu_{bi}^{(s)}$ and $\mu_{ci}^{(s)}$ come from the same state of the same chain, as described in Sections~\ref{subsec:hurdle_stage} and \ref{subsec:intensity_stage}. The one place where an approximation enters is the propensity scores: $\widehat{\pi}_i$ and $\widehat{\pi}_i^+$ are estimated once by logistic regression before the sampler starts and held fixed as covariates, so their estimation uncertainty is not propagated into the outcome posterior. This is a deliberate cut, and the same one made by BCF itself \citep{hahn2020bcf}, since feeding the propensity score into a joint posterior would let the flexible outcome model distort the treatment model. Its cost is recorded in Section~\ref{subsec:limitations}. The complete fitting algorithm, the exactness argument for the block-wise sweep, the tree-prior hyperparameters, and the convergence diagnostics are given in Section~\ref{sec:supp_mcmc} of the Supplementary Material.

\subsection{Relation to Existing Models}
\label{subsec:comparison}
ZIC-BCF-Smear combines the structural benefits of two-part hurdle models with the distribution-free flexibility of nonparametric causal forests. In Table~\ref{tab:model_feature_comparison}, the proposed model is compared against existing semicontinuous causal estimators on key qualitative and structural characteristics.

\begin{table}[htbp]
\centering
\caption{Qualitative comparison of semicontinuous causal estimators.}
\label{tab:model_feature_comparison}
\resizebox{\textwidth}{!}{%
\begin{tabular}{lcccccc}
\toprule
Feature & BCF-Linear & DPglm & EDP+ZI & Gamma Hurdle & Gamma + .01 & \textbf{ZIC-BCF-Smear} \\
\midrule
Separates Zero-Inflation & No & Yes & Yes & Yes & No & \textbf{Yes} \\
Hurdle CATE Estimation  & No & Yes & Yes & Yes & No & \textbf{Yes} \\
MCMC Augmentation Noise & None & Extreme & Extreme & None & None & \textbf{None} \\
Subpopulation Confounding & N/A & Adjusted & Adjusted & Adjusted & N/A & \textbf{Adjusted (SPA)} \\
Scale Re-transformation & None & Parametric & Parametric & Parametric & Parametric & \textbf{Non-Parametric} \\
\bottomrule
\end{tabular}%
}
\end{table}

BCF-Linear ignores the hurdle process completely, fitting a single prognostic surface over the zero discontinuity, which produces large point error under right skewness. DPglm and EDP+ZI impute latent continuous values for all zero observations, causing extreme data augmentation noise and posterior variance inflation. The parametric benchmarks are competitive under correct specification but suffer numerical instability and negative CATE correlation once link-linearity fails. ZIC-BCF-Smear decouples the hurdle from the continuous stage, adjusts for subpopulation confounding, and remains distribution-free via non-parametric smearing.

\section{Simulation Studies}
\label{sec:simulation_studies}

To evaluate the empirical performance of the proposed model, a series of comprehensive simulation studies is conducted. In this section, the simulation design is outlined, the mathematical details of the data-generating processes (DGPs) under both link-linear and link-nonlinear specifications are presented, and the comparative results are analyzed.

\subsection{Simulation Design, Estimators, and MCMC Settings}
\label{subsec:sim_design}

Five distinct estimators that represent different structural and distributional modeling assumptions are evaluated and compared:
\begin{enumerate}
    \item \textbf{BCF-Linear (Single-Part Gaussian):} A standard continuous BCF model \citep{hahn2020bcf} fit directly on the raw semicontinuous outcome $Y_i$, ignoring the discrete point mass at zero.
    \item \textbf{DPglm (Bayesian Nonparametric):} A Dirichlet-process mixture of zero-inflated generalized linear models \citep{oganisian2021zeroinflated}. This model cluster-imputes latent continuous values for all zero observations at each iteration, representing a flexible Bayesian nonparametric benchmark.
    \item \textbf{Gamma Hurdle (Parametric Two-Part):} A standard two-part model comprising a logistic regression hurdle and a Gamma log-link continuous intensity regression \citep{oganisian2021zeroinflated}, fit as weakly informative Bayesian GLMs with a fully treatment-interacted design and combined via G-computation.
    \item \textbf{Gamma + .01 (Naive Single-Part):} A naive single-part Gamma regression model with a log link, fit to the outcome after adding a small positive constant ($Y_i + 0.01$) to resolve the zero outcomes \citep{oganisian2021zeroinflated}.
    \item \textbf{ZIC-BCF-Smear (Proposed):} The proposed two-part BCF model that decouples the probit hurdle and active continuous intensity stage, utilizing Duan's non-parametric smearing to reconstruct response-scale effects.
\end{enumerate}

For all BCF-based models, MCMC settings are standardized to 1,000 burn-in iterations followed by 1,000 post-warmup sampling iterations, which was found to be sufficient for posterior convergence. The performance of each estimator is evaluated across 100 independent Monte Carlo replications (seeds) per cell. Performance is assessed using six key metrics: CATE Root Mean Squared Error (RMSE), CATE absolute bias, CATE 95\% credible interval coverage, CATE correlation (the Pearson correlation between estimated and true CATEs), CATE credible interval length, and population-level ATE RMSE and absolute bias.

\paragraph{Why EDP+ZI is absent from the benchmark.}
A sixth estimator, the enriched Dirichlet process zero-inflated model (EDP+ZI) of \citet{kim2024edpzi}, was implemented and verified against the authors' reference code but could not be carried into the simulation grid. Its Gibbs updates assign observations to latent clusters through multivariate normal marginal-likelihood evaluations and joint covariance updates that scale cubically in the sample size, $O(N^3)$, whereas a BCF forest costs $O(T \cdot P \cdot N \log N)$ per sweep for $T$ trees and $P$ covariates. The empirical consequence is stark: at $N = 500$ a single standard fit takes $\approx 18$ minutes for ZIC-BCF-Smear against $\approx 27.5$ hours for EDP+ZI, a $\approx 90\times$ slowdown, and the cubic term pushes a single EDP+ZI fit past nine days at $N = 1{,}000$. Across the full design (3 DGPs, 4 sample sizes, 4 zero-inflation levels, 2 link regimes, 100 seeds) it would require on the order of $10^5$ CPU hours, roughly a decade of continuous single-machine runtime, against capped academic computing sessions. It was therefore excluded, and that absence is recorded as a limitation in Section~\ref{subsec:limitations}. The exclusion is itself a substantive result: an estimator that cannot be run 100 times on a $500$-observation problem cannot be validated by simulation at all, which is precisely the barrier the forest-based formulation removes. The full timing analysis is given in Section~\ref{sec:supp_timings} of the Supplementary Material.

\subsection{Data-Generating Processes (DGPs)}
\label{subsec:dgps}

For all DGPs, $P = 5$ baseline covariates are generated independently as $X_{ip} \sim \mathcal{N}(0, 1)$, and confounding is introduced through a treatment propensity score that is non-linear in $X_{i2}$:
\begin{equation}
\pi(X_i) = P(Z_i = 1 \mid X_i) = \Phi(-0.5 + 0.4 X_{i1} + 0.3 X_{i2}^2), \qquad Z_i \sim \text{Bernoulli}(\pi(X_i)).
\label{eq:sim_propensity}
\end{equation}
Outcomes are then generated under three distribution families, each in a link-linear and a link-nonlinear version, giving six DGPs. Throughout, $c$ is the intercept that calibrates the marginal zero proportion, fixed at $\approx 40\%$ in the standard configuration and swept from $3\%$ to $85\%$ in Section~\ref{subsec:sensitivity}.

\paragraph{Link-linear DGPs.}
DGPs A and B are generated from a two-part process with a common probit
participation model and a common log-scale predictor for the positive outcome.
To make the connection with the proposed estimator explicit, define the true
hurdle-stage prognostic and treatment functions as
\begin{align}
\mu_b^*(X_i)
&=
0.5X_{i1}-0.3X_{i3}, \\
\tau_b^*(X_i)
&=
0.2+0.1X_{i1},
\end{align}
so that
\begin{equation}
I_i
\sim
\operatorname{Bernoulli}
\left[
\Phi\!\left(
c+\mu_b^*(X_i)+Z_i\tau_b^*(X_i)
\right)
\right].
\label{eq:dgp_a_hurdle}
\end{equation}
To distinguish these data-generating functions from the fitted BCF
components introduced in Section~\ref{sec:zic_bcf_smear_methods}, we denote the true
functions with a superscript $*$. Conditional on participation, $I_i=1$, define the true log-intensity
prognostic and treatment functions as
\begin{align}
\mu_c^*(X_i)
&=
1.5+0.8X_{i2}+0.4X_{i4}, \\
\tau_c^*(X_i)
&=
0.25-0.15X_{i2},
\end{align}
giving the common log-intensity predictor
\begin{equation}
\ell^*(X_i,Z_i)
=
\mu_c^*(X_i)+Z_i\tau_c^*(X_i).
\end{equation}
The observed outcome is then $Y_i=I_iY_i^+$.

DGPs A and B differ only in the conditional distribution of the positive
component. In DGP A,
\begin{equation}
\log Y_i^+
\sim
\mathcal{N}\!\left(
\ell^*(X_i,Z_i),\,\sigma^2
\right),
\qquad \sigma=0.5,
\end{equation}
so the probit hurdle and Gaussian log-intensity structure used by
ZIC-BCF-Smear agree with the data-generating mechanism. In DGP B,
\begin{equation}
Y_i^+
\sim
\operatorname{Gamma}
\left(
\text{shape}=2,\,
\text{scale}=
\frac{\exp\{\ell^*(X_i,Z_i)\}}{2}
\right),
\end{equation}
for which
$\mathbb{E}[Y_i^+\mid X_i,Z_i,I_i=1]
=\exp\{\ell^*(X_i,Z_i)\}$.
Thus, the parametric Gamma hurdle benchmark has the correct positive-outcome
family, whereas the Gaussian log-intensity model of ZIC-BCF-Smear is
distributionally misspecified. Its response-scale reconstruction must
therefore rely on the non-parametric smearing step rather than on correct
log-normal specification. Setting $c=0.2$ calibrates the common participation
model of DGPs A and B to approximately $40\%$ zeros.

DGP C provides a structurally different benchmark. Rather than generating
zeros through a separate participation process, it draws the outcome directly
from a compound Poisson--Gamma Tweedie distribution,
\begin{equation}
Y_i
\sim
\operatorname{Tweedie}
\left(
\mu_i,\phi=1.5,p=1.5
\right),
\end{equation}
with
\begin{equation}
\log\mu_i
=
1.2+c+0.8X_{i1}-0.4X_{i3}
+
Z_i(0.3+0.15X_{i1}).
\end{equation}
Here the probability mass at zero and the distribution of positive outcomes
are generated jointly by the same mean--dispersion mechanism rather than by
distinct hurdle and intensity processes. Consequently, any estimator that
imposes a two-part hurdle factorization, including ZIC-BCF-Smear, is
structurally misspecified under DGP C. This design is included deliberately
to quantify the cost of that structural misspecification. Setting $c=0$
yields approximately $50\%$ zeros.

\paragraph{Link-nonlinear DGPs.}
The link-nonlinear DGPs preserve the same distributional structure as their
link-linear counterparts while replacing the affine prognostic and
treatment-effect surfaces with nonlinear functions. This allows the
simulation to vary functional-form complexity without simultaneously changing
the underlying outcome family. For DGPs A and B, the true nonlinear
hurdle and intensity functions are:
\begin{align}
\mu_b^*(X_i)
&=
0.50 \sin(1.3 X_{i1})
- 0.30 X_{i3}
+ 0.35 X_{i1}X_{i2}
- 0.20(X_{i3}^2-1), \\
\tau_b^*(X_i)
&=
0.20 + 0.20 X_{i1}X_{i2}, \\
\mu_c^*(X_i)
&=
1.50
+ 0.80 \sin(1.2 X_{i2})
+ 0.40 X_{i4}
+ 0.35 X_{i2}X_{i4}
- 0.25(X_{i1}^2-1), \\
\tau_c^*(X_i)
&=
0.25
- 0.15 \sin(1.3 X_{i2})
+ 0.12 \tanh(X_{i1}X_{i2}).
\end{align}

The quadratic terms are mean-centered under
$X_i \sim \mathcal{N}(0,I_5)$, while the sinusoidal terms have expectation
zero by symmetry. This keeps the overall outcome scale and zero proportion
approximately comparable to those of the link-linear designs, so that the
two regimes differ primarily in functional form.

For nonlinear DGPs A and B, the participation probability is
\begin{equation}
p_z(X_i)
=
\Phi\!\left(
c + \mu_b^*(X_i) + z\,\tau_b^*(X_i)
\right),
\end{equation}
and the positive-outcome linear predictor is replaced by
\begin{equation}
\ell^*(X_i,z)
=
\mu_c^*(X_i) + z\,\tau_c^*(X_i),
\end{equation}
while retaining the log-normal and Gamma positive-outcome distributions,
respectively.

For nonlinear DGP C, the Tweedie dispersion and power parameters remain
$\phi=1.5$ and $p=1.5$, with
\begin{equation}
\log \mu_i
=
c + \mu_T^*(X_i)
+ Z_i\tau_T^*(X_i),
\end{equation}
where
\begin{align}
\mu_T^*(X_i)
&=
1.2
+ 0.80\sin(1.2X_{i1})
- 0.40X_{i3}
+ 0.35X_{i1}X_{i3}
- 0.25(X_{i2}^2-1), \\
\tau_T^*(X_i)
&=
0.30
- 0.15\sin(1.3X_{i1})
+ 0.12\tanh(X_{i1}X_{i3}).
\end{align}
The intercepts $c=0.2$ for DGPs A and B and $c=-2.072$ for DGP C
calibrate the three nonlinear designs to approximately $40\%$ zeros.

The link-linear regime provides the classical benchmark and is deliberately
favorable to the parametric competitors, because their fixed
$(1,z,X,z\cdot X)$ specification matches the true link-scale functional
form. In the link-nonlinear regime, the forest-based estimators retain the
flexibility to represent the true response surfaces, whereas the parametric
benchmarks are misspecified in functional form while retaining the correct
distributional family. The comparison therefore isolates functional-form
misspecification as the principal difference between the two regimes.

Closed-form expressions for the true potential-outcome means
$\mu_0(X_i)$ and $\mu_1(X_i)$ used to evaluate each estimator, together
with the intercepts used to calibrate each zero-inflation level, are given
in Section~\ref{sec:supp_dgps} of the Supplementary Material.

\subsection{Link-Linear DGP Comparative Results}
\label{subsec:linear_results}

Performance is first analyzed under the standard configuration ($N = 500$, nominal zero-inflation, 100 seeds) where conditional-mean surfaces are linear on the link scale. Under this setup, the parametric Gamma Hurdle model is correctly specified for DGP B and nearly correctly specified for DGP A. The comparative results are summarized in Table~\ref{tab:linear_standard_results}.

\begin{table}[htbp]
\centering
\caption{Performance metrics compiled across 100 independent seeds at the standard configuration ($N = 500$, ZI Level 3, $\approx 40\%$ zeros) under link-linear DGPs.}
\label{tab:linear_standard_results}
\resizebox{\textwidth}{!}{%
\begin{tabular}{llcccccc}
\toprule
DGP & Model & CATE RMSE & CATE Cov & CATE Corr & CATE CI Len & ATE RMSE & ATE Abs Bias \\
\midrule
\multirow{5}{*}{DGP A: Log-Normal}
& BCF-Linear    & 3.662 & 93.2\% & 0.184 & 6.315 & 0.849 & 0.282 \\
& DPglm         & \textbf{1.409} & 100.0\% & 0.321 & 51.210 & 0.781 & 0.637 \\
& Gamma Hurdle  & 1.810 & \textbf{94.9\%} & \textbf{0.419} & 4.641 & \textbf{0.585} & \textbf{0.029} \\
& Gamma +.01    & 2.946 & 91.7\% & 0.267 & 6.215 & 0.765 & 0.212 \\
& \textbf{ZIC-BCF-Smear} & 1.526 & \textbf{94.9\%} & 0.397 & \textbf{3.819} & 0.614 & 0.406 \\
\midrule
\multirow{5}{*}{DGP B: Gamma}
& BCF-Linear    & 4.010 & 92.5\% & 0.156 & 6.489 & 0.812 & 0.253 \\
& DPglm         & 1.406 & 100.0\% & 0.275 & 51.759 & 0.713 & 0.520 \\
& Gamma Hurdle  & 1.906 & 93.7\% & 0.370 & 4.860 & \textbf{0.629} & \textbf{0.024} \\
& Gamma +.01    & 3.076 & 91.1\% & 0.259 & 6.350 & 0.836 & 0.246 \\
& \textbf{ZIC-BCF-Smear} & \textbf{1.380} & \textbf{95.7\%} & \textbf{0.398} & \textbf{4.174} & 0.650 & 0.339 \\
\midrule
\multirow{5}{*}{DGP C: Tweedie}
& BCF-Linear    & 5.387 & 85.3\% & 0.723 & 7.379 & 0.715 & 0.169 \\
& DPglm         & 4.837 & 99.9\% & 0.267 & 51.351 & 1.003 & 0.793 \\
& \textbf{Gamma Hurdle} & \textbf{2.067} & \textbf{95.0\%} & \textbf{0.944} & \textbf{4.961} & \textbf{0.534} & 0.051 \\
& Gamma +.01    & 2.549 & 95.1\% & 0.914 & 5.662 & 0.614 & \textbf{0.012} \\
& ZIC-BCF-Smear & 2.471 & 98.2\% & 0.916 & 7.404 & 0.850 & 0.224 \\
\bottomrule
\end{tabular}%
}
\end{table}

Under hurdle DGPs (A and B), ZIC-BCF-Smear achieves the lowest CATE RMSE (1.526 and 1.380) and the narrowest credible interval lengths, demonstrating that the proposed decoupled forest structure is highly precise. Crucially, ZIC-BCF-Smear matches or outperforms the parametric Gamma Hurdle model (1.810 and 1.906) even though the latter is correctly specified for DGP B and nearly correctly specified for DGP A. This shows that committing to a parametric distribution for the positive part can be avoided without sacrificing estimation precision.

In contrast, BCF-Linear performs poorly under hurdle right-skewness (RMSE 3.662 and 4.010) because it forces a single Gaussian fit over the zero discontinuity. Naive Gamma + .01 is similarly biased, resulting in high CATE RMSE (2.946 and 3.076) and sub-nominal coverage (91.1\%--91.7\%), which highlights the danger of shifting zero outcomes arbitrarily. While DPglm attains low CATE RMSE, its posterior credible intervals are pathologically over-wide ($\approx 51.2$ compared to ZIC-BCF-Smear's $3.8$), pinning coverage at a degenerate $100\%$ and rendering it unusable for calibrated uncertainty quantification. That interval inflation is the direct signature of the decoupling argument of Section~\ref{subsec:intensity_stage}: because DPglm fits its continuous stage on the full sample by imputing a latent continuous value for every zero at every sweep, roughly $40\%$ of the continuous data it splits on at the standard configuration is simulated MCMC noise, which propagates into the cluster and split proposals and inflates the posterior. ZIC-BCF-Smear trains its intensity forest only on units with $Y_i > 0$ and therefore never pays that price, which is why it can be simultaneously more accurate and an order of magnitude sharper.

In DGP C, which features an intrinsic compound Tweedie structure without a hurdle partition, the parametric Gamma Hurdle is the strongest single model; however, ZIC-BCF-Smear remains competitive, achieving low RMSE (2.471) and positive correlation (0.916). This is the expected cost of imposing a two-part factorization on a process that does not have one, and it bounds the penalty: even where the structural assumption of the method is false by construction, it stays within $20\%$ of the best model on CATE RMSE and retains a correlation above $0.9$.

\subsection{Link-Nonlinear DGP Comparative Results}
\label{subsec:nonlinear_results}

Performance is next evaluated under link-nonlinear conditional-mean surfaces. Under this setup, the fixed linear design $(1, z, X, z \cdot X)$ of the parametric benchmarks (Gamma Hurdle and Gamma + .01) is misspecified. Standard cell results are presented in Table~\ref{tab:nonlinear_standard_results}; the accompanying numerical-stability metrics are reported in Section~\ref{sec:supp_full_results} of the Supplementary Material.

\begin{table}[htbp]
\centering
\caption{Standard comparative results ($N = 500$, ZI Level 3, $\approx 40\%$ zeros) under link-nonlinear DGPs. Dagger ($\dagger$) indicates metrics contaminated by numerical divergence.}
\label{tab:nonlinear_standard_results}
\resizebox{\textwidth}{!}{%
\begin{tabular}{llcccccc}
\toprule
DGP & Model & CATE RMSE & CATE Cov & CATE Corr & CATE CI Len & ATE RMSE & ATE Abs Bias \\
\midrule
\multirow{5}{*}{DGP A: Log-Normal}
& BCF-Linear    & 5.812 & 88.1\% & 0.347 & 6.385 & 1.127 & 0.172 \\
& DPglm         & 3.070 & \textbf{99.7\%} & 0.198 & 61.532 & 0.680 & \textbf{0.021} \\
& Gamma Hurdle  & $9.7 \times 10^7 \dagger$ & 69.2\% & -0.039 & $5.8 \times 10^7 \dagger$ & $5.6 \times 10^8 \dagger$ & $5.9 \times 10^7 \dagger$ \\
& Gamma +.01    & $3.0 \times 10^9 \dagger$ & 69.5\% & -0.081 & $2.1 \times 10^{11} \dagger$ & $7.5 \times 10^8 \dagger$ & $7.3 \times 10^7 \dagger$ \\
& \textbf{ZIC-BCF-Smear} & \textbf{1.734} & 89.7\% & \textbf{0.686} & \textbf{4.422} & \textbf{0.537} & 0.119 \\
\midrule
\multirow{5}{*}{DGP B: Gamma}
& BCF-Linear    & 5.792 & 89.1\% & 0.318 & 6.624 & 1.101 & 0.217 \\
& DPglm         & 2.639 & 99.5\% & 0.211 & 61.457 & 0.716 & 0.126 \\
& Gamma Hurdle  & $1.7 \times 10^{10} \dagger$ & 74.1\% & -0.040 & $1.5 \times 10^{11} \dagger$ & $7.5 \times 10^{10} \dagger$ & $1.1 \times 10^{10} \dagger$ \\
& Gamma +.01    & $5.1 \times 10^8 \dagger$ & 68.9\% & -0.068 & $9.7 \times 10^8 \dagger$ & $8.7 \times 10^8 \dagger$ & $5.0 \times 10^7 \dagger$ \\
& \textbf{ZIC-BCF-Smear} & \textbf{1.656} & \textbf{91.7\%} & \textbf{0.645} & \textbf{4.753} & \textbf{0.641} & \textbf{0.067} \\
\midrule
\multirow{5}{*}{DGP C: Tweedie}
& BCF-Linear    & 0.302 & 91.6\% & 0.060 & \textbf{0.638} & 0.111 & \textbf{0.026} \\
& DPglm         & 0.404 & 99.9\% & 0.022 & 37.874 & 0.114 & 0.058 \\
& Gamma Hurdle  & $3.8 \times 10^{10} \dagger$ & 81.4\% & -0.060 & $1.7 \times 10^{11} \dagger$ & $6.3 \times 10^{10} \dagger$ & $7.8 \times 10^9 \dagger$ \\
& Gamma +.01    & $1.9 \times 10^{10} \dagger$ & 76.3\% & -0.045 & $2.2 \times 10^{11} \dagger$ & $1.2 \times 10^{11} \dagger$ & $1.2 \times 10^{10} \dagger$ \\
& \textbf{ZIC-BCF-Smear} & \textbf{0.250} & \textbf{94.6\%} & \textbf{0.175} & 0.704 & \textbf{0.101} & 0.030 \\
\bottomrule
\end{tabular}%
}
\end{table}

Under link-nonlinearity, the parametric Gamma benchmarks break down in two ways. First, they fail to capture treatment heterogeneity, yielding negative CATE correlations in all three DGPs (ranging from $-0.039$ to $-0.081$). This shows that a misspecified linear model ranks treatment effects in roughly the wrong order, which for a CATE learner is a more serious failure than a scale error: a decision rule that targets treatment to the units with the largest estimated benefit would, under these benchmarks, select approximately the wrong units. Second, the misspecified linear predictors generate extreme fitted values on tail covariate combinations, causing the link-$\exp(\cdot)$ and penalized-IRLS updates to overflow. Divergence occurs on 4 to 5 of 100 seeds for Gamma Hurdle and 5 to 9 for Gamma + .01, inflating the mean CATE RMSE and ATE metrics to the $10^7$--$10^{11}$ range. Robust medians computed after excluding the diverged seeds are reported in Section~\ref{sec:supp_full_results} of the Supplementary Material; they do not rescue the benchmarks, since the negative CATE correlations persist and the median CATE RMSEs remain at 3.7 to 4.3 in DGP A and B.

The mechanism behind that overflow is exactly what Duan's smearing is introduced to defuse. A parametric log-scale model re-transforms through a multiplier fixed by its assumed family, $\exp(\sigma_c^2/2)$ in the log-normal case; when the model is wrong in functional form, the residuals it fails to explain inflate $\sigma_c^2$, the multiplier grows exponentially in that inflation, and the error compounds on precisely the tail covariate combinations where the linear predictor is already most extreme. The smearing factor $\phi_{\text{smear}}^{(s)}$ of Section~\ref{subsec:smearing} replaces it with the empirical mean of exponentiated residuals, which is bounded by the observed data. ZIC-BCF-Smear is therefore not merely more accurate here but categorically more stable, recording zero diverged seeds across all 100 replications of all three nonlinear DGPs.

Correspondingly, ZIC-BCF-Smear achieves the lowest CATE RMSE (1.734 / 1.656 / 0.250) across all DGPs, outperforming the robust medians of the parametric models by a factor of 2.0 to 2.7. It also maintains positive CATE correlations (0.686 / 0.645 / 0.175), successfully identifying treatment effect direction and magnitude. Under nonlinearity, the ATE ``refuge'' that Gamma Hurdle enjoyed in the linear study collapses: ZIC-BCF-Smear yields the lowest ATE RMSE in all three DGPs. BCF-Linear is stable but weak under hurdle skew (RMSE $\approx 5.8$, coverage $\approx 88\%$ in A/B), while DPglm remains pathologically over-wide (CI length $\approx 61$). The one blemish is a mild under-coverage of ZIC-BCF-Smear in DGP A ($89.7\%$) and DGP B ($91.7\%$), traced in Section~\ref{subsec:limitations} to the conservative BART regularization priors, which shrink the posterior toward the prior mean for units in the tails of the covariate distribution.

\subsection{Hurdle Stage Performance}
\label{subsec:hurdle_performance}

Because the response-scale effect is a product of a participation probability and an intensity, the quality of the hurdle stage is as consequential as that of the intensity stage, and it is separately identified. Estimators that decouple zero-inflation from positive outcomes model the binary indicator $I_i = \mathbb{I}(Y_i > 0)$ in the first stage, so the hurdle metrics of ZIC-BCF-Smear (Probit BCF), DPglm (Dirichlet process mixture), and Gamma Hurdle (logistic GLM) can be compared directly; BCF-Linear and Gamma + .01 do not model a hurdle and are absent from this comparison. Table~\ref{tab:hurdle_results} reports the standard configuration under both link regimes.

\begin{table}[htbp]
\centering
\caption{Hurdle-stage performance at the standard configuration ($N = 500$, ZI Level 3, $\approx 40\%$ zeros), across 100 seeds. RMSE is on the participation-probability scale and coverage is the empirical $95\%$ credible-interval coverage of the true participation probability. BCF-Linear and Gamma + .01 do not model a hurdle stage.}
\label{tab:hurdle_results}
\begin{tabular}{llcccc}
\toprule
& & \multicolumn{2}{c}{Link-Linear} & \multicolumn{2}{c}{Link-Nonlinear} \\
\cmidrule(lr){3-4} \cmidrule(lr){5-6}
DGP & Model & Hurdle RMSE & Coverage & Hurdle RMSE & Coverage \\
\midrule
\multirow{3}{*}{DGP A: Log-Normal}
& DPglm         & 0.053 & 99.7\% & 0.077 & 98.6\% \\
& Gamma Hurdle  & 0.098 & \textbf{94.0\%} & 0.129 & 88.6\% \\
& \textbf{ZIC-BCF-Smear} & \textbf{0.046} & 99.3\% & \textbf{0.065} & \textbf{94.7\%} \\
\midrule
\multirow{3}{*}{DGP B: Gamma}
& DPglm         & 0.055 & 99.7\% & 0.077 & 98.2\% \\
& Gamma Hurdle  & 0.098 & \textbf{94.0\%} & 0.129 & 88.5\% \\
& \textbf{ZIC-BCF-Smear} & \textbf{0.046} & 99.3\% & \textbf{0.066} & \textbf{94.5\%} \\
\midrule
\multirow{3}{*}{DGP C: Tweedie}
& DPglm         & 0.044 & 87.0\% & 0.057 & 99.8\% \\
& Gamma Hurdle  & 0.081 & \textbf{95.0\%} & 0.111 & \textbf{92.4\%} \\
& \textbf{ZIC-BCF-Smear} & \textbf{0.028} & 99.3\% & \textbf{0.044} & 98.4\% \\
\bottomrule
\end{tabular}
\end{table}

ZIC-BCF-Smear attains the lowest hurdle RMSE (the RMSE of the estimated unit-level treatment effect on the participation probability) in every cell of the table, under both link regimes and all three DGPs. Under link-linearity its advantage over the logistic GLM of the Gamma Hurdle benchmark (0.046 versus 0.098 in DGP A and B) has a clean explanation: the true participation stage in \eqref{eq:dgp_a_hurdle} is a probit, so the Probit BCF stage matches the true link while the logistic GLM is mildly misspecified even in the regime designed to favor it. Under link-nonlinearity the logistic GLM is misspecified in functional form as well, its RMSE rises to 0.129, and its coverage collapses to $\approx 88.5\%$; ZIC-BCF-Smear holds RMSE at 0.065--0.066 with essentially nominal coverage (94.5\%--94.7\%). DPglm is accurate but uninformative, its hurdle credible intervals averaging a length of $\approx 1.99$ on a probability scale bounded in $[0,1]$, which is a second manifestation of the augmentation-noise inflation seen on the response scale. Full hurdle-stage results across sample sizes and zero-inflation levels are tabulated in Section~\ref{sec:supp_full_results} of the Supplementary Material.

\subsection{Sensitivity to Sample Size and Zero-Inflation}
\label{subsec:sensitivity}

Two sensitivity sweeps are run: a sample-size grid $N \in \{100, 250, 500, 1000\}$ under nominal zero-inflation, and a zero-inflation grid spanning realized zero proportions from $3\%$ to $85\%$ at $N = 500$. Both are run under both link regimes and all three DGPs, giving 48 additional cells. The complete numerical results and all sensitivity figures are reported in Section~\ref{sec:supp_full_results} of the Supplementary Material; the substantive findings are four.

First, ZIC-BCF-Smear is asymptotically well behaved. Under link-linearity its CATE RMSE falls and its credible intervals contract systematically with $N$ (interval length $6.73 \rightarrow 4.82 \rightarrow 3.82 \rightarrow 3.20$ across the grid in DGP A) while coverage holds at the nominal level, as does the correctly specified Gamma Hurdle. DPglm's intervals contract too but remain an order of magnitude wider ($99.19 \rightarrow 42.61$), so its coverage stays pinned near $100\%$ at every sample size; BCF-Linear contracts but its coverage collapses (to $83.5\%$ under Tweedie at $N = 1{,}000$), and Gamma + .01 diverges at $N = 100$.

Second, the parametric failure under nonlinearity is not a small-sample artifact and does not wash out with data. The parametric benchmarks diverge in 23 of the 24 nonlinear sensitivity cells, with divergence rates up to $10\%$, and their CATE correlations remain negative at every sample size, including $N = 1{,}000$. Misspecification in functional form is not something more observations repair.

Third, the advantage of the decoupled structure scales directly with the amount of zero-inflation, which is the design claim of the method. At $\approx 85\%$ zeros the CATE RMSE of ZIC-BCF-Smear in DGP A is $17.9\%$ (link-linear) and $27.5\%$ (link-nonlinear) of the single-part BCF-Linear baseline. The reason is structural rather than incidental: BCF-Linear must resolve the zero mass with the same smooth surface it uses for the positive part, and the more mass sits at zero the more that single surface is dominated by the discontinuity, whereas the ZIC-BCF-Smear intensity forest is never asked to see a zero at all.

Fourth, extreme zero-inflation is where the competitors break rather than merely degrade. At $85\%$ zeros in DGP B the correctly specified parametric hurdle collapses to a normalized RMSE of $1.372$, worse than the BCF-Linear baseline, because with only $\approx 15\%$ positive observations the Gamma GLM has too little active data to fit; Gamma + .01 diverges outright; and DPglm collapses under Tweedie (normalized RMSE $14.424$), its latent imputation being almost pure noise when the active fraction is small. ZIC-BCF-Smear remains stable across the entire range and outperforms every competitor at the extremes of both grids.

\section{Applied Study: Severe Adverse-Event Burden in a Head-and-Neck Cancer Trial}
\label{sec:applied}

This section applies our proposed model to individual-participant data from SPECTRUM, an open-label phase III randomized trial of cisplatin and fluorouracil with or without panitumumab in recurrent or metastatic squamous-cell carcinoma of the head and neck \citep{vermorken2013spectrum, nct00460265}, obtained through Project Data Sphere \citep{pds_headne_265}. A condensed account is given here; the full study, with the complete subgroup tables, the exposure and missingness diagnostics, the model comparison, the literature positioning, and the reproducibility trail, is Section~\ref{sec:supp_onc} of the Supplementary Material.

\subsection{Data, Outcome, and Estimand}
\label{subsec:applied_data}

The analytic extract contains 520 participants balanced at 260 per arm, with $Z_i = 1$ denoting panitumumab plus chemotherapy. Because the trial randomized assignment, the unadjusted arm difference is already unbiased for the average effect in this cohort; the baseline adjustment set is a variance-reduction device rather than a requirement for validity, and the known design propensity is $0.5$. That guarantee attaches to the randomization, so it holds for the 520-person extract to the extent that the extract is a randomization-respecting subset of the parent trial's 657 randomized patients. The supplied files do not record how the 657 became 520, and the exactly balanced 260 per arm cannot by itself distinguish a benign completeness filter from a selection that differs by arm. Section~\ref{sec:supp_onc} of the Supplementary Material states the reconciliation that is checkable from the supplied files and flags the step that is not.

The covariate vector $X_i$ contains: 1. age, 2. ECOG performance status, 3. sex, 4. diagnosis site, 5. HPV status, 6. disease status, and the 7. patient-level prior-radiotherapy indicator. Prior radiotherapy is present in $76.9\%$ of the chemotherapy arm and $83.1\%$ of the panitumumab arm and is included because it is clinically relevant to the mucosal and cutaneous toxicity that the outcome measures. The broader composite prior-treatment variable is reported descriptively but is not entered alongside prior-radiotherapy indicator, in the adjustment or the subgroup lists, because the two indicators differ for only four patients and their near-redundancy destabilized the Gamma benchmark while adding an unadjusted near-duplicate contrast to every subgroup grid. Detailed reconciliation of the radiotherapy administration domain and lesion-level records against patient-level baseline flags is provided in Section~\ref{sec:supp_onc} of the Supplementary Material. Furthermore, tumor category is excluded because it is a deterministic recode of diagnosis site in this extract, so carrying both places an exact linear combination of existing columns in the design matrix.

The outcome is the patient-level burden of severe toxicity, $Y_i^{\mathrm{sum}} = \sum_{j \in \mathcal{A}_i} d_{ij}\,\mathbf{1}(G_{ij} \ge 3)$, where $\mathcal{A}_i$ indexes the on-study adverse-event records for patient $i$, $d_{ij}$ is the recorded duration of record $j$, and $G_{ij}$ its recorded grade. This is exactly the regime the estimator targets: $123$ patients ($23.7\%$) accrue no severe-adverse-event time at all, while the positive part is heavily right-skewed with a median of $26$ and a maximum of $641$ event-days. Because $Y_i^{\mathrm{sum}}$ sums record durations rather than taking the union of calendar intervals, concurrent severe events contribute more than one event-day on the same date. A companion endpoint $Y_i^{\mathrm{cal}}$, the number of distinct on-study calendar days on which at least one grade 3+ event was ongoing, is therefore also constructed by merging overlapping intervals within each patient; the union removes $27.1\%$ of the record-days but leaves the zero count unchanged. Of the $1{,}671$ qualifying severe records, $1{,}547$ have a complete imputed interval and $124$ do not, in every case because the end day rather than the start day is unresolved. Those $124$ records fall across $88$ patients, $51$ on panitumumab and $37$ on chemotherapy, whose calendar-day totals are correspondingly understated, and a smaller group of $13$ patients has no complete qualifying interval at all and is assigned zero outright. That group is imbalanced, $10$ on panitumumab against $3$ on chemotherapy, which runs in the direction of understating the panitumumab arm and lands on the hurdle component. 

Patient disposition records provide the scheduled safety follow-up milestone and the last-contact study day for all 520 patients.      However, only 354 patients completed the protocol-specified safety follow-up, and three complete adverse-event intervals extend beyond the recorded last-contact day. These disposition data therefore permit diagnostic bounding of the observed follow-up window, but they do not establish an entirely closed observation period; an unresolved event end date cannot be treated as definitive evidence of ongoing toxicity at the time of data collection.

For either outcome the target is the analytic-cohort average treatment effect $\Delta^{k} = \mathbb{E}\{Y_i^{k}(1) - Y_i^{k}(0) \mid i \in \mathcal{S}\}$ over the 520-person cohort $\mathcal{S}$, which under randomization and consistency is the effect of \emph{assignment} on accrued severe-adverse-event days under the trial's own differential treatment duration. Treatment duration is not held fixed and is not adjusted for, deliberately: panitumumab continues as maintenance after chemotherapy stops, so time on treatment is a post-randomization consequence of assignment and lies on the causal pathway from assignment to accrued burden. It can therefore be neither ignored in interpretation nor naively conditioned on in estimation. Section~\ref{subsec:applied_exposure} quantifies how much of the effect travels through that channel. The estimand is not a fixed-horizon restricted mean, not a rate per unit of exposure, and not a competing-risk estimand.

\subsection{Average Effects and Endpoint Sensitivity}
\label{subsec:applied_ate}

ZIC-BCF-Smear was fit as four chains of $4{,}000$ retained draws after $5{,}000$ burn-in iterations, and the Bayesian Gamma hurdle model of \citet{oganisian2021zeroinflated} was fit to the same cohort, covariates, and treatment coding as a specification benchmark, likewise as four chains. Across the global monitored functionals the largest Gelman--Rubin statistic is $1.001$ and the smallest effective sample size is $2{,}670$ for ZIC-BCF-Smear; the corresponding Gamma-hurdle values are $1.000$ and $16{,}000$ on the same four-chain, $4{,}000$-draw-per-chain budget \citep{gelman1992inference, plummer2006coda}. The complete diagnostic tables are in the Supplementary Material. Table~\ref{tab:applied_ate} reports both models on both endpoints, together with the participation margin.

\begin{table}[htbp]
\centering
\caption{Posterior average treatment effects of panitumumab assignment on severe-adverse-event burden, SPECTRUM analytic cohort ($N = 520$). Positive values indicate a higher burden, or a higher probability of a positive reconstructed severe-adverse-event burden, under panitumumab. The hurdle margin is not literally the probability of any severe adverse event: 13 patients with severe events are reconstructed as zeros because their observation intervals were incomplete. The first four rows are on the event-day scale; the last two are differences in probability. The unadjusted arm differences, which randomization makes unbiased for the same estimand, are $15.68$ and $12.06$ event-days.}
\label{tab:applied_ate}
\begin{tabular}{llcc}
\toprule
Estimand & Model & Estimate & 95\% Credible Interval \\
\midrule
\multirow{2}{*}{Summed record duration (event-days)}
& ZIC-BCF-Smear & 14.27 & [4.69, 24.23] \\
& Gamma Hurdle  & 15.95 & [6.15, 26.30] \\
\midrule
\multirow{2}{*}{Distinct calendar days}
& ZIC-BCF-Smear & 10.00 & [3.57, 16.62] \\
& Gamma Hurdle  & 12.20 & [5.56, 19.23] \\
\midrule
\multirow{2}{*}{Participation margin (probability)}
& ZIC-BCF-Smear & 0.029 & [$-0.038$, 0.102] \\
& Gamma Hurdle  & 0.017 & [$-0.055$, 0.089] \\
\bottomrule
\end{tabular}
\end{table}

Three results follow. First, the average effect is credibly positive and survives the change of outcome definition: ZIC-BCF-Smear estimates $14.27$ additional event-days on the summed scale (posterior probability of a positive effect $0.999$) and $10.00$ additional days once concurrent events are counted once ($0.999$). Both the sign and the near-certainty agree with the parent trial's qualitative safety finding, namely that the panitumumab-containing regimen produced more severe toxicity across several conventional adverse-event categories \citep{vermorken2013spectrum}, and the conclusion survives the one outcome-definition change that could plausibly have manufactured it, the double-counting of concurrent severe events. The $4.27$-day attenuation measures how much of the summed burden reflects simultaneity, not a reversal of the safety conclusion. What the literature cannot supply here is a numerical benchmark, since standard safety tables summarize event incidence, grade, and category rather than summed record durations or overlap-adjusted calendar days \citep{vermorken2013spectrum}; external agreement is therefore assessable for direction and robustness but not for magnitude, which is exactly the gap our applied study fills.

Second, the two specifications agree on the average, at $15.95$ and $12.20$ days for the Gamma hurdle, with a gap between point estimates that is small relative to the width of either posterior, so the finding is not an artifact of either functional form. Both also sit close to the corresponding unadjusted arm differences of $15.68$ and $12.06$ event-days. Under randomization that proximity is the benchmark rather than a weakness: the unadjusted contrast is itself unbiased for the cohort average effect, so a model that departed sharply from it would need to explain why. What the adjusted fits add is precision and the ability to resolve the effect by covariate profile and by component, not a correction to a biased raw comparison.

Third, and most informative for the two-part structure, the participation margin is uncertain under both models. Since $76.3\%$ of the cohort has a positive burden in both arms, that margin has little room to move, and the credible effect is carried mainly by the positive-intensity component. The supported claim therefore concerns how much severe-toxicity time accumulates among patients whose reconstructed burden is positive, rather than whether that reconstructed burden is positive at all. That attribution is a direct product of decoupling the two stages; a single-part estimator on the raw day scale returns the same aggregate number with no way to localize it, and the published safety evidence, which emphasizes particular event categories and grades rather than an aggregate any-severe-AE indicator \citep{vermorken2013spectrum}, contains no quantity that either confirms or refutes this margin. The defensible interpretation is a difference concentrated in positive intensity, which can reflect event duration, event multiplicity, overlap, treatment exposure, and incomplete interval reconstruction rather than persistence alone.

\subsection{Treatment Exposure and What the Estimand Does Not Cover}
\label{subsec:applied_exposure}

The endpoint counts days in a severe-adverse-event state, and that quantity accrues with time on treatment. Exposure is grossly unequal, and structurally so: mean treatment duration is $83.7$ days in the chemotherapy arm against $130.3$ days under panitumumab, a $56\%$ difference, with maxima of $276$ and $1{,}149$ days. The consequence is large. Truncating every patient's observed accrual at a common study day collapses the arm difference in distinct severe-adverse-event days from $12.06$ days uncapped to $4.10$ at 120 days, $3.37$ at 90 days, and $3.14$ at 84 days, and the same attenuation reproduces across all four start-date and end-date variants recorded in the file.

Two cautions govern how that table is read, and they matter more than the numbers. It is truncation of observed accrual, not a fixed-horizon estimand: capping at day 84 removes events recorded after day 84 but does not create observation for a patient whose follow-up ended earlier. The new disposition domains now make that distinction measurable. Using the last-contact study day variable as an observed upper bound gives mean uncapped observed-window burdens of $19.18$ and $30.90$ days and, at 84 days, $11.25$ and $14.33$ days in the chemotherapy and panitumumab arms, respectively; $47$ and $41$ patients had last contact before day 84. These are still diagnostics, not fixed-horizon causal estimates, because early contact loss and incomplete interval boundaries require a censoring strategy. Whether the effect reverses on a per-exposure basis also depends on the normalization, since the patient-level mean of per-100-treatment-day rates reverses, at $53.7$ against $49.9$, while the aggregate person-time rate does not, at $22.9$ against $24.0$. Fitting the same two models to the per-exposure rate reproduces that ambiguity rather than resolving it: ZIC-BCF-Smear returns $5.19$ days per 100 treatment-days $[-4.39, 17.28]$ and the Gamma hurdle $-5.23$ $[-37.06, 25.19]$, both compatible with no effect. The defensible claim is severe attenuation once exposure is accounted for, not reversal. The horizon diagnostic provides the first check on that collapse, and it is deliberately weaker than a corrected estimate. Table~\ref{tab:applied_horizon} reports means among all patients and among patients whose last contact reaches each horizon. The arm difference is not smaller among patients observed through the horizon, going from 3.08 to 3.68 days at day 84 and from 4.03 to 5.13 days at day 120, so simple truncation of the shortest-observed records is not what produces the attenuation. Restricting to patients observed through the horizon conditions on a post-randomization variable, however, so informative censoring remains possible and the direction of any resulting selection bias cannot be signed in advance; these rows are therefore reported as diagnostics, not as a corrected fixed-horizon effect.

\begin{table}[htbp]
\centering
\caption{Mean accrued severe-adverse-event days by arm under common-horizon truncation and under capping at the patient-specific last-contact day (LASTCTDY), on the 520-person cohort. The observed-to-horizon rows restrict to patients whose last contact reaches the window ($n=213$ chemotherapy and $219$ panitumumab at day 84; $194$ and $203$ at day 120). Capping applies only to intervals with complete boundaries, so these are diagnostics bounding, not removing, the unresolved-record problem discussed above. The day-84 and day-120 rows apply the same last-contact cap, which is why they sit slightly below the plain uncapped truncation differences quoted in the text ($3.08$ against $3.14$ at day 84 and $4.03$ against $4.10$ at day 120).}
\label{tab:applied_horizon}
\begin{tabular}{lccr}
\toprule
Window & Chemotherapy & Panitumumab & Difference \\
\midrule
Last contact capped & 19.18 & 30.90 & 11.73 \\
Day 84, all patients & 11.25 & 14.33 & 3.08 \\
Day 84, observed to day 84 & 11.25 & 14.93 & 3.68 \\
Day 120, all patients & 15.55 & 19.58 & 4.03 \\
Day 120, observed to day 120 & 16.12 & 21.25 & 5.13 \\
\bottomrule
\end{tabular}
\end{table}

None of this is offered as a corrected causal estimate, and the randomized analysis of Section~\ref{subsec:applied_ate} is left standing rather than replaced. Treatment duration is measured after randomization, so dividing by it conditions on a post-treatment variable and can induce selection bias whose direction is not fixed in advance: a patient who stops early because of toxicity contributes a short denominator alongside a large numerator, inflating that patient's rate and making their arm look worse, whereas a patient who stops early for an unrelated reason pushes the same arm's rate down. Which dominates depends on why patients discontinued, and the supplied extract does not provide a treatment-stop reason that resolves this. The per-exposure quantities are therefore reported as descriptive rates per treatment-day, and the truncation table as a diagnostic showing that the effect is exposure-dependent. 

Although disposition records make a formal fixed-horizon analysis conceptually feasible, they do not obviate the need to explicitly address informative censoring, competing events, and intercurrent treatment policy. The available data therefore justify an observed-window sensitivity diagnostic bounded by last contact, whereas a formal restricted-mean estimand represents a distinct follow-up analysis rather than a replacement for the randomized cohort estimand

\subsection{Heterogeneity and Specification Sensitivity}
\label{subsec:applied_heterogeneity}

Subgroup effects are formed draw by draw, averaging unit-level effects within each posterior draw over the members of a subgroup, so that a contrast between two subgroups inherits their posterior covariance rather than being read off two unrelated marginal intervals. Figure~\ref{fig:applied_cates} displays the result.

\begin{figure}[htbp]
\centering
\includegraphics[width=0.56\textwidth]{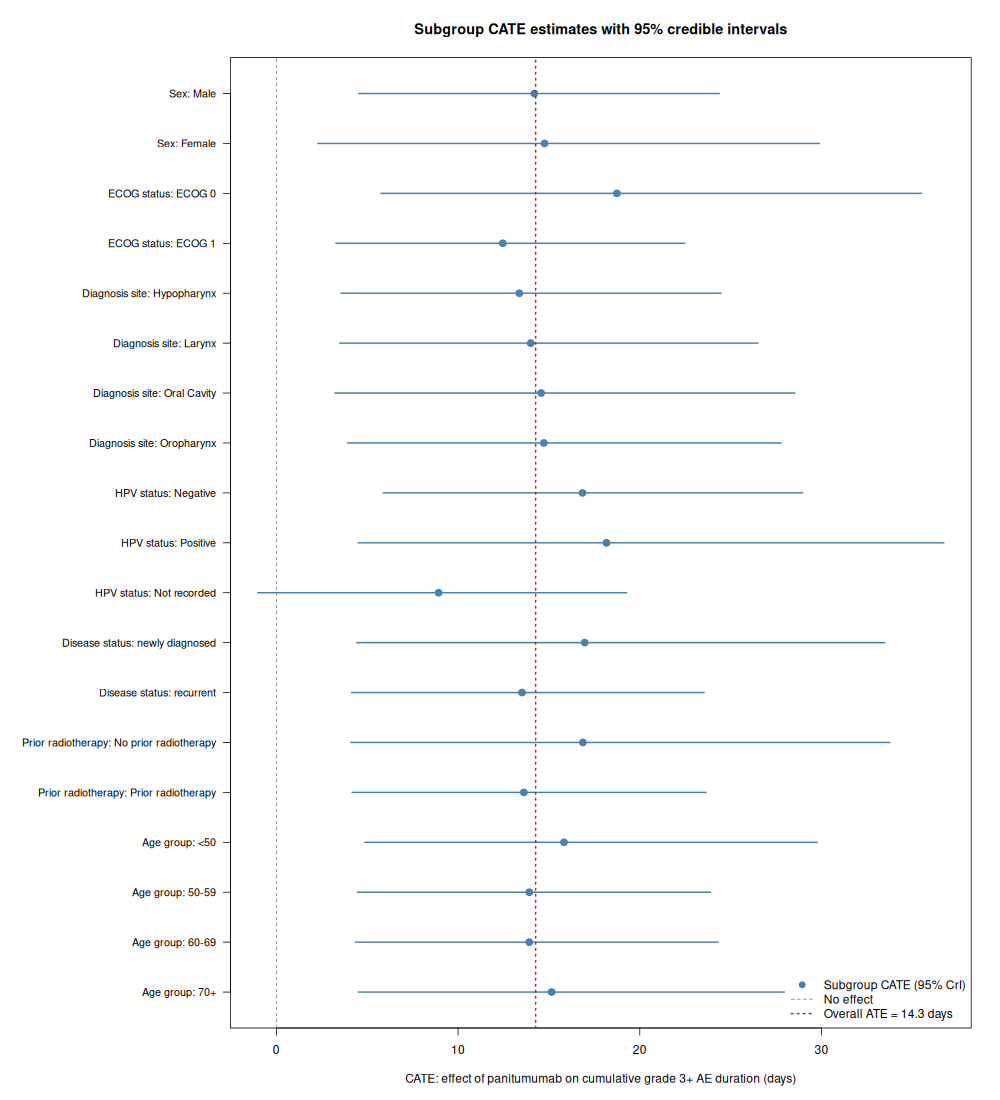}
\caption{ZIC-BCF-Smear subgroup conditional average treatment effects on summed severe-adverse-event-record duration, in event-days, with $95\%$ credible intervals; the dashed line marks no effect and the dotted line the overall effect of $14.3$ event-days. Every subgroup posterior mean falls within roughly five event-days of the overall effect and every interval covers it, so the estimated effect surface is broadly homogeneous. Interval width tracks subgroup size almost monotonically, widest for HPV-positive patients ($n = 69$), women ($n = 65$), and patients aged 70 or older ($n = 36$), which is the expected consequence of borrowing strength across a tree ensemble rather than fitting stratum-specific models.}
\label{fig:applied_cates}
\end{figure}

Posterior means are positive in every prespecified subgroup on both endpoints, but no ZIC-BCF-Smear pairwise contrast has a $95\%$ credible interval excluding zero. The largest ZIC-BCF-Smear signal, ECOG 0 versus ECOG 1, is $6.28$ event-days $[-2.57, 22.20]$ on the summed endpoint and $6.69$ days $[-0.83, 20.43]$ on the distinct-day endpoint. The Gamma hurdle benchmark reaches the opposite conclusion on the same data, estimating those contrasts at $26.94$ event-days $[4.74, 51.94]$ and $23.34$ days $[7.83, 40.98]$, reproducibly across seeds, and it likewise returns credible participation-margin contrasts for ECOG 0 versus ECOG 1 and for HPV-negative versus HPV-not-recorded where ZIC-BCF-Smear returns none. Because both models are fit to identical data, covariates, and treatment coding, the difference is attributable to specification: the parametric Gamma positive component with linear dummy-encoded covariates extrapolates subgroup means without the regularization the causal forest imposes. This is the applied counterpart of the divergence documented under link-nonlinearity in Section~\ref{subsec:nonlinear_results}, where the same benchmark, fit to data whose truth was known, produced negative CATE correlations and diverged outright on 4 to 5 seeds in 100. Its practical consequence here is concrete, since an ECOG effect-modification claim would be reported as credible under the parametric benchmark and as inconclusive under ZIC-BCF-Smear.

One of these contrasts deserves some extra attention in our analysis. The HPV comparison that the benchmark reports as credible is against a category created by coalescing patients with no recorded baseline HPV result, so it is a missingness contrast over $181$ patients rather than a comparison of two measured biological states, and it should not be read as evidence about HPV biology. Prior radiotherapy is the clinically motivated candidate modifier: the summed-duration effect is $16.87$ event-days $[4.10, 33.76]$ among the $104$ patients without prior radiotherapy and $13.62$ $[4.17, 23.65]$ among the $416$ with prior radiotherapy, with contrast $3.25$ $[-7.61, 17.82]$. This contrast is not credible, and the smaller stratum is limited. The composite prior-treatment variable gives materially identical strata (the two indicators differ for four patients), so it is not reported as a separate subgroup comparison, which would add a near-duplicate contrast to an analysis that is already unadjusted for multiplicity.

Which answer is right cannot be settled internally, but it can be settled against the literature. Previous SPECTRUM subgroup work concentrated on efficacy and on HPV or p16 status, and the one available HPV-specific safety analysis is a conference abstract that does not estimate baseline-modified causal effects on either burden outcome \citep{stohlmacher2012hpv}; neither ECOG nor HPV status is established anywhere in the SPECTRUM literature as a modifier of the causal burden of panitumumab-associated toxicity \citep{vermorken2013spectrum, stohlmacher2012hpv}. A regularized estimator returning no credible heterogeneity is therefore the concordant answer, and the Gamma-specific signals are best read as model-sensitive hypotheses rather than as externally corroborated effect modification. This is the same asymmetry the simulations produced, with the direction of the error preserved: under functional misspecification the parametric hurdle does not merely lose precision, it reports structure that is not there.

\subsection{A Second Application: Dental Expenditures}
\label{subsec:applied_meps}

The oncology study exercises the estimator where the applied literature is thin, since the study is novel in exploring effect in adverse-event duration scale; consequently, external agreement can be assessed for direction and robustness but not for magnitude. A second study supplies the complementary case, estimating the causal effect of dental-insurance coverage on annual dental expenditure in the 2023 Medical Expenditure Panel Survey \citep{meps2023file}, an $N = 18{,}763$ observational sample in which $54.9\%$ of respondents report zero dental spending and the positive part is severely right-skewed (median $\$361$, mean $\$1{,}038$, maximum $\$40{,}350$). This is the canonical semicontinuous outcome of health economics \citep{duan1983healthexpenditure, manning2001estimating}, and decades of applied work have documented the gradients that coverage effects should exhibit by race and ethnicity, age, income, and region \citep{manski2002private, zhang2016racial, reda2018inequality, vujicic2016financial}. It is reported in full in Section~\ref{sec:supp_meps} of the Supplementary Material rather than here, because its role is confirmatory rather than methodological.

MEPS is a stratified, clustered, unequal-probability sample, so two estimands are reported for every quantity and they answer different questions. The analytic-sample average describes the $18{,}763$ records, which oversample several subpopulations; the survey-weighted population target uses the person weight and carries a design-based variance component obtained by Taylor linearization over the sampling strata and primary sampling units \citep{binder1983variances, meps_survey_estimation}, and only that column supports a reading about the United States population. Clustering ordinarily widens an interval and stratification ordinarily narrows it, so which dominates is a matter for computation rather than assertion. Race and ethnicity are taken from the standard MEPS analytic variable, whose first level is Hispanic; the alternative race-only variable carries no ethnicity dimension and, in this sample, would place $3{,}758$ Hispanic respondents inside the White category.

Its value as a check is that ZIC-BCF-Smear recovers the documented gradients directly on the dollar-scale conditional effect, and the parametric benchmark, fit to the same sample with the same covariates, does not. The estimated average effect of coverage is $\$188.02$ $[\$156.95, \$220.67]$ as a sample average and $\$184.67$ $[\$154.63, \$214.70]$ as a design-aware population target, positive as the MEPS literature on coverage, visits, and expenditure would require \citep{manski2002private, christian2014determinants, naib2026dental}. The design component is small here, contributing a standard error of $\$1.36$ against a posterior standard deviation of $\$15.26$, so the interval widens only slightly; that is a result rather than an assumption, and it did not have to come out that way.

On the conditional effects, the contrast between non-Hispanic White and non-Hispanic Black beneficiaries is $+\$71.42$ $[\$11.96, \$130.87]$ and the contrast between Hispanic and non-Hispanic White beneficiaries is $-\$62.38$ $[-\$100.02, -\$24.74]$, reproducing the racial and ethnic gradient documented across the oral-health disparities literature \citep{zhang2016racial, reda2018inequality, robison2020racial}; the effect rises monotonically across the five age groups from $\$127$ to $\$250$, with the children-versus-65-and-over contrast at $-\$123.30$ $[-\$204.88, -\$41.73]$, tracking the documented age gradient in dental expenditure and coverage gaps \citep{christian2014determinants, zhang2016racial, manski2015dental}; and the low-income-versus-middle-income contrast is $-\$79.08$ $[-\$145.02, -\$13.14]$, matching the finding that coverage is a weaker enabling resource where cost-sharing and annual caps bite hardest \citep{manski2002private, reda2018inequality, li2020socioeconomic, vujicic2016financial}. In total 18 of the 56 pairwise contrasts are credible under design-aware inference.

Accounting for the complex survey design and contrasting the two margins reveals two key nuances. First, while point estimates indicate lower spending in the South relative to the West (-$\$41.12$ $[-\$87.62, \$5.38]$), the design-aware credible interval covers zero, indicating that this regional gradient is suggestive rather than conclusive \citep{luo2020rural, joudi2024appreciate}. Second, the population-weighted average effect of \$184.67 is essentially identical to the raw unadjusted spending gap of $\$185.9$, exhibiting virtually no attenuation. Instead, the attenuation expected from positive selection into coverage manifests on the participation margin, where the adjusted effect ($0.139$ $[0.124, 0.155]$) absorbs roughly three tenths of the unadjusted 0.201 difference in the probability of dental spending. On the dollar scale, observed covariates absorb little of the raw association, suggesting that they primarily capture sociodemographic background rather than latent oral-health demand. Section~\ref{sec:supp_meps} of the Supplementary Material discusses  this adjustment limitation in detail.

The comparison with the benchmark is not symmetric. Under design-aware inference the Gamma hurdle produces exactly one credible dollar-scale contrast out of 56, and it runs against the literature rather than with it, placing the low-income tertile $\$124.06$ $[\$2.80, \$245.43]$ \emph{above} the high-income tertile. Its race and ethnicity point estimates do not order the groups as three decades of applied work do (non-Hispanic Black $\$144.2$ above non-Hispanic White $\$143.3$ above Hispanic $\$138.4$) \citep{zhang2016racial, reda2018inequality, robison2020racial}, its dollar-scale age gradient is non-monotone, and its regional ordering is inverted, with the South highest at $\$198.7$. On the estimand that both models are asked to report, the conditional effect on the dollar scale, ZIC-BCF-Smear converges to the established applied literature and the Gamma hurdle does not. The mechanism is the one identified in Section~\ref{subsec:nonlinear_results}: the Gamma positive-intensity likelihood imposes a single parametric tail shape that smooths away the heavy right tail in which the literature-aligned heterogeneity lives, whereas the nonparametric smearing re-transformation absorbs it. What rules out the reading that the Gamma hurdle is simply measuring something different is that it does recover credible structure on its participation margin, where its decomposition is well suited, returning 17 credible contrasts to ZIC-BCF-Smear's 11. The parametric benchmark loses the heterogeneity exactly where its distributional assumption binds, which is the scale-misspecification failure of Section~\ref{subsec:smearing} observed against an external standard rather than a known truth. Taken with the simulations, the applied evidence points one way: on hurdle-structured data the proposed estimator recovers heterogeneity that the classical parametric hurdle either misses or inverts.

\section{Limitations}
\label{sec:discussion}

While ZIC-BCF-Smear demonstrates strong empirical performance, it is subject to several limitations that warrant candid acknowledgment. This section is confined to what the framework does not resolve.

\label{subsec:limitations}

First, a \emph{mild under-coverage of CATE credible intervals} is observed under complex, non-linear conditional mean surfaces at larger sample sizes (specifically, $89.7\%$ and $91.7\%$ coverage for $N=500$, and similar behavior at $N=1000$). This coverage dip occurs because the BCF prognostic and treatment forests incorporate highly conservative regularizing priors designed to prevent overfitting. Under highly complex, non-linear link surfaces, these priors can cause the posterior distributions to be overly conservative (shrunk toward the prior mean), leading to slightly narrower credible intervals that do not fully cover the true CATE for units in the tails of the covariate distribution.

Second, the continuous intensity stage (Stage 2) assumes \emph{homoskedastic log-scale errors} ($\epsilon_i \sim \mathcal{N}(0, \sigma_c^2)$), and Duan's smearing factor inherits that assumption: a single scalar per posterior draw is applied to both counterfactual arms, so every response-scale estimate scales linearly in it.

Third, the SPA framework conditions on the post-treatment event $I_i = 1$ ($Y_i > 0$). This causal identification strategy relies on the \emph{subpopulation strong ignorability assumption} (Assumption 3.1). Although this assumption is mathematically guaranteed under full-sample strong ignorability and conditional two-part independence, it may be violated in settings where unobserved confounders simultaneously affect both the hurdle participation process and the continuous response intensity. Under selection on unobservables, the active-subset regression may suffer from residual collider bias that cannot be adjusted for via the subpopulation propensity score alone.

Fourth, both propensity scores are \emph{plug-in estimates}. As described in Section~\ref{subsec:posterior}, $\widehat{\pi}_i$ and $\widehat{\pi}_i^+$ are fit once (in this paper by logistic regression but it could be any other proper estimator) before the sampler starts and are then held fixed, so their estimation uncertainty is cut out of the outcome posterior. This is the same design choice made by BCF itself \citep{hahn2020bcf} and is deliberate, since feeding the propensity score into a joint posterior would allow the flexible outcome model to distort the treatment model. The cost is a mild understatement of posterior uncertainty in settings with weak overlap, where the propensity model is itself poorly determined, and a sensitivity to misspecification of the propensity model that the credible intervals do not display.

Fifth, the estimator targets a contrast in the cumulative outcome as recorded and has \emph{no mechanism for standardizing exposure or follow-up time}. When the outcome accrues over time and the intervention alters how long patients remain on treatment, as in the SPECTRUM trial, treatment duration lies on the causal pathway and the estimand is the effect of assignment under the trial's own differential     duration. That is a coherent target, and under randomization it is identified, but it is not a fixed-horizon restricted mean and it is not a rate per unit of exposure. Section~\ref{subsec:applied_exposure} illustrates how substantially these formulations can diverge in practice. Although trial disposition records document patient-level safety follow-up and last-contact milestones, only 354 patients completed protocol-specified safety follow-up, and last contact does not ensure exhaustive event ascertainment. Extending the two-part framework to a fixed-horizon or multistate formulation therefore requires explicit structural choices regarding censoring mechanisms, competing risks, and intercurrent treatment policies, representing an important avenue for future methodological work on clinical safety endpoints.

Sixth, the sampler has \emph{no interface for survey weights, strata, or primary sampling units}. In the MEPS application the design is therefore handled outside the model: unit-level effect draws are averaged with the person weight to obtain population targets, and a design-based variance component from the sampling strata and primary sampling units is added to the posterior variance \citep{binder1983variances}. This is defensible because every reported quantity is a linear functional of unit-level effects, but it treats sampling and posterior uncertainty as approximately independent and does not let the design influence the fitted response surface itself. A weighted or design-informed prior would be the principled alternative and is not attempted here.

Finally, due to the \emph{extreme computational complexity of the enriched Dirichlet process benchmark (EDP+ZI)}, its inclusion in the main Monte Carlo simulation grid was not possible. EDP+ZI requires $O(N^3)$ matrix operations at each Gibbs sweep to update the Dirichlet process parameters, leading to a $\approx 90 \times$ computational slowdown compared to ZIC-BCF-Smear ($\approx 27.5$ hours versus $\approx 18$ minutes for a single standard fit at $N = 500$), as documented in Section~\ref{subsec:sim_design}. While this computational bottleneck is itself a practical argument for the forest-based approach, the absence of a direct empirical performance comparison against EDP+ZI across the sensitivity cells remains a limitation.

\section{Conclusion}
\label{sec:conclusion}

This paper introduced the Zero-Inflated Continuous Bayesian Causal Forest with Duan's Smearing (ZIC-BCF-Smear), a two-part nonparametric framework for estimating conditional and average treatment effects on semicontinuous outcomes. By decoupling the causal process into a Probit BCF hurdle stage and a Gaussian BCF intensity stage fit strictly on active units, it avoids the computational burden and data augmentation noise that degrade joint selection models; Subpopulation Propensity Adjustment addresses the selection bias inherent in conditioning on active participation; and a non-parametric smearing estimator maps predictions back to the response scale, bypassing the fragile parametric multipliers that cause scaling bias and numerical instability in traditional two-part models.

Comprehensive simulation studies show that ZIC-BCF-Smear matches or outperforms correctly specified parametric hurdle models without committing to any parametric family, and remains stable and accurate under nonlinear surfaces where those benchmarks diverge numerically and reverse the ordering of treatment effects. Two applied studies carry it onto real data. In the SPECTRUM randomized phase III head-and-neck cancer trial, assignment to panitumumab plus chemotherapy credibly increases severe-toxicity burden by about ten days per patient once concurrent adverse events are counted once, and the two-part decomposition places that effect mainly in the positive-intensity part rather than in whether the reconstructed severe-AE burden is positive. That effect is the effect of assignment under the trial's own differential treatment duration, which the study reports explicitly rather than leaving implicit. In the 2023 Medical Expenditure Panel Survey, where a mature applied literature supplies an external standard, the estimator recovers the documented dental-coverage gradients by race and ethnicity, age, and income directly on the expenditure scale, and does so under survey-weighted, design-aware inference, while a parametric Gamma hurdle benchmark fit to the same data reproduces them only on its participation margin.

Two extensions follow directly from the limitations. Heteroskedastic log-scale errors could be accommodated by modeling the residual variance nonparametrically as a separate regression forest, and the framework could be generalized to longitudinal or survival settings with right-censoring. The oncology application motivates the latter most directly, since a prespecified fixed-horizon restricted-mean or multistate formulation would sharpen the safety estimand by standardizing follow-up time and accommodating competing events.

\section*{Acknowledgments}
This publication is based on research using information obtained from www.projectdatasphere.org, which is maintained by Project Data Sphere. Neither Project Data Sphere nor the owner(s) of any information from the web site have contributed to, approved or are in any way responsible for the contents of this publication.

\section*{Data Availability}
The health-economics application uses the 2023 Medical Expenditure Panel Survey Full-Year Consolidated file (HC-251), a public release of the Agency for Healthcare Research and Quality \citep{meps2023file}. It is freely redistributable subject to two conditions of use that survive redistribution, namely that the data be used for statistical reporting and analysis only and that they not be linked with individually identifiable records from any source other than the Medical Expenditure Panel Survey and the National Health Interview Survey from which its sample is drawn. That arm of the study is fully reproducible: the analysis repository accompanying the \texttt{zicbcf} package \citep{zicbcf_package} contains a download-and-prepare script for the public file together with every analysis script and derived results table.

The oncology application uses individual-participant data from the SPECTRUM trial obtained through Project Data Sphere \citep{pds_headne_265}. Those data cannot be redistributed. Section III.B(ii) of the Project Data Sphere Platform Agreement prohibits disclosure to any third party who is not a registered Authorized User, and Article VIII requires that the data be held in strict confidence \citep{pds_platform_agreement}. This is a contractual confidentiality obligation rather than a patient-privacy one, so further de-identification does not make a per-patient file redistributable. Article VII permits disclosure of what is reasonably necessary to publish research results, which is the basis on which the aggregate estimates, tables, and figures reported here appear; per-patient derivatives, including patient-level endpoint files and unit-level posterior draws, are withheld. Access is granted per individual, so any reader wishing to reproduce this arm must register as an Authorized User and request the \texttt{HeadNe\_Amgen\_2007\_265} package directly.

\bibliographystyle{plainnat}
\bibliography{references}

\clearpage
\part*{Supplementary Material}

This supplementary part provides additional details, figures, and tables to support the main manuscript. Section~\ref{sec:supp_dgps} presents the closed-form response-scale potential outcomes and zero-inflation calibration schedules across the simulation designs. Section~\ref{sec:supp_mcmc} summarizes the MCMC algorithms, tree priors, hyperparameters, and diagnostic convergence checks. Section~\ref{sec:supp_timings} presents the details of the computational complexity and timing analysis of the EDP+ZI benchmark model. Section~\ref{sec:supp_full_results} provides the complete set of simulation results, including the full 5-model sensitivity figures displaying the diverged Gamma benchmark lines. Section~\ref{sec:supp_onc} is the self-contained oncology applied study, giving the outcome definitions, posterior analyses, model comparisons, literature positioning, and limitations for the SPECTRUM head-and-neck cancer trial that is summarized in the main manuscript. Section~\ref{sec:supp_meps} is the self-contained health-economics applied study, in which the causal effect of dental-insurance coverage on annual dental expenditure is estimated from the 2023 Medical Expenditure Panel Survey and every estimate is reconciled against the applied dental health-economics literature.

\setcounter{section}{0}
\renewcommand{\thesection}{S\arabic{section}}
\renewcommand{\theHsection}{S\arabic{section}}

\section{True Potential Outcomes and Zero-Inflation Calibration}
\label{sec:supp_dgps}

The baseline covariates, propensity score, hurdle models, and nonlinear prognostic and treatment surfaces are defined in Section~\ref{subsec:dgps} of the main text. This section records the exact closed-form potential-outcome means $\mu_0(X_i) = \mathbb{E}[Y_i(0) \mid X_i]$ and $\mu_1(X_i) = \mathbb{E}[Y_i(1) \mid X_i]$ on the natural response scale used to evaluate estimator accuracy, alongside the intercept calibration schedules that govern the zero-inflation sensitivity grid.

\subsection{Closed-Form Response-Scale Potential Outcomes}
\label{subsec:supp_potential_outcomes}

Under both the link-linear and link-nonlinear regimes, the true conditional average treatment effect on the natural response scale is $\tau(X_i) = \mu_1(X_i) - \mu_0(X_i)$.

\paragraph{DGP A (Log-Normal Hurdle).}
In DGP A, the positive outcome conditional on participation follows a log-normal distribution with log-scale variance $\sigma^2 = 0.25$ ($\sigma = 0.50$). By the properties of the log-normal distribution, the conditional mean of the positive component incorporates the expectation factor $\exp(\sigma^2 / 2) = \exp(0.125)$. Multiplying by the participation probability $p_z(X_i) = P(Y_i(z) > 0 \mid X_i)$, the true response-scale potential outcomes are:
\begin{equation}
\mu_z(X_i) = p_z(X_i) \cdot \exp\!\left( \ell^*(X_i, z) + 0.125 \right), \quad z \in \{0, 1\}.
\end{equation}
Under the link-linear specification, this yields:
\begin{align}
\mu_0(X_i) &= \Phi(c + 0.5 X_{i1} - 0.3 X_{i3}) \cdot \exp(1.625 + 0.8 X_{i2} + 0.4 X_{i4}), \\
\mu_1(X_i) &= \Phi(c + 0.2 + 0.6 X_{i1} - 0.3 X_{i3}) \cdot \exp(1.875 + 0.65 X_{i2} + 0.4 X_{i4}).
\end{align}
Under the link-nonlinear specification, using the nonlinear surfaces $\mu_b^*, \tau_b^*, \mu_c^*, \tau_c^*$ defined in the main text, the potential outcomes take the form:
\begin{align}
\mu_0(X_i) &= \Phi(c + \mu_b^*(X_i)) \cdot \exp(\mu_c^*(X_i) + 0.125), \\
\mu_1(X_i) &= \Phi(c + \mu_b^*(X_i) + \tau_b^*(X_i)) \cdot \exp(\mu_c^*(X_i) + \tau_c^*(X_i) + 0.125).
\end{align}

\paragraph{DGP B (Gamma Hurdle).}
In DGP B, the positive outcome conditional on participation follows a Gamma distribution with shape parameter $\alpha = 2.0$ and conditional scale parameter $\exp(\ell^*(X_i, z)) / \alpha$. Because the expectation of a Gamma variate is the product of its shape and scale parameters, $\mathbb{E}[Y_i^+(z) \mid X_i] = \exp(\ell^*(X_i, z))$ without an additive variance correction. The response-scale potential outcomes are therefore:
\begin{equation}
\mu_z(X_i) = p_z(X_i) \cdot \exp\!\left( \ell^*(X_i, z) \right), \quad z \in \{0, 1\}.
\end{equation}
Under the link-linear specification, this evaluates to:
\begin{align}
\mu_0(X_i) &= \Phi(c + 0.5 X_{i1} - 0.3 X_{i3}) \cdot \exp(1.50 + 0.8 X_{i2} + 0.4 X_{i4}), \\
\mu_1(X_i) &= \Phi(c + 0.2 + 0.6 X_{i1} - 0.3 X_{i3}) \cdot \exp(1.75 + 0.65 X_{i2} + 0.4 X_{i4}).
\end{align}
Under the link-nonlinear specification, the expressions are:
\begin{align}
\mu_0(X_i) &= \Phi(c + \mu_b^*(X_i)) \cdot \exp(\mu_c^*(X_i)), \\
\mu_1(X_i) &= \Phi(c + \mu_b^*(X_i) + \tau_b^*(X_i)) \cdot \exp(\mu_c^*(X_i) + \tau_c^*(X_i)).
\end{align}

\paragraph{DGP C (Compound Poisson--Gamma Tweedie).}
In DGP C, zeros and positive continuous outcomes are generated jointly from an intrinsic Tweedie dispersion model with power parameter $p = 1.5$ and dispersion $\phi = 1.5$. The mean parameter directly governs the marginal expectation on the natural response scale, so $\mu_z(X_i) = \exp(\log \mu_i(z))$. Under the link-linear specification:
\begin{align}
\mu_0(X_i) &= \exp(1.2 + c + 0.8 X_{i1} - 0.4 X_{i3}), \\
\mu_1(X_i) &= \exp(1.5 + c + 0.95 X_{i1} - 0.4 X_{i3}).
\end{align}
Under the link-nonlinear specification, with nonlinear surfaces $\mu_T^*$ and $\tau_T^*$ defined in the main text:
\begin{align}
\mu_0(X_i) &= \exp(c + \mu_T^*(X_i)), \\
\mu_1(X_i) &= \exp(c + \mu_T^*(X_i) + \tau_T^*(X_i)).
\end{align}

\subsection{Zero-Inflation Intercept Calibration Schedules}
\label{subsec:supp_zi_calibration}

In the sensitivity studies reported in Section~\ref{subsec:sensitivity} of the main text, the degree of zero-inflation is systematically varied across five target levels while keeping the underlying prognostic functions and treatment effects fixed. For DGPs A and B, shifting the hurdle intercept $c$ adjusts the marginal probability of participation $\mathbb{E}[\Phi(c + \mu_b^*(X_i) + Z_i \tau_b^*(X_i))]$. For DGP C, shifting the log-mean intercept shifts the compound Poisson rate parameter $\lambda = \mu^{2-p} / ((2-p)\phi)$, calibrating the zero mass $P(Y_i = 0) = \exp(-\lambda)$. Table~\ref{tab:supp_zi_calibration} reports the exact calibration intercepts and empirical zero proportions across all five sensitivity levels.

\begin{table}[htbp]
\centering
\caption{Intercept calibration parameter $c$ and empirical zero proportions across the five zero-inflation sensitivity levels in the simulation studies.}
\label{tab:supp_zi_calibration}
\resizebox{\textwidth}{!}{%
\begin{tabular}{lccccc}
\toprule
& & \multicolumn{2}{c}{DGPs A \& B (Hurdle)} & \multicolumn{2}{c}{DGP C (Tweedie)} \\
\cmidrule(lr){3-4} \cmidrule(lr){5-6}
Configuration & Target Zeros & Link-Linear $c$ & Nonlinear $c$ & Link-Linear $c$ & Nonlinear $c$ \\
\midrule
Level 1 (Severe zero-inflation) & $\approx 85\%$ & $-1.500$ & $-1.367$ & $-3.500$ & $-5.648$ \\
Level 2 (High zero-inflation)   & $\approx 60\%$ & $-0.500$ & $-0.465$ & $-2.000$ & $-3.315$ \\
Level 3 (Standard configuration)& $\approx 40\%$ & $+0.200$ & $+0.200$ & $\phantom{+}0.000$ & $-2.072$ \\
Level 4 (Moderate zero-inflation)& $\approx 15\%$ & $+1.000$ & $+1.005$ & $\phantom{+}0.000$ & $-0.163$ \\
Level 5 (Low zero-inflation)    & $\approx 5\%$  & $+1.800$ & $+1.783$ & $+1.000$ & $+1.131$ \\
\bottomrule
\end{tabular}%
}
\end{table}

\section{MCMC Details \& Diagnostics}
\label{sec:supp_mcmc}

\subsection{Proof of Lemma 3.3 (Identification on the Active Subpopulation)}
\label{subsec:supp_lemma_proof}

The main manuscript states Lemma 3.3 in compressed form. The full statement and its proof are given here.

\begin{lemma*}[Identification on the active subpopulation]
Let $L_i(z) = \log Y_i^+(z)$ for $z \in \{0,1\}$, where both potential positive intensities are defined on the active subpopulation. Suppose $P(I_i=1)>0$ and $E[|L_i(z)|\mid I_i=1]<\infty$ for both treatment arms, and define the active-subpopulation log-scale contrast
\begin{equation}
\tau_c^I(x) = E\!\left[L_i(1) - L_i(0) \mid X_i = x,\ I_i = 1\right].
\end{equation}
Suppose that (i) Assumption 3.1 of the main manuscript holds, (ii) consistency holds on the active subpopulation, that is, $\log Y_i^+ = L_i(Z_i)$ whenever $I_i = 1$, and (iii) there exists a constant $\delta_+ > 0$ such that
\begin{equation}
\delta_+ < \pi^+(x) = P(Z_i = 1 \mid X_i = x,\ I_i = 1) < 1 - \delta_+
\end{equation}
for $P_{X_i\mid I_i=1}$-almost every $x$. Then, for $P_{X_i\mid I_i=1}$-almost every $x$,
\begin{equation}
\tau_c^I(x) = E[\log Y_i^+ \mid X_i = x, Z_i = 1, I_i = 1] - E[\log Y_i^+ \mid X_i = x, Z_i = 0, I_i = 1],
\label{eq:supp_spa_identification}
\end{equation}
so $\tau_c^I(x)$ is identified, up to $P_{X_i\mid I_i=1}$-null sets, from the observed active-subpopulation distribution. If the continuous-stage conditional mean model is correctly specified on the original log-outcome scale, its treatment function $\tau_c(x)$ equals this identified contrast.
\end{lemma*}

\begin{proof}
Work under the conditional distribution given $I_i=1$, which exists because $P(I_i=1)>0$. Since the logarithm is a measurable function on $(0,\infty)$, Assumption 3.1 implies
\begin{equation}
(L_i(0),L_i(1))\independent Z_i\mid X_i,I_i=1.
\end{equation}
For $z\in\{0,1\}$, write
\begin{equation}
m_z(x)=E[L_i(z)\mid X_i=x,I_i=1],\qquad
e_z(x)=P(Z_i=z\mid X_i=x,I_i=1).
\end{equation}
The moment assumption makes $m_0(x)$ and $m_1(x)$ finite almost everywhere. Take a common set of full $P_{X_i\mid I_i=1}$-measure on which these moments, conditional independence, and overlap hold. On this set, $e_1(x)=\pi^+(x)>\delta_+$ and $e_0(x)=1-\pi^+(x)>\delta_+$. Thus both arms have positive conditional probability, and the arm-specific covariate distributions have the same null sets as $P_{X_i\mid I_i=1}$.

For either arm, consistency on $\{Z_i=z,I_i=1\}$ and conditional independence give
\begin{equation}
\begin{aligned}
m_z^{\mathrm{obs}}(x)
&\equiv E[\log Y_i^+\mid X_i=x,Z_i=z,I_i=1] \\
&=\frac{E[\mathbb{I}(Z_i=z)L_i(z)\mid X_i=x,I_i=1]}{e_z(x)} \\
&=\frac{e_z(x)m_z(x)}{e_z(x)}=m_z(x).
\end{aligned}
\end{equation}
The first equality uses consistency and conditioning on an event of conditional probability $e_z(x)>0$; the second uses conditional independence. The same calculation with $|L_i(z)|$ shows that both observed arm means are finite. Consequently, linearity of conditional expectation yields
\begin{equation}
\begin{aligned}
\tau_c^I(x)
&=m_1(x)-m_0(x) \\
&=m_1^{\mathrm{obs}}(x)-m_0^{\mathrm{obs}}(x),
\end{aligned}
\end{equation}
which is \eqref{eq:supp_spa_identification}. Each observed mean is determined by the active-subpopulation law, so this identifies the contrast almost everywhere. Conditional distributions at covariate values of zero probability can be changed without changing the observed law; identification at every point of the support would require additional regularity.

Finally, let $q(x)$ denote a fixed propensity-score input to the prognostic function, such as the population score $\pi^+(x)$ or a fixed estimate $\widehat{\pi}^+(x)$. Under a correctly specified continuous-stage conditional mean model on the original log-outcome scale,
\begin{equation}
m_z^{\mathrm{obs}}(x)=\mu_c(x,q(x))+z\tau_c(x).
\end{equation}
The prognostic term is the same for both arms at a fixed $x$, so subtracting the control mean from the treated mean gives $\tau_c(x)=\tau_c^I(x)$ almost everywhere.
\end{proof}

\begin{remark*}
Applying Theorem 1 of \citet{rosenbaumrubin1983propensity} to the conditional distribution given $I_i = 1$ gives $X_i\independent Z_i\mid\pi^+(X_i),I_i=1$, so $\pi^+(X_i)$ is a balancing score within the active subpopulation. Under Assumption 3.1 and active-subpopulation overlap, their Theorem 3 then gives $(Y_i^+(0),Y_i^+(1))\independent Z_i\mid\pi^+(X_i),I_i=1$. Because $\pi^+(X_i)$ is a function of $X_i$, however, supplying it to the prognostic forest $\mu_c(\cdot)$ changes the model's regularization without adding any identifying condition beyond (i)--(iii). In particular it does not substitute for Assumption 3.1, for active-subpopulation overlap, or for adjustment on $X_i$ itself when the target contrast is heterogeneous in $X_i$.
\end{remark*}

\subsection{Posterior Sampling Details}
The posterior distribution of the ZIC-BCF-Smear model is obtained from a single Gibbs sampler in which the hurdle block and the intensity block are updated in turn within each sweep, and are written to a common draw index after burn-in. The two blocks are as follows.
\begin{enumerate}
    \item \textbf{Hurdle Stage (Probit BCF):} Let $I_i = \mathbb{I}(Y_i > 0)$ be the binary response indicator. Latent Gaussian utility variables $I_i^* \sim \mathcal{N}(\mu_{b}(X_i, \widehat{\pi}_i) + Z_i \tau_b(X_i), 1)$ are introduced such that $I_i = \mathbb{I}(I_i^* > 0)$. Using the data augmentation algorithm of \citet{albert1993bayesian}, the latent variables $I_i^*$ are sampled from truncated normal distributions:
    \begin{align}
    (I_i^* \mid I_i = 1) &\sim \mathcal{N}\left(\mu_{b}(X_i, \widehat{\pi}_i) + Z_i \tau_b(X_i), 1\right) \text{ truncated to } (0, \infty) \\
    (I_i^* \mid I_i = 0) &\sim \mathcal{N}\left(\mu_{b}(X_i, \widehat{\pi}_i) + Z_i \tau_b(X_i), 1\right) \text{ truncated to } (-\infty, 0]
    \end{align}
    Given the latent outcomes $I_i^*$, the regression forests $\mu_b(\cdot)$ and $\tau_b(\cdot)$ are updated using standard Bayesian backfitting updates for BART \citep{chipman2010bart}.
    
    \item \textbf{Continuous Stage (Gaussian BCF with SPA):} On the active subpopulation $\{i \mid Y_i > 0\}$, the log outcomes $y_{i,\text{active}} = \log(Y_i^+)$ are computed and standardized to $y^*_{i,\text{active}}$. The continuous BCF model is fit on these standardized outcomes, conditioning on the active-set propensity score $\widehat{\pi}_i^+$. Updates for the forests $\mu_c(\cdot)$, $\tau_c(\cdot)$, and the residual variance $\sigma_c^2$ are executed using Gaussian backfitting steps.
\end{enumerate}

\subsubsection{Exactness of the Block-Wise Sweep}
\label{subsec:supp_joint_mcmc}

Updating the two blocks one after the other within a sweep, rather than through a joint proposal, is not an approximation. The two-part likelihood factorizes exactly,
\begin{equation}
p\!\left(\{Y_i\} \mid \theta_b, \theta_c\right) = \underbrace{\prod_{i=1}^{N} \Phi(\mu_{bi} + Z_i \tau_{bi})^{I_i}\left[1 - \Phi(\mu_{bi} + Z_i \tau_{bi})\right]^{1 - I_i}}_{\text{participation, all } N \text{ units}} \; \times \underbrace{\prod_{i : I_i = 1} p\!\left(Y_i^+ \mid \theta_c\right)}_{\text{intensity, active units only}},
\label{eq:supp_factorization}
\end{equation}
because the participation indicators $\{I_i\}$ are a deterministic function of the observed data and the two blocks share no parameter. Combined with independent BART priors on $\theta_b = (\mu_b, \tau_b)$ and $\theta_c = (\mu_c, \tau_c, \sigma_c^2)$, \eqref{eq:supp_factorization} gives $p(\theta_b, \theta_c \mid \text{data}) = p(\theta_b \mid \text{data})\, p(\theta_c \mid \text{data})$. The blocks are therefore \emph{a posteriori} independent, the block-wise sweep is an exact Gibbs sampler for the joint posterior, and nothing is lost relative to a hypothetical joint update.

Two practical consequences follow. First, there is nothing to gain from coupling the blocks more tightly than the sampler already does. Second, there is something to lose from decoupling them further: running two separate chains and pairing their draws arbitrarily would be valid only because of the same factorization, while forfeiting the shared random state and the shared convergence diagnostics that make the composite functional $\tau_{\text{resp}}(X_i)$, a product of a participation probability and an intensity, straightforward to summarize.

The one place where an approximation genuinely enters is the propensity scores. The full-sample $\widehat{\pi}_i$ and the subpopulation $\widehat{\pi}_i^+$ are estimated once by logistic regression before the sampler starts and are then held fixed as covariates, so their estimation uncertainty is not propagated into the outcome posterior. This is a deliberate cut, and it is the same cut made by BCF itself \citep{hahn2020bcf}: feeding the propensity score back into a joint posterior would let the flexible outcome model distort the treatment model, and the plug-in estimate is what makes $\widehat{\pi}_i$ a stable basis function for the prognostic ensemble rather than a parameter competing with it. The cost is a mild understatement of posterior uncertainty in settings with weak overlap, where the propensity model is itself poorly determined.

\subsection{Complete Model Fitting Algorithm}
\label{subsec:supp_algorithm}
For reference, the full ZIC-BCF-Smear fitting and re-transformation procedure summarized in the main manuscript is stated below as pseudocode.

\begin{table}[htbp]
\centering
\begin{minipage}{0.95\textwidth}
\hrule
\vspace{1ex}
\textbf{Algorithm 1: ZIC-BCF-Smear MCMC and Smearing Retransformation}
\vspace{1ex}
\hrule
\vspace{1ex}
\textbf{Input:} Semicontinuous outcomes $Y_i \ge 0$, binary treatments $Z_i \in \{0, 1\}$, design matrices $X_i$. \\
\textbf{Output:} Posterior draws of response-scale potential outcomes $\mu_0^{(s)}(X_i)$, $\mu_1^{(s)}(X_i)$, and CATE $\tau_{\text{resp}}^{(s)}(X_i)$.
\begin{enumerate}
\item Fit the full-sample propensity score $\widehat{\pi}_i = P(Z_i = 1 \mid X_i)$.
\item Identify the active subpopulation $\{i : Y_i > 0\}$. Fit the subpopulation propensity score $\widehat{\pi}_i^+ = P(Z_i = 1 \mid X_i, Y_i > 0)$ on this subset.
\item Run the Probit BCF model on the full sample to obtain $S$ posterior MCMC draws of the latent hurdle functions $\mu_{bi}^{(s)}$ and $\tau_{bi}^{(s)}$.
\item Run the Gaussian BCF model on the active subset for the log-outcomes $\log(Y_i^+)$ to obtain $S$ posterior MCMC draws of the continuous intensity functions $\mu_{ci}^{(s)}$, $\tau_{ci}^{(s)}$, and log-residual variance $\sigma_c^{(s)}$.
\item For each saved posterior MCMC draw $s \in \{1, \dots, S\}$:
\begin{enumerate}
    \item Compute active log-scale residuals: $e_j^{(s)} = \log(Y_j^+) - [\mu_{cj}^{(s)} + Z_j \tau_{cj}^{(s)}]$ for all $j \in \text{active}$.
    \item Compute the smearing factor: $\phi_{\text{smear}}^{(s)} = \frac{1}{N_{\text{active}}} \sum_{j \in \text{active}} \exp(e_j^{(s)})$.
    \item Re-transform predictions for all $N$ units to the response scale:
    \begin{align*}
    \mu_0^{(s)}(X_i) &= \Phi(\mu_{bi}^{(s)}) \cdot \exp(\mu_{ci}^{(s)}) \cdot \phi_{\text{smear}}^{(s)} \\
    \mu_1^{(s)}(X_i) &= \Phi(\mu_{bi}^{(s)} + \tau_{bi}^{(s)}) \cdot \exp(\mu_{ci}^{(s)} + \tau_{ci}^{(s)}) \cdot \phi_{\text{smear}}^{(s)}
    \end{align*}
    \item Compute CATE draws: $\tau_{\text{resp}}^{(s)}(X_i) = \mu_1^{(s)}(X_i) - \mu_0^{(s)}(X_i)$.
\end{enumerate}
\end{enumerate}
\hrule
\end{minipage}
\end{table}

\subsection{Tree Prior Hyperparameter Choices}
Hyperparameters are set to standard default values in the literature \citep{hahn2020bcf}.
\begin{enumerate}
    \item \textbf{Forest Sizes:} The prognostic forests ($\mu_b, \mu_c$) are allocated $T_{\text{control}} = 200$ trees. The treatment-moderating forests ($\tau_b, \tau_c$) are allocated $T_{\text{moderate}} = 50$ trees, regularizing treatment effects to prevent overfitting.
    \item \textbf{Tree Structure Prior:} For any node at depth $d$, the probability of splitting is given by $\alpha(1 + d)^{-\beta}$. The settings are:
    \begin{enumerate}
        \item Prognostic forests: $\alpha_c = 0.95, \beta_c = 2.0$.
        \item Moderating forests: $\alpha_m = 0.25, \beta_m = 3.0$.
    \end{enumerate}
    The much smaller $\alpha_m$ and larger $\beta_m$ heavily penalize deep trees for the treatment effect, reflecting the prior that treatment effect heterogeneity is smooth and low-dimensional.
    \item \textbf{Leaf Parameters Scale Prior:} The leaf parameters are assigned normal priors $\mathcal{N}(0, \sigma^2_{\text{leaf}})$.
    \begin{enumerate}
        \item Hurdle stage: the prior medians of the cumulative forest scales are set to $\sigma_{\mu_b} = 2.0$ (latent scale) and $\sigma_{\tau_b} = 1.0$.
        \item Continuous stage: prior medians are scaled data-adaptively: $\sigma_{\mu_c} = 2.0 \cdot \text{sd}(\log(Y_{\text{active}}))$ and $\sigma_{\tau_c} = 0.25 \cdot \text{sd}(\log(Y_{\text{active}})) / \text{sd}(Z_{\text{active}})$.
    \end{enumerate}
    \item \textbf{Variance Prior:} The continuous error variance is assigned a conjugate scaled inverse chi-squared prior $\sigma_c^2 \sim \text{Inv-}\chi^2(\nu, \lambda)$ with degrees of freedom $\nu = 3$. The scale parameter $\lambda$ is calibrated such that the prior 90th percentile of $\sigma_c$ matches the residual standard deviation from a linear model fit on the active subset.
\end{enumerate}

\subsection{MCMC Convergence Diagnostics}
\label{subsec:supp_convergence}

Convergence is assessed rather than assumed, and the assessment is part of the software rather than a separate manual exercise. The \texttt{zicbcf} package exports \nolinkurl{zicbcf_convergence()}, which takes one or more fitted chains and returns, for every monitored functional, the Gelman--Rubin potential scale reduction factor $\widehat{R}$ \citep{gelman1992inference}, the effective sample size, the Monte Carlo standard error, and the Geweke spectral-density statistic comparing the first tenth of each chain with its last half \citep{geweke1992evaluating}. All three are obtained from \texttt{coda} \citep{plummer2006coda}. When a single chain is supplied, $\widehat{R}$ is the split-chain variant, so within-chain drift is still detectable; when several chains are supplied it is the split version of the across-chain statistic. A companion function, \nolinkurl{zicbcf_convergence_table()}, applies the same three diagnostics to any set of posterior draws, and is used for the Gamma-hurdle benchmark, whose fits are not \texttt{zicbcf} objects.

The monitored set is deliberately not restricted to the average treatment effect, because a two-part model can mix well on an aggregate while mixing poorly on a component. It comprises 1. the response-scale ATE, 2. the participation-margin ATE, 3. Duan's smearing factor, 4. the log-scale residual standard deviation $\sigma_c$, 5. the cross-unit means of the four forest surfaces $\mu_b$, $\tau_b$, $\mu_c$, and $\tau_c$, and 6. ten unit-level CATEs chosen at evenly spaced quantiles of the posterior-mean effect distribution, so that the monitored units span the estimated effect surface rather than an arbitrary corner of it.

\subsubsection{Simulation Grid}
Diagnostics for the simulation study are produced by \path{simulation_studies/run_convergence_diagnostics.R}, which fits four independent chains of $1{,}000$ retained draws after $1{,}000$ burn-in iterations to five replicates of each of the six data-generating processes at $N = 500$, and writes the per-quantity table together with a per-DGP summary. Table~\ref{tab:supp_sim_convergence} reports the summary.

\begin{table}[htbp]
\centering
\caption{Convergence diagnostics for the simulation grid: four chains of 1,000 retained draws on each of five replicates per data-generating process at $N=500$. Global summaries are the eight scalar functionals listed above; unit-level CATEs are ten units spanning the effect surface. The Geweke column is the largest absolute value over all chains and quantities in the cell.}
\label{tab:supp_sim_convergence}
\resizebox{\textwidth}{!}{%
\begin{tabular}{llrrrrrr}
\toprule
Data-generating process & Monitored set & Median $\widehat{R}$ & 95th pct.\ $\widehat{R}$ & Max $\widehat{R}$ & Median ESS & Min ESS & Max $|z_{\mathrm{Geweke}}|$ \\
\midrule
\multirow{2}{*}{Linear DGP A (log-normal hurdle)} & Global summaries & 1.004 & 1.014 & 1.017 & 928 & 449 & 4.05 \\
 & Unit-level CATEs & 1.005 & 1.012 & 1.048 & 1{,}125 & 508 & 3.68 \\
\multirow{2}{*}{Linear DGP B (Gamma hurdle)} & Global summaries & 1.002 & 1.007 & 1.010 & 1{,}124 & 444 & 2.57 \\
 & Unit-level CATEs & 1.003 & 1.010 & 1.022 & 1{,}563 & 693 & 3.53 \\
\multirow{2}{*}{Linear DGP C (Tweedie)} & Global summaries & 1.004 & 1.038 & 1.118 & 1{,}113 & 203 & 4.08 \\
 & Unit-level CATEs & 1.004 & 1.011 & 1.030 & 1{,}489 & 869 & 4.91 \\
\multirow{2}{*}{Nonlinear DGP A (log-normal hurdle)} & Global summaries & 1.008 & 1.050 & 1.131 & 668 & 159 & 7.06 \\
 & Unit-level CATEs & 1.009 & 1.027 & 1.074 & 890 & 396 & 4.99 \\
\multirow{2}{*}{Nonlinear DGP B (Gamma hurdle)} & Global summaries & 1.003 & 1.012 & 1.015 & 888 & 451 & 4.20 \\
 & Unit-level CATEs & 1.006 & 1.015 & 1.026 & 1{,}398 & 532 & 4.13 \\
\multirow{2}{*}{Nonlinear DGP C (Tweedie)} & Global summaries & 1.005 & 1.074 & 1.162 & 1{,}086 & 161 & 4.46 \\
 & Unit-level CATEs & 1.004 & 1.009 & 1.016 & 1{,}646 & 966 & 3.43 \\
\bottomrule
\end{tabular}
}
\end{table}

The picture Table~\ref{tab:supp_sim_convergence} gives is not uniformly clean, and it should not be reported as though it were. Median $\widehat{R}$ never exceeds $1.009$ in any cell, and in the three link-linear cells and the nonlinear Gamma cell the 95th percentile stays below $1.015$. The two hardest cells, nonlinear DGP A and nonlinear DGP C, carry a small tail of poorly mixed quantities: the 95th percentile of $\widehat{R}$ reaches $1.050$ and $1.074$ and the maximum reaches $1.131$ and $1.162$, with effective sample sizes as low as 159 out of 4,000 pooled draws. Those values attach to global forest-scale summaries rather than to unit-level effects, and they occur in exactly the link-nonlinear configurations that the main manuscript identifies as most demanding. The practical reading is that a budget of 1,000 retained draws is adequate for the reported Monte Carlo summaries, whose accuracy depends on averaging over 100 replicates rather than on any single chain, but is not generous for a single applied fit. That is why both applied studies use four chains and a larger per-chain budget, and report $\widehat{R}$ below $1.01$ throughout.

\subsubsection{Applied Studies}
For the final SPECTRUM fits, the summed-duration endpoint has maximum $\widehat{R}=1.001$, minimum effective sample size $2{,}670$, and maximum absolute Geweke statistic $2.03$ for ZIC-BCF-Smear, while the Gamma hurdle has corresponding values $1.000$, $16{,}000$, and $1.42$. On the distinct-calendar-day endpoint, the corresponding values are $1.002$, $2{,}777$, and $2.21$ for ZIC-BCF-Smear and $1.000$, $16{,}000$, and $1.42$ for the Gamma hurdle. All of these sit well inside the calibrated null ranges at the multiplicity that the monitored set implies, and the final ZIC-BCF-Smear fits use the 5,000-iteration burn-in motivated in Section~\ref{sec:supp_models} (the originally reported 2,000-iteration configuration produced largest statistics of $3.67$ on the summed-duration endpoint and $4.90$ on the distinct-calendar-day endpoint, both on $\sigma_c$). The MEPS fits use the same 5,000-iteration burn-in, and their diagnostics are reported below.

\subsection{Homoskedasticity of the Log-Scale Errors}
\label{subsec:supp_smearing_test}

Duan's re-transformation replaces the parametric multiplier $\exp(\sigma_c^2/2)$ with the non-parametric factor
\begin{equation}
\widehat{\phi}^{(s)} = \frac{1}{n_{+}}\sum_{i \in \mathcal{A}} \exp\!\left(\widehat{\varepsilon}_i^{(s)}\right),
\qquad
\widehat{\varepsilon}_i^{(s)} = \log Y_i - \mu_c^{(s)}(X_i) - Z_i \tau_c^{(s)}(X_i),
\label{eq:supp_smearing_factor}
\end{equation}
one scalar per retained draw, applied identically to both counterfactual arms. Equation~\eqref{eq:supp_smearing_factor} is consistent for the response-scale mean when the log-scale errors are independent of the covariates. Under heteroskedasticity the correct multiplier varies with $x$, and because both potential-outcome means are multiplied by the same scalar, every response-scale estimate scales linearly in it. The assumption is therefore consequential and is tested.

The package exports \nolinkurl{zicbcf_smearing_diagnostics()}, which returns three things. First, a studentized Breusch--Pagan test \citep{breusch1979simple, koenker1981note} regressing the squared log-scale residuals on the fitted log-scale mean and on the covariates; the studentized form is used because the raw Breusch--Pagan statistic assumes normal errors, which is precisely what a smearing re-transformation exists to avoid assuming. Second, the same test recomputed inside each of a thinned set of retained draws, so the evidence is summarized by a posterior distribution of $p$-values rather than by a single plug-in number. Third, and most useful in practice, a locally smeared sensitivity analysis that replaces the single scalar with stratum-specific factors over quantiles of the fitted log-scale mean and recomputes the average treatment effect on that basis. When the locally smeared estimate is close to the scalar-smeared benchmark, homoskedasticity violations do not materially distort the reported effect regardless of formal test significance.

Table~\ref{tab:supp_sim_smearing} applies this to the simulation grid, where the truth is known. DGP A generates log-normal errors with a constant log-scale variance and therefore acts as a negative control; the Gamma and Tweedie generators do not, and act as positive ones.

\begin{table}[htbp]
\centering
\caption{Homoskedasticity of the log-scale errors across the simulation grid, five replicates per data-generating process at $N=500$. The residual-scale ratio is the largest divided by the smallest residual standard deviation across five strata of the fitted log-scale mean. The relative change is the locally smeared average treatment effect divided by the scalar-smeared one, minus one.}
\label{tab:supp_sim_smearing}
\resizebox{\textwidth}{!}{%
\begin{tabular}{lrrrrr}
\toprule
Data-generating process & Median BP $p$ & Rejected at $5\%$ & Residual-scale ratio & Median relative change & Max $|$relative change$|$ \\
\midrule
Linear DGP A (log-normal hurdle) & 0.194 & 1 of 5 & 1.23 & $+0.9\%$ & $3.1\%$ \\
Linear DGP B (Gamma hurdle) & 0.026 & 3 of 5 & 1.21 & $-4.4\%$ & $8.5\%$ \\
Linear DGP C (Tweedie) & $2.9 \times 10^{-7}$ & 5 of 5 & 2.12 & $-17.7\%$ & $19.5\%$ \\
Nonlinear DGP A (log-normal hurdle) & 0.313 & 0 of 5 & 1.21 & $+1.3\%$ & $2.6\%$ \\
Nonlinear DGP B (Gamma hurdle) & 0.242 & 1 of 5 & 1.27 & $+2.5\%$ & $7.7\%$ \\
Nonlinear DGP C (Tweedie) & $7.0 \times 10^{-4}$ & 5 of 5 & 1.54 & $-25.1\%$ & $32.5\%$ \\
\bottomrule
\end{tabular}
}
\end{table}

The pattern is the one the diagnostic is meant to produce. Where homoskedasticity holds by construction the test does not reject and the locally smeared estimate is within a few percent of the scalar-smeared one. Where it fails, the test rejects decisively and the locally smeared estimate moves materially, which is the quantity a reader actually needs in order to judge how much the assumption is doing. The applied results are reported with their studies, in Sections~\ref{sec:supp_onc} and~\ref{sec:supp_meps}.

\section{EDP+ZI Timing Analysis}
\label{sec:supp_timings}

The Enriched Dirichlet Process Zero-Inflated (EDP+ZI) model \citep{kim2024edpzi} is a flexible Bayesian nonparametric benchmark that models semicontinuous data by placing a Dirichlet process mixture prior over the joint distribution of hurdle and continuous intensity coefficients. However, its high computational complexity limits its practical utility for large datasets or simulation studies.

\subsection{Computational Bottleneck of Dirichlet Process Mixtures}
In EDP+ZI, observations are dynamically grouped into latent clusters. At each Gibbs sweep, the cluster assignments and the cluster-specific regression parameters are updated. For a dataset of size $N$, let $N_k$ be the size of cluster $k$. Updating the multivariate regression parameters for cluster $k$ requires inverting the covariance matrix of the design matrix, which has dimension $P \times P$ (where $P$ is the number of covariates). While this step is relatively fast, updating cluster assignments and evaluating marginal likelihoods for each observation involves computing multivariate normal densities that scale with the cluster sample size. In more general formulations (or when handling latent variables for zero-inflated outcomes), Gibbs updates require matrix operations that scale cubically with the sample size, $O(N^3)$, due to joint multivariate Normal updates or Gaussian Process approximations. 

In contrast, BCF models represent nonparametric functions using sum-of-trees. The computational cost of fitting a single regression tree of depth $d$ on a dataset of size $N$ with $P$ covariates scales as $O(d P N \log N)$, which is dominated by sorting the covariates at each split. For a forest of $T$ trees, the computational complexity per MCMC sweep is:
\begin{equation}
\mathcal{O}(T \cdot P \cdot N \log N)
\end{equation}
This linear-logarithmic scaling allows BCF models to scale efficiently to very large datasets, whereas the cubic cost of EDP+ZI becomes prohibitive.

\subsection{Empirical Timing Comparisons}
The EDP+ZI model was implemented and verified using the authors' reference code. The empirical runtimes on standard CPU threads are summarized below:
\begin{enumerate}
    \item \textbf{ZIC-BCF-Smear ($N=500$):} A single simulation run (1,000 burn-in + 1,000 sampling iterations) takes approximately \textbf{18 minutes} ($\approx 0.3$ hours). A complete set of 100 seeds for a single cell takes $\approx 30$ hours (parallelized to less than 5 hours on a 6-core machine).
    \item \textbf{EDP+ZI ($N=500$):} A single run takes approximately \textbf{27.5 hours}. This represents a $\bm{\approx 90\times}$ slowdown compared to a single run of ZIC-BCF-Smear.
    \item \textbf{EDP+ZI ($N=1000$):} Due to the $O(N^3)$ complexity bottleneck, scaling the sample size from $N=500$ to $N=1000$ increases the runtime by a factor of approximately $(1000/500)^3 = 8$. A single fit of EDP+ZI at $N=1000$ is estimated to take $\approx 220$ hours (\textbf{more than 9 days}) of continuous single-threaded execution.
\end{enumerate}

To run the full simulation grid for EDP+ZI (comprising 3 DGPs, 4 sample sizes, 4 zero-inflation levels, and 100 seeds) would require over 108,000 CPU hours (approximately 12 years of continuous runtime on a single machine). Given that standard academic cloud environments (like Google Colab) cap continuous sessions at 12 hours, running EDP+ZI across the full design was computationally infeasible. Consequently, while the model was verified for correctness, it was excluded from the empirical benchmark. This highlights that computational efficiency is a primary methodological advantage of ZIC-BCF-Smear.

\section{Complete Simulation Results \& Full Figures}
\label{sec:supp_full_results}

The main manuscript summarizes both sensitivity sweeps in prose and reports no sensitivity figures, because the numerical divergence (overflow) of the parametric models under misspecification generates extreme outlier values (RMSE up to $10^{11}$), which swamps the y-axis and makes visual comparison of the remaining stable models impossible. Both variants are therefore given here: the full 5-model figures including all benchmarks and diverged runs in this subsection and the next, and the legible variants restricted to the numerically stable estimators in Section~\ref{subsec:supp_nonlinear_stable}.

\subsection{Linear Sensitivity Plots}
Under link-linear DGPs, the parametric Gamma Hurdle model is correctly specified. However, the naive Gamma + .01 model (which adds a constant of 0.01 to the zero outcomes) still suffers from severe bias and numerical instability.
\begin{enumerate}
    \item Figure~\ref{fig:linear_n_sens_full} displays the full $N$-sensitivity plots for DGP A, B, and C. The naive Gamma + .01 model exhibits large RMSE values at small sample sizes ($N=100$) where the shift introduces massive bias relative to the small active set.
    \item Figure~\ref{fig:linear_zi_sens_full} displays the full zero-inflation sensitivity plots. As the zero proportion increases to $85\%$, the performance of the parametric Gamma Hurdle begins to degrade because the active sample size contracts, while Gamma + .01 collapses entirely.
\end{enumerate}

\subsection{Nonlinear Sensitivity Plots}
Under link-nonlinearity, both parametric models (Gamma Hurdle and Gamma + .01) suffer from severe functional-form misspecification. The resulting out-of-sample prediction overflows lead to massive divergence in 4\%--9\% of the seeds, inflating the mean CATE and ATE RMSEs.
\begin{enumerate}
    \item Figure~\ref{fig:nonlinear_n_sens_full} displays the full $N$-sensitivity plots under nonlinearity. The y-axis scales are dominated by the diverged parametric model metrics, which reach values of $10^8$ to $10^{10}$, showing the extreme risk of relying on parametric log-link models when the true surface is nonlinear.
    \item Figure~\ref{fig:nonlinear_zi_sens_full} displays the full zero-inflation sensitivity plots under nonlinearity, illustrating that the numerical instability of the parametric benchmarks is present across all levels of zero-inflation.
\end{enumerate}

\begin{figure}[htbp]
\centering
\begin{tabular}{@{}cc@{}}
\includegraphics[width=0.48\textwidth]{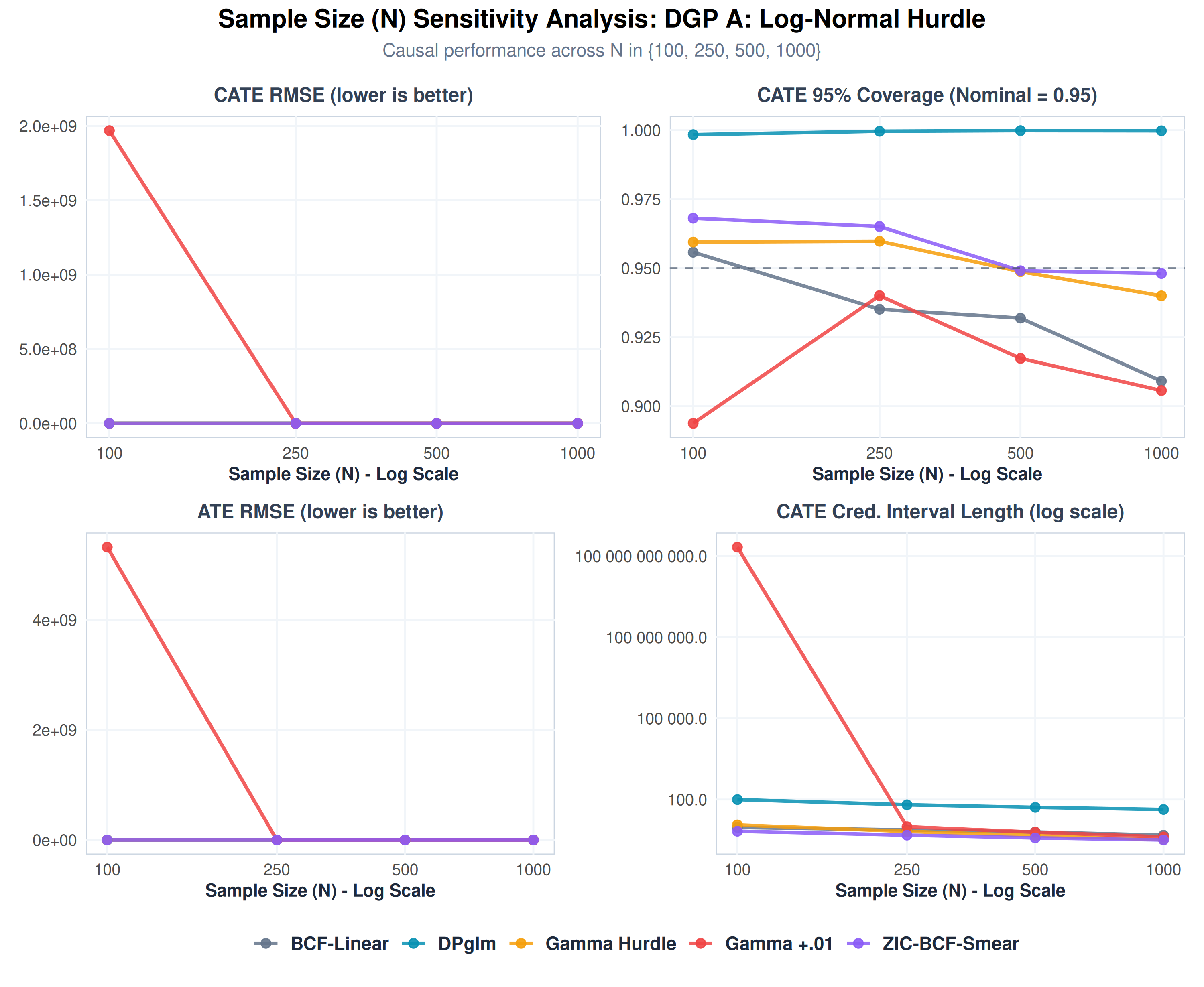} &
\includegraphics[width=0.48\textwidth]{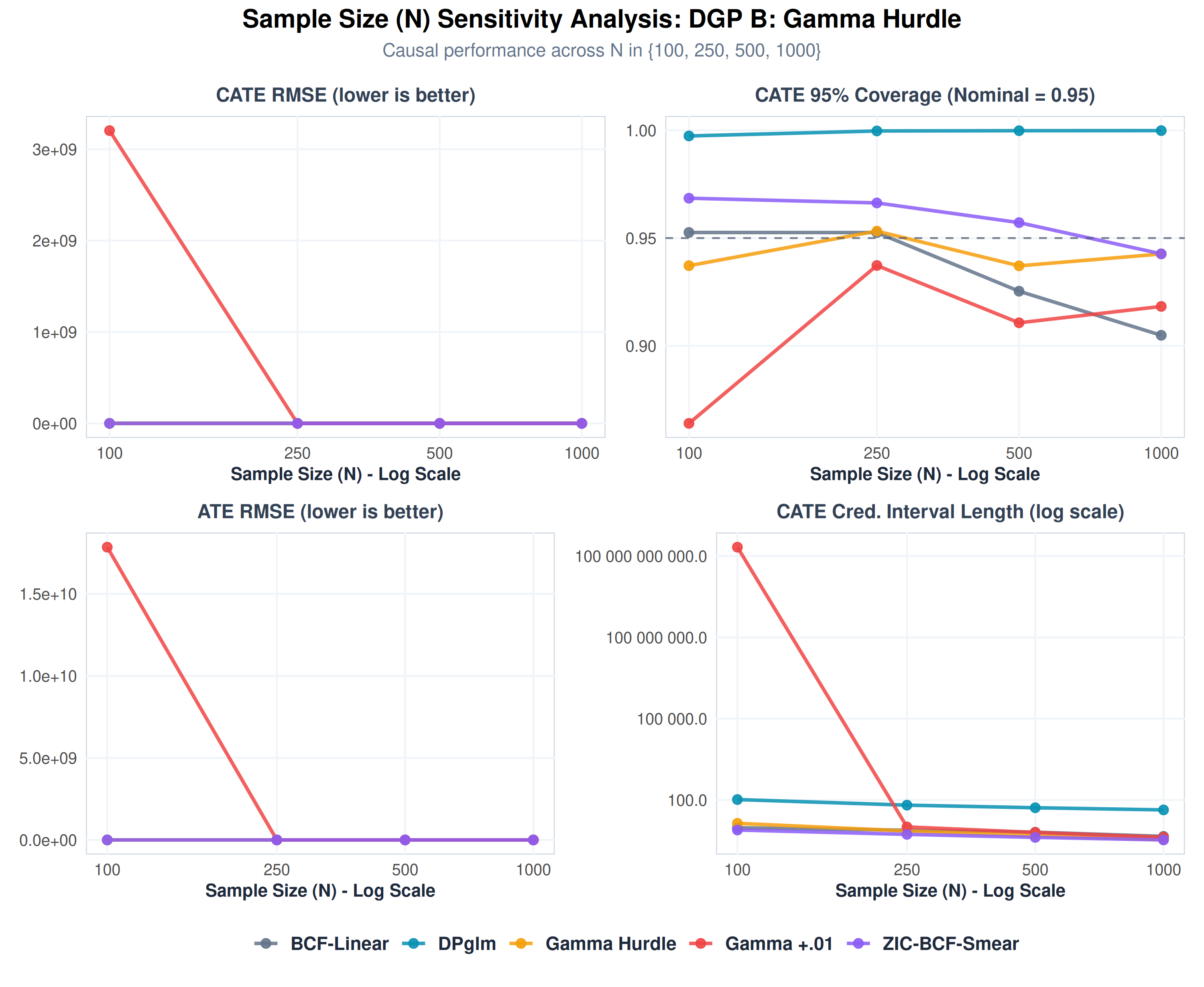} \\
(a) DGP A: Log-Normal Hurdle & (b) DGP B: Gamma Hurdle \\
\multicolumn{2}{c}{\includegraphics[width=0.48\textwidth]{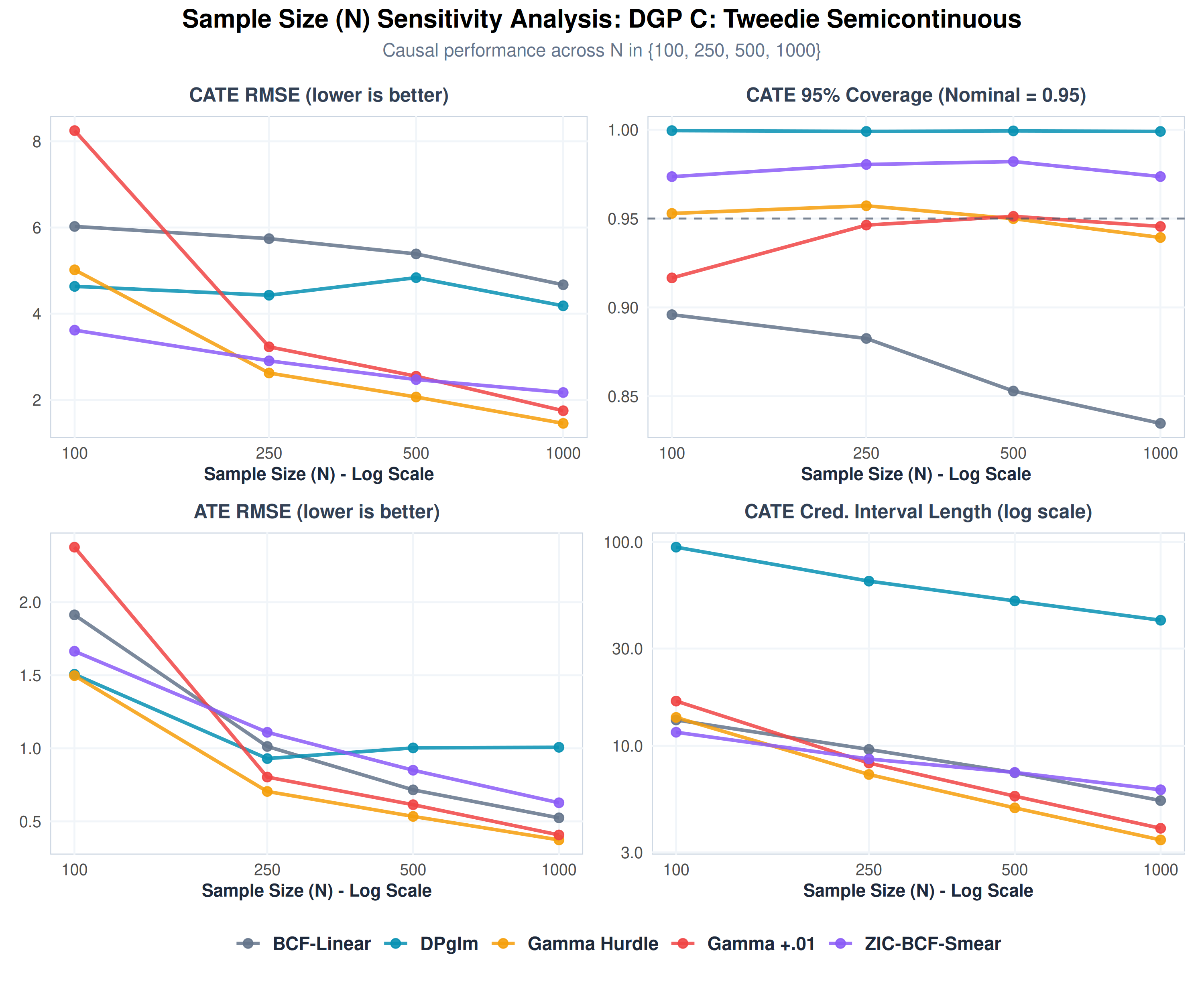}} \\
\multicolumn{2}{c}{(c) DGP C: Tweedie Semicontinuous}
\end{tabular}
\caption{Full link-linear $N$-sensitivity analysis (including the diverged Gamma + .01 benchmark) across DGP A, B, and C.}
\label{fig:linear_n_sens_full}
\end{figure}

\begin{figure}[htbp]
\centering
\begin{tabular}{@{}cc@{}}
\includegraphics[width=0.48\textwidth]{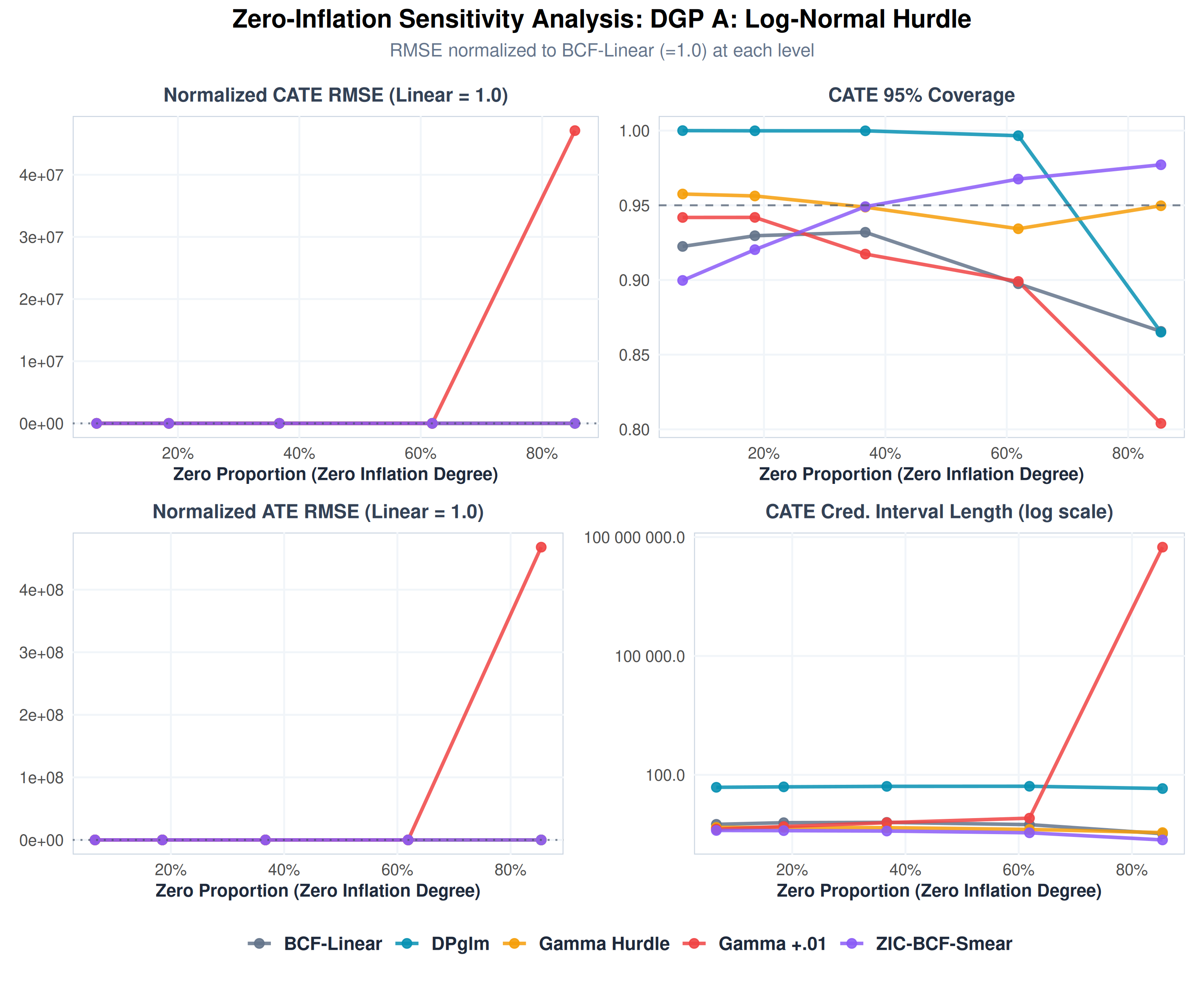} &
\includegraphics[width=0.48\textwidth]{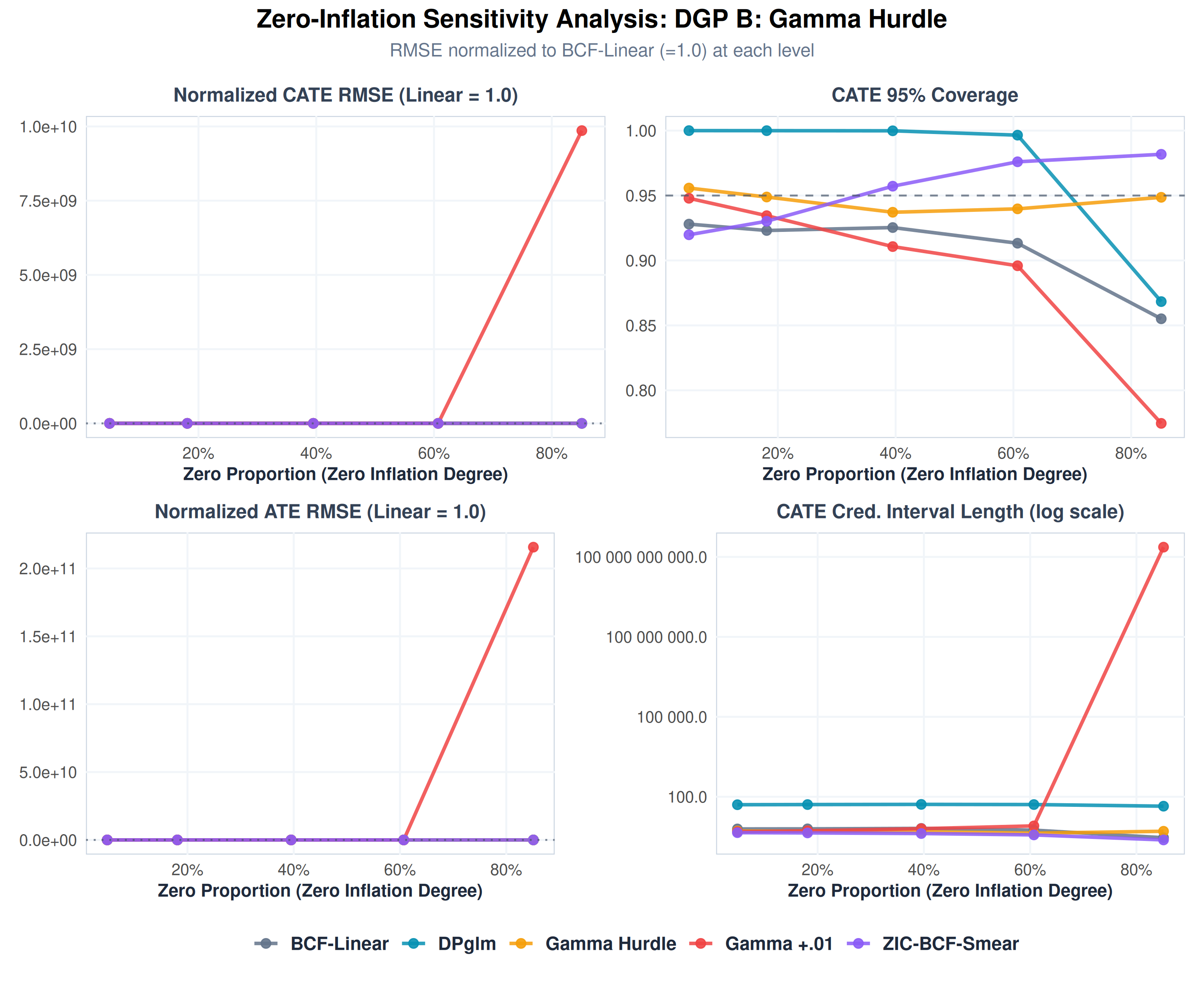} \\
(a) DGP A: Log-Normal Hurdle & (b) DGP B: Gamma Hurdle \\
\multicolumn{2}{c}{\includegraphics[width=0.48\textwidth]{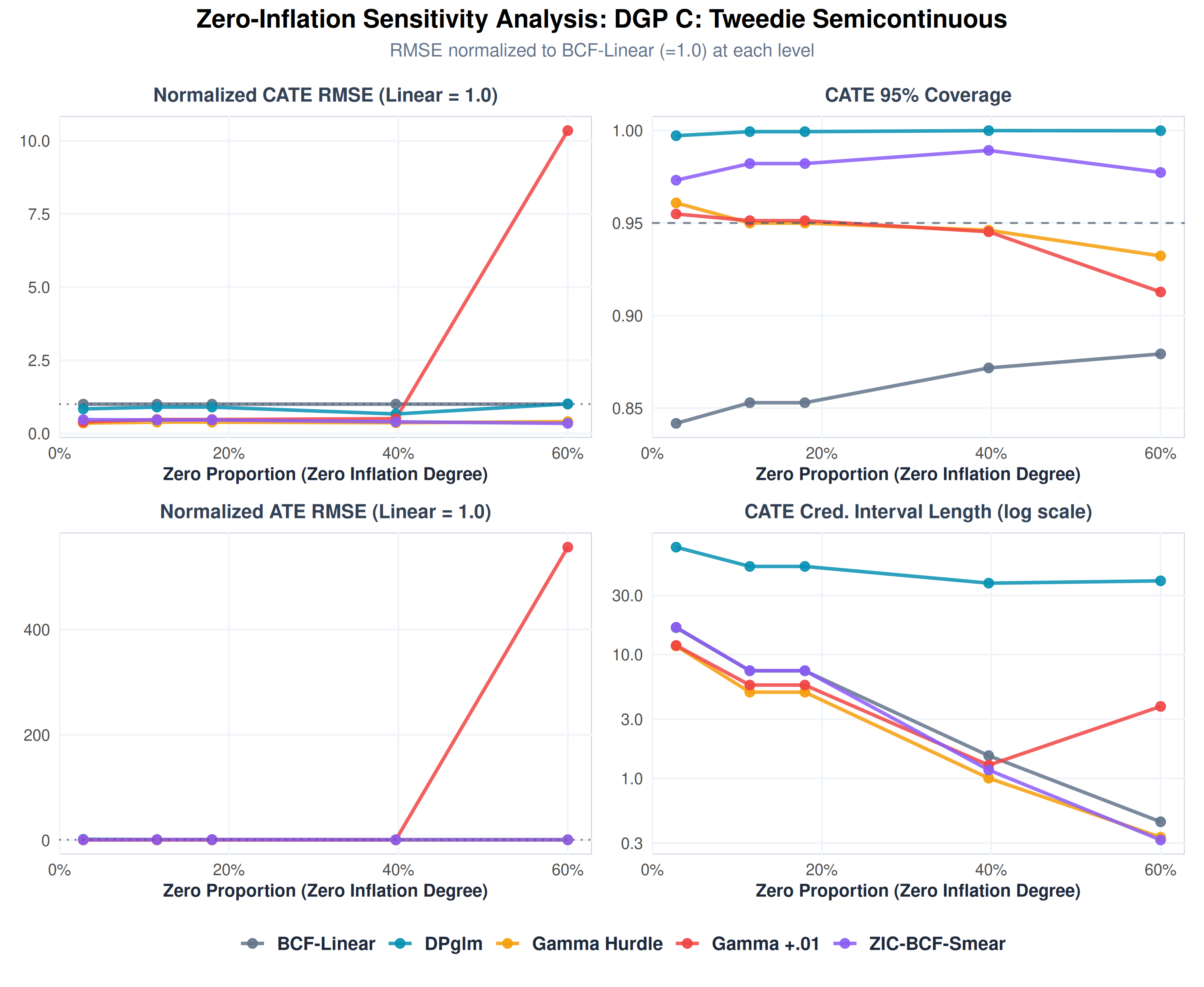}} \\
\multicolumn{2}{c}{(c) DGP C: Tweedie Semicontinuous}
\end{tabular}
\caption{Full link-linear zero-inflation sensitivity analysis (including the diverged Gamma + .01 benchmark) across DGP A, B, and C.}
\label{fig:linear_zi_sens_full}
\end{figure}

\begin{figure}[htbp]
\centering
\begin{tabular}{@{}cc@{}}
\includegraphics[width=0.48\textwidth]{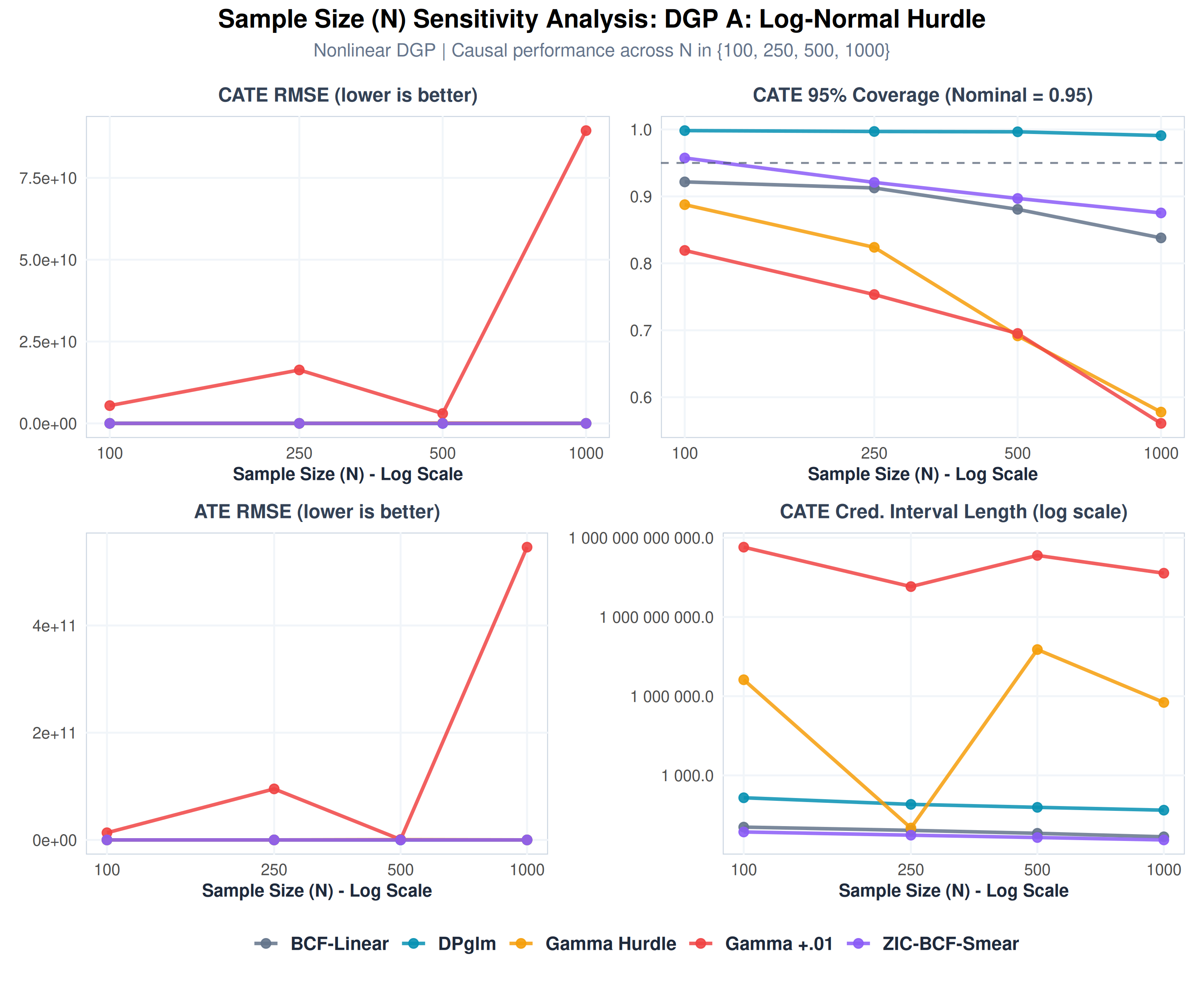} &
\includegraphics[width=0.48\textwidth]{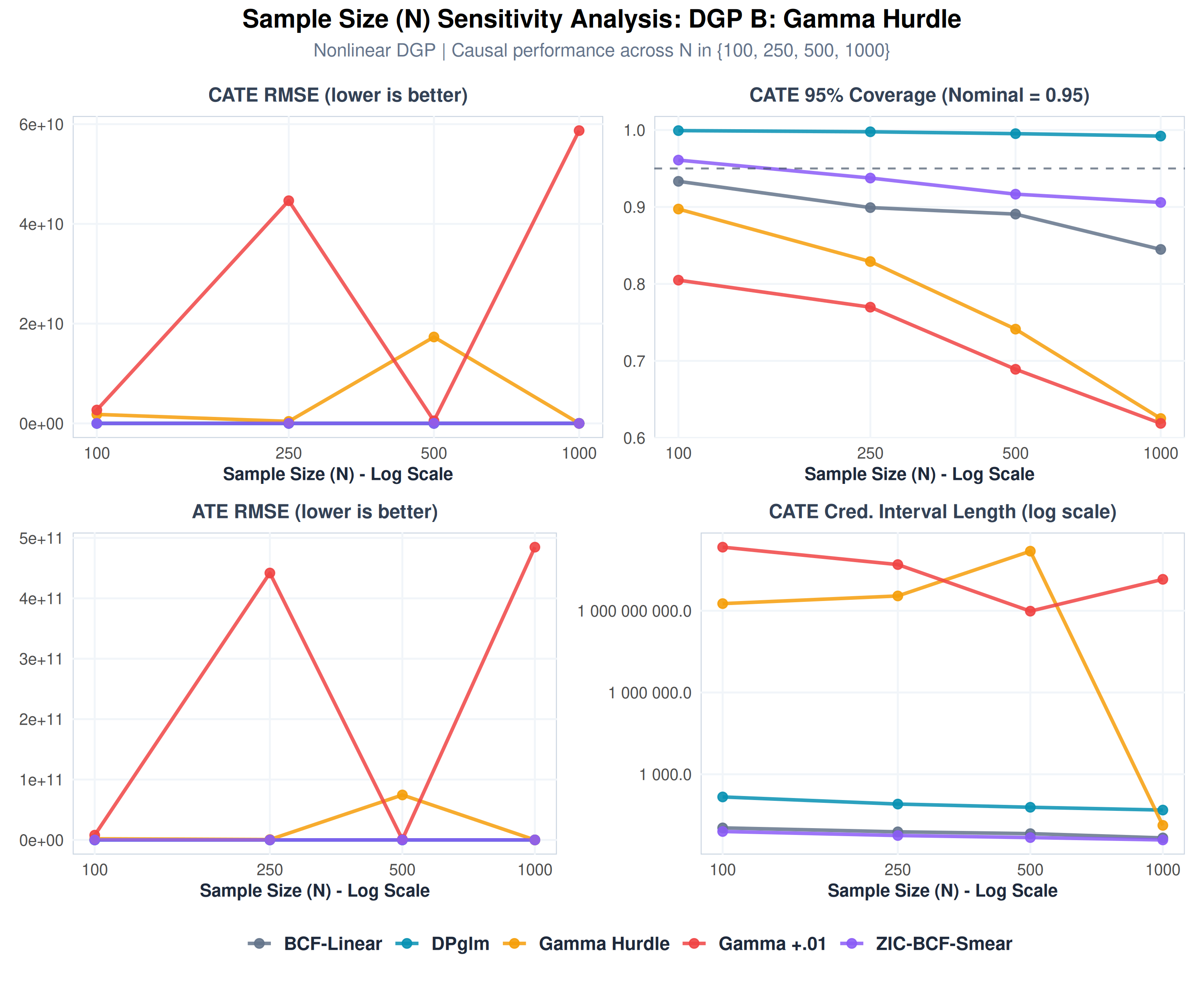} \\
(a) DGP A: Log-Normal Hurdle & (b) DGP B: Gamma Hurdle \\
\multicolumn{2}{c}{\includegraphics[width=0.48\textwidth]{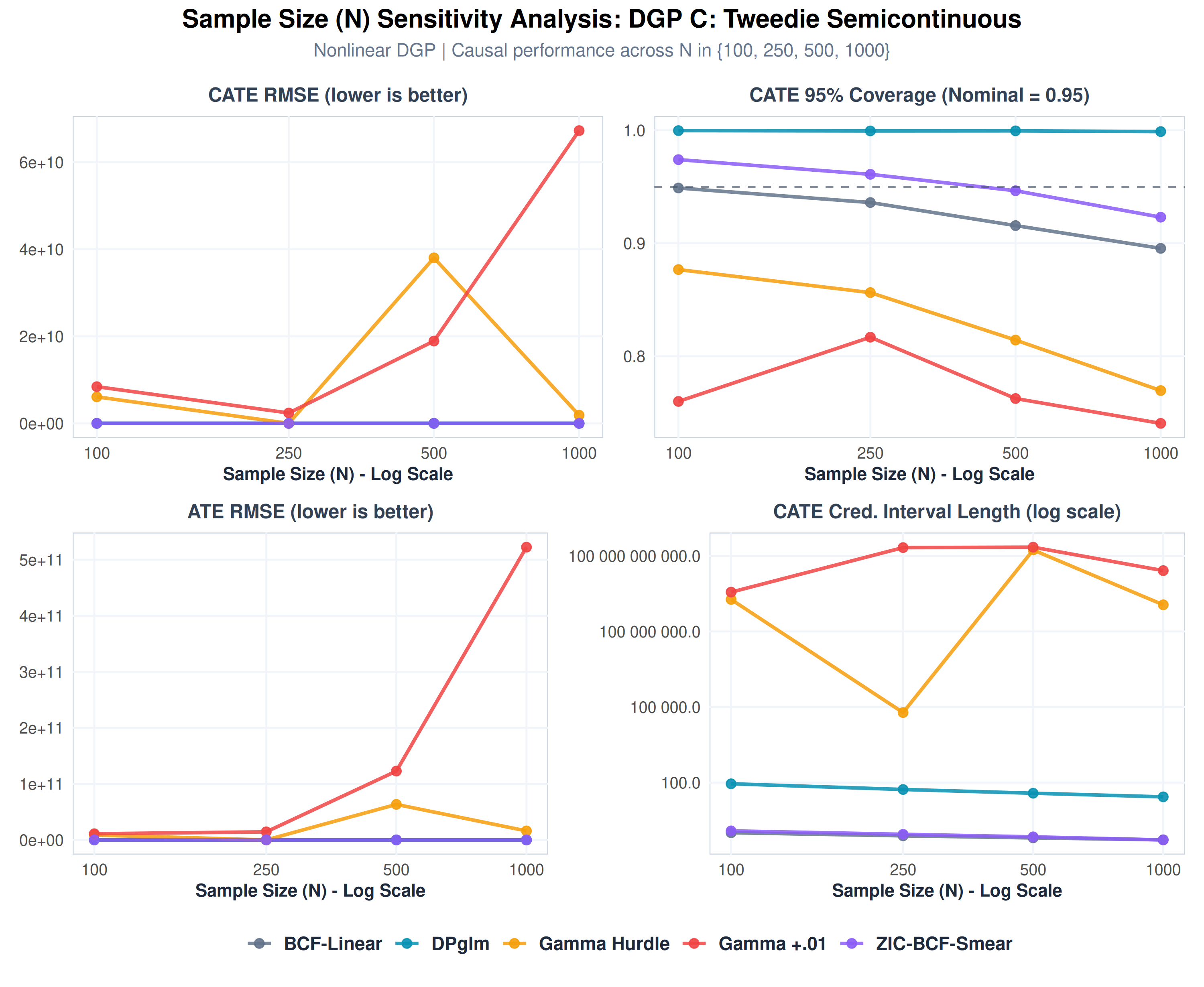}} \\
\multicolumn{2}{c}{(c) DGP C: Tweedie Semicontinuous}
\end{tabular}
\caption{Full link-nonlinear $N$-sensitivity analysis under misspecification, displaying the extreme y-axis distortion caused by the diverged parametric benchmarks.}
\label{fig:nonlinear_n_sens_full}
\end{figure}

\begin{figure}[htbp]
\centering
\begin{tabular}{@{}cc@{}}
\includegraphics[width=0.48\textwidth]{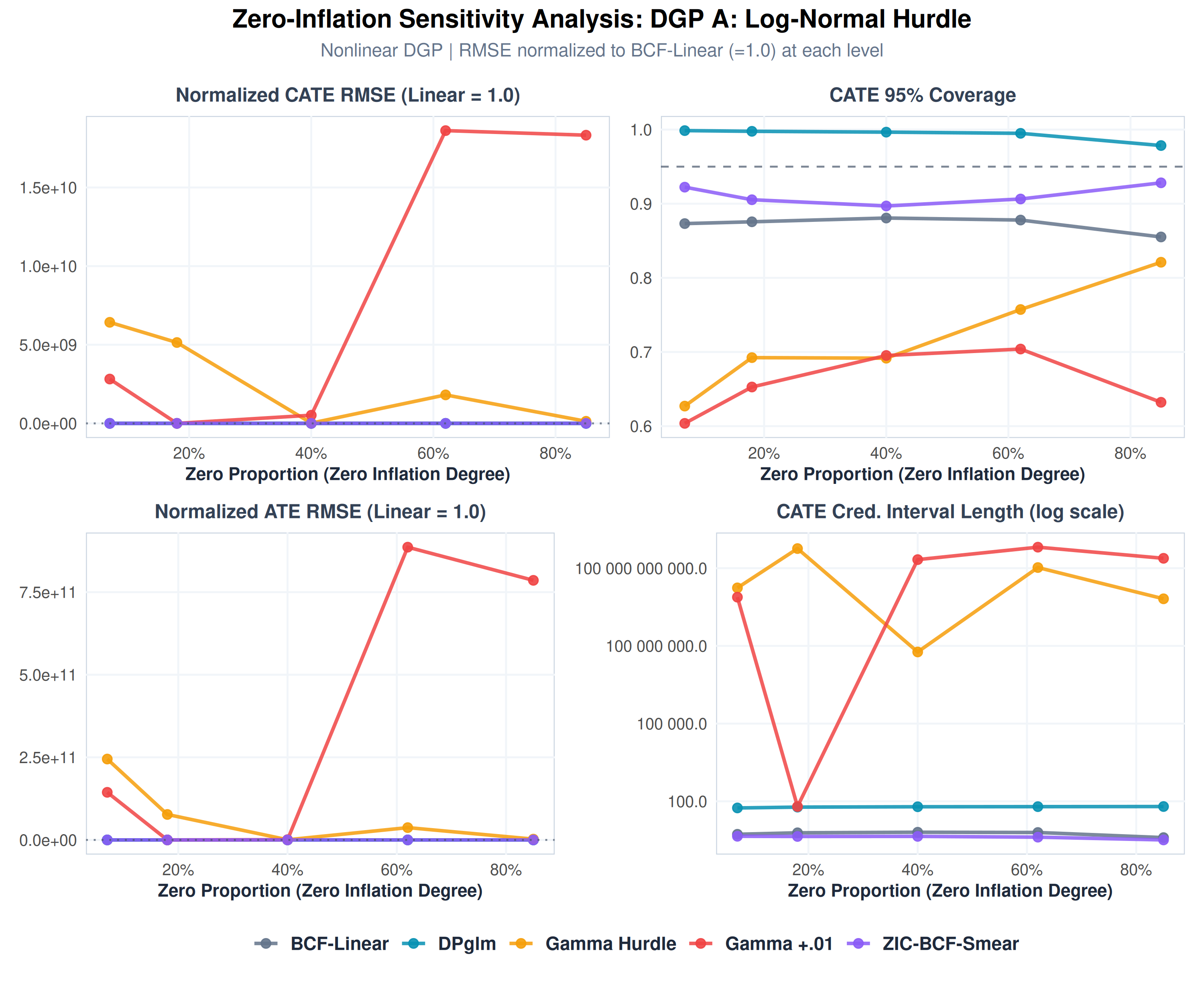} &
\includegraphics[width=0.48\textwidth]{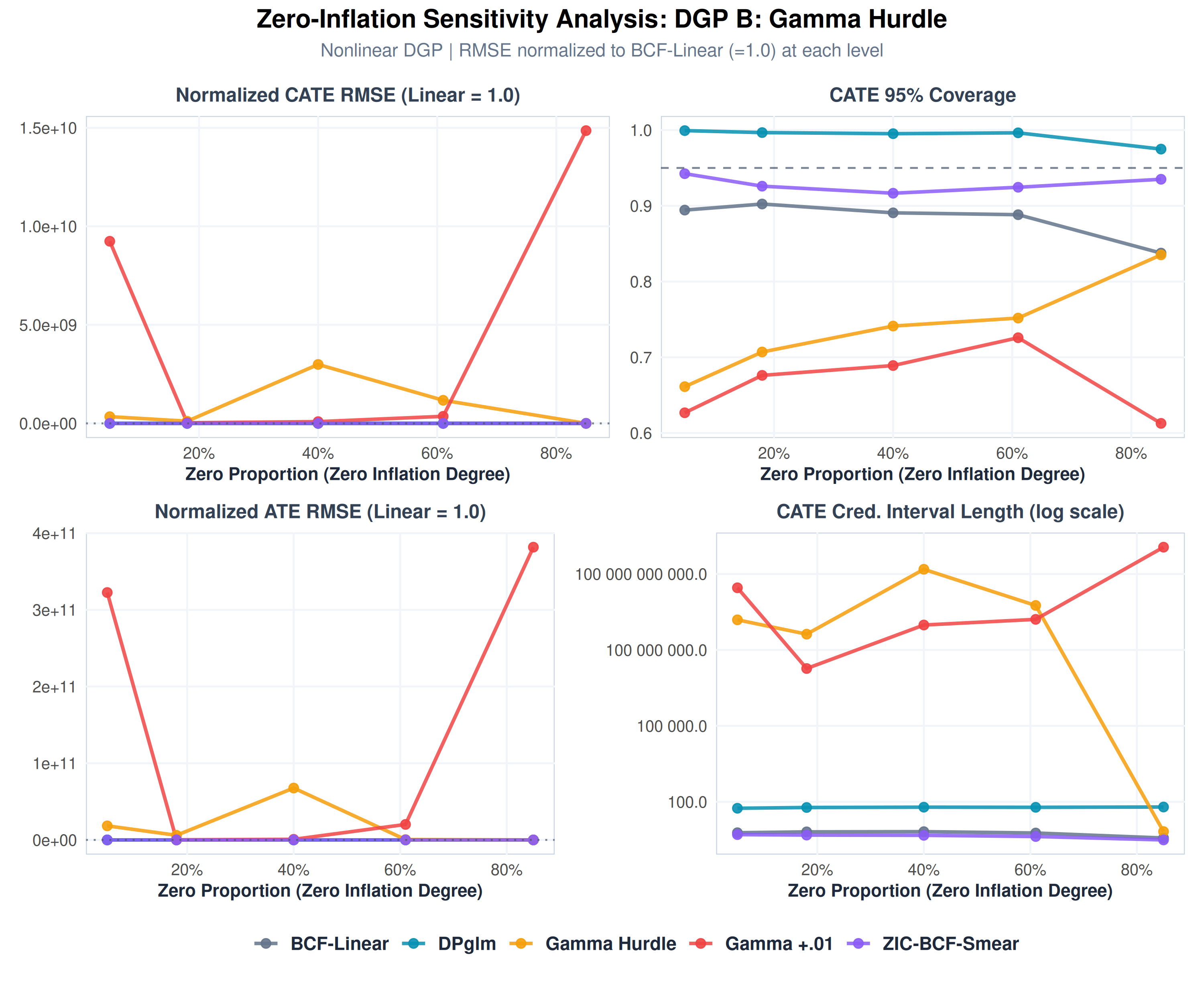} \\
(a) DGP A: Log-Normal Hurdle & (b) DGP B: Gamma Hurdle \\
\multicolumn{2}{c}{\includegraphics[width=0.48\textwidth]{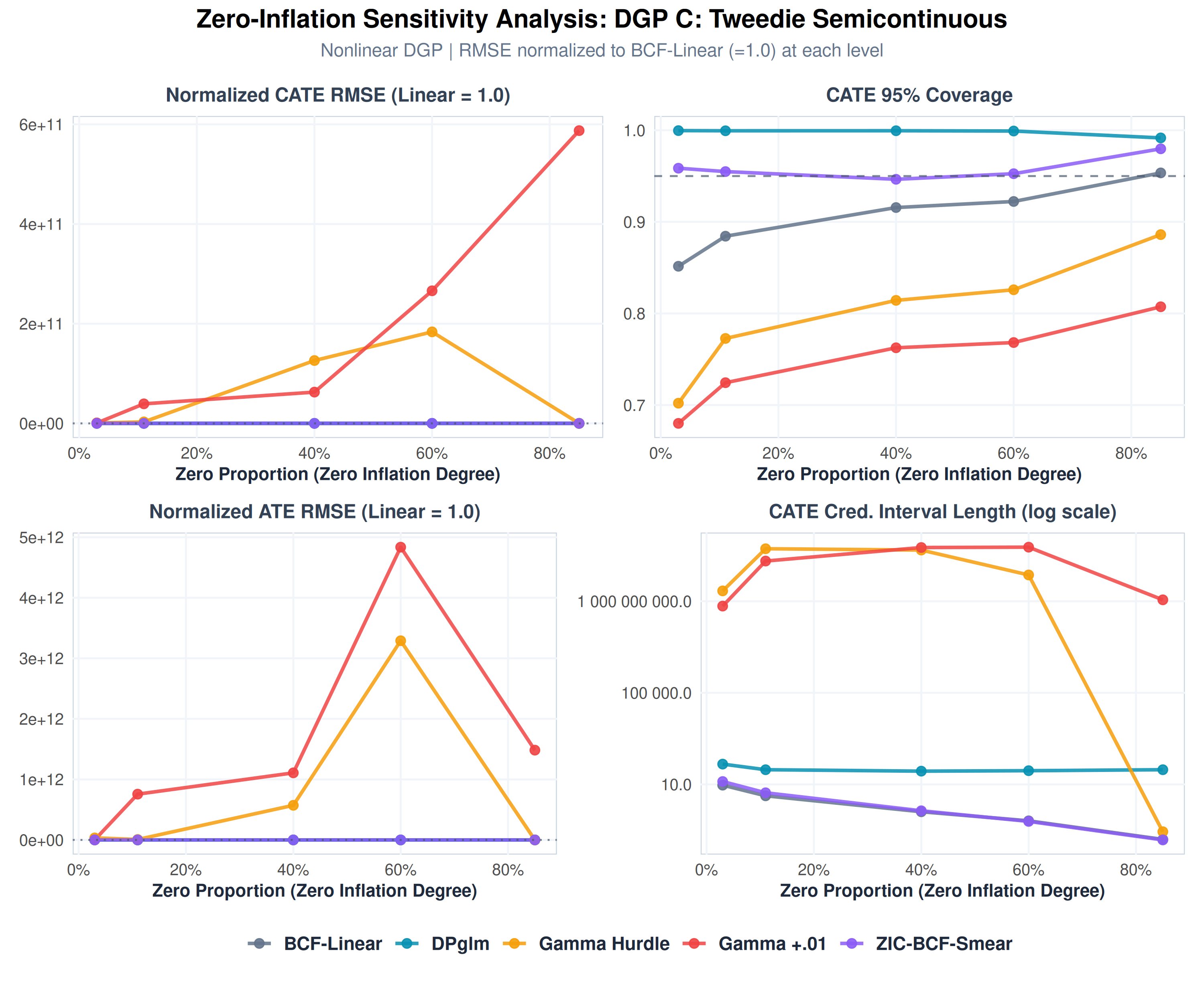}} \\
\multicolumn{2}{c}{(c) DGP C: Tweedie Semicontinuous}
\end{tabular}
\caption{Full link-nonlinear zero-inflation sensitivity analysis under misspecification, displaying the extreme y-axis distortion caused by the diverged parametric benchmarks.}
\label{fig:nonlinear_zi_sens_full}
\end{figure}

\clearpage

\subsection{Legible Sensitivity Plots (Stable Estimators Only)}
\label{subsec:supp_nonlinear_stable}
The main manuscript reports both sensitivity sweeps numerically rather than graphically, because the diverged parametric benchmarks destroy the y-axis scale of any plot that includes them. This subsection supplies the graphical counterpart with the diverged benchmarks removed, so that the numerically stable estimators can be compared on a readable scale. Figure~\ref{fig:supp_linear_n_sens_a} and Figure~\ref{fig:supp_linear_zi_sens_a} give the link-linear $N$-sensitivity and zero-inflation-sensitivity analyses for DGP A with the diverged Gamma + .01 benchmark removed; Figure~\ref{fig:supp_nonlinear_n_sens_stable} and Figure~\ref{fig:supp_nonlinear_zi_sens_stable} give the corresponding link-nonlinear analyses for all three DGPs with both diverged parametric benchmarks (Gamma Hurdle and Gamma + .01) removed.

\begin{figure}[htbp]
\centering
\includegraphics[width=0.86\textwidth]{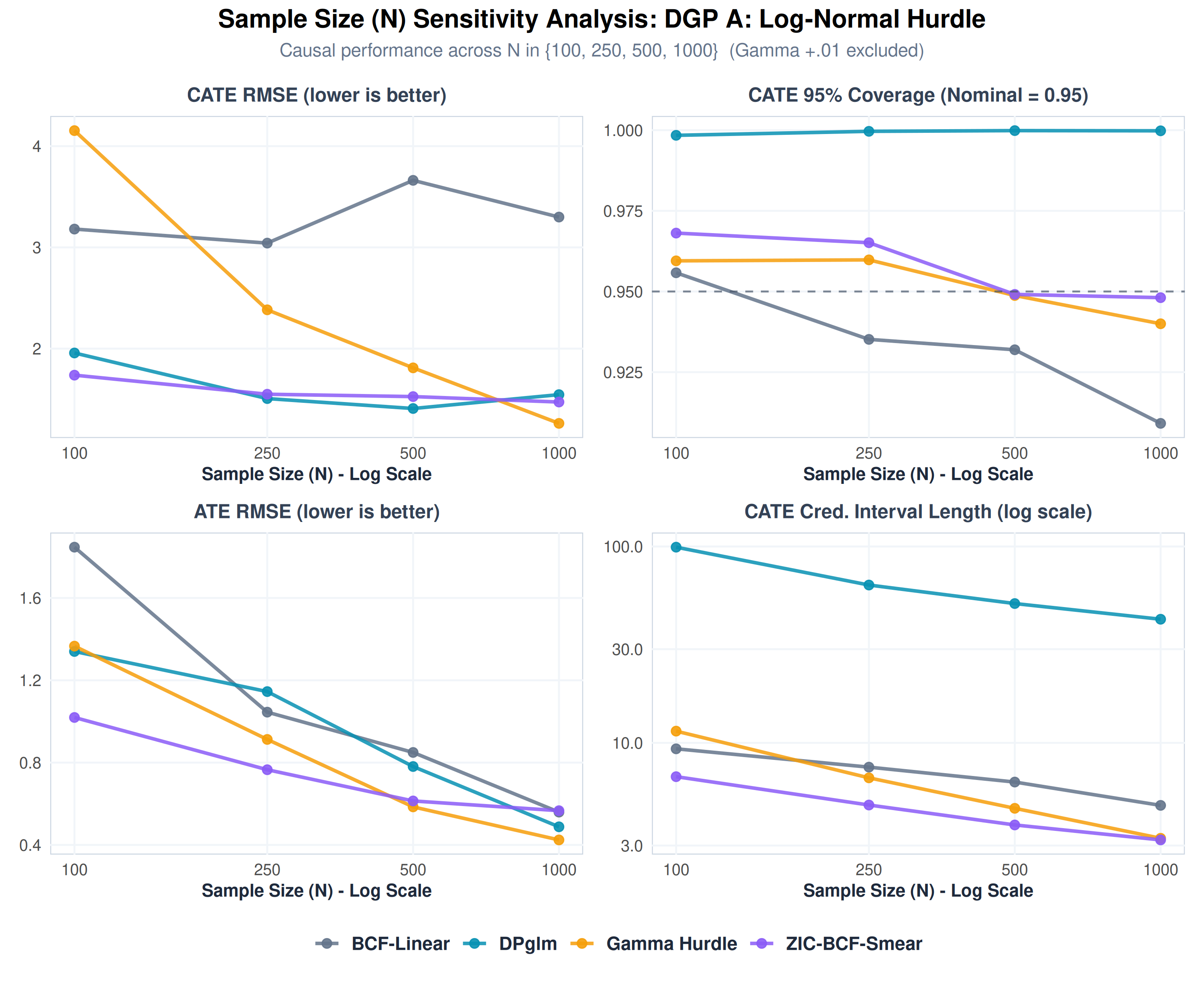}
\caption{Link-linear $N$-sensitivity analysis for DGP A (Log-Normal Hurdle), excluding the diverged Gamma +.01 benchmark for legibility. Reading the four panels together, ZIC-BCF-Smear and the correctly specified Gamma Hurdle both converge as $N$ grows, but ZIC-BCF-Smear reaches the lower CATE RMSE with the shorter credible intervals (log scale, bottom right) while holding coverage at the nominal $95\%$; DPglm's intervals remain an order of magnitude wider at every sample size, which is why its coverage is pinned at $100\%$.}
\label{fig:supp_linear_n_sens_a}
\end{figure}

\begin{figure}[htbp]
\centering
\includegraphics[width=0.86\textwidth]{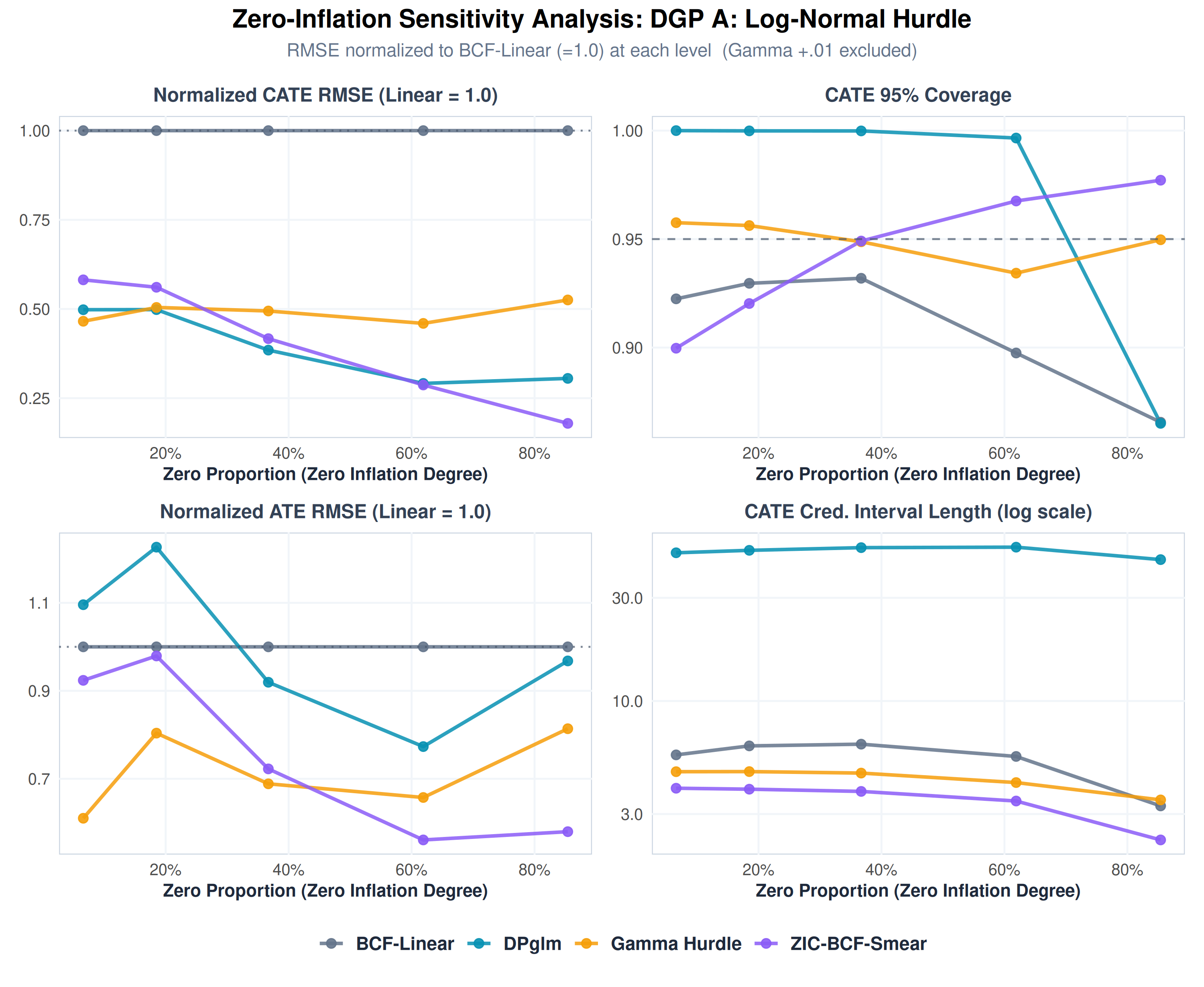}
\caption{Link-linear zero-inflation sensitivity analysis for DGP A (Log-Normal Hurdle), excluding the diverged Gamma +.01 benchmark for legibility. RMSE is normalized by the true CATE scale so that cells with different zero proportions are comparable. The ZIC-BCF-Smear advantage over the single-part BCF-Linear baseline widens monotonically as the zero proportion rises toward $85\%$, while the correctly specified Gamma Hurdle deteriorates at the right-hand end of the range as the active sample shrinks below $\approx 15\%$ of the data.}
\label{fig:supp_linear_zi_sens_a}
\end{figure}

\begin{figure}[htbp]
\centering
\begin{tabular}{@{}cc@{}}
\includegraphics[width=0.48\textwidth]{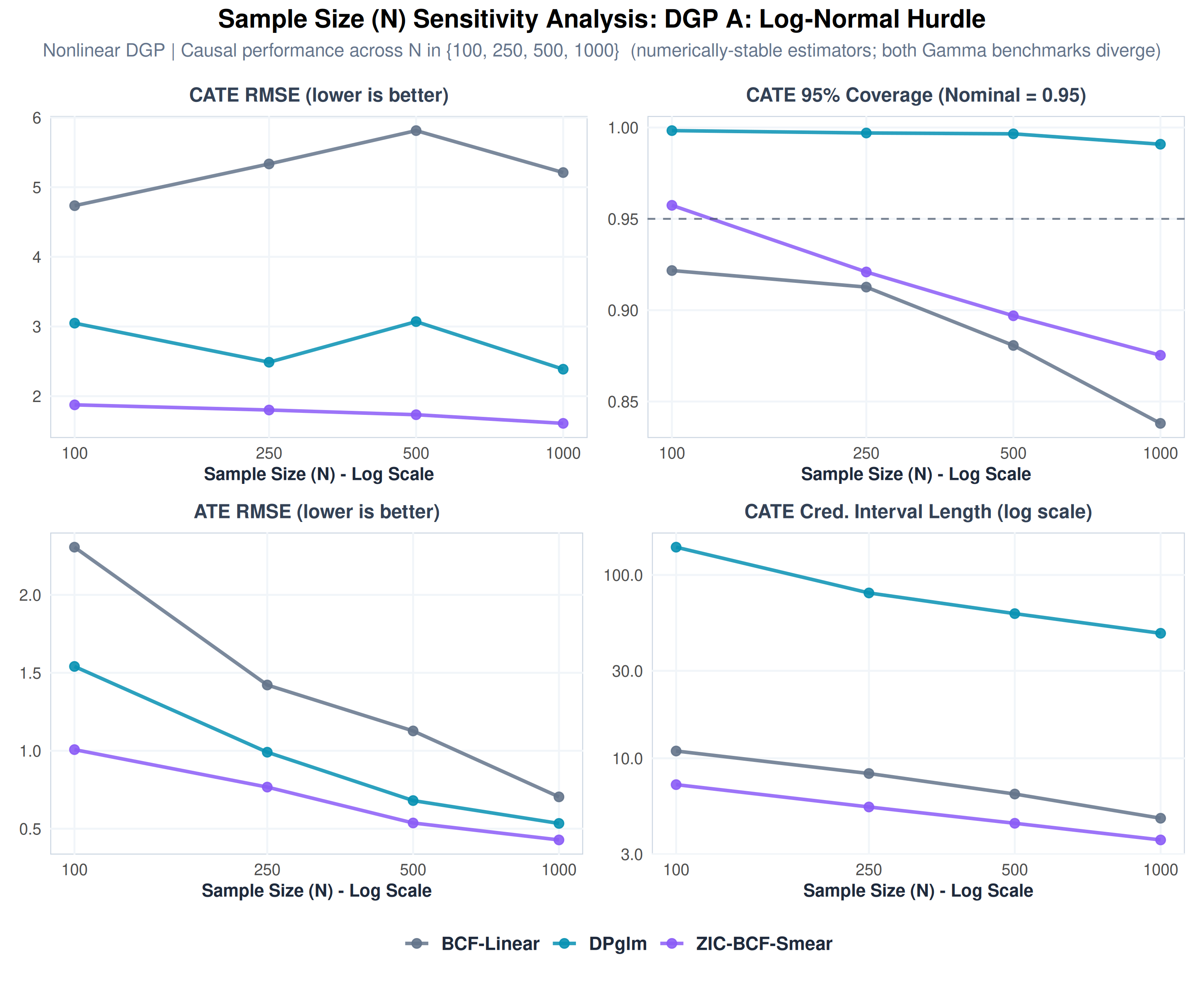} &
\includegraphics[width=0.48\textwidth]{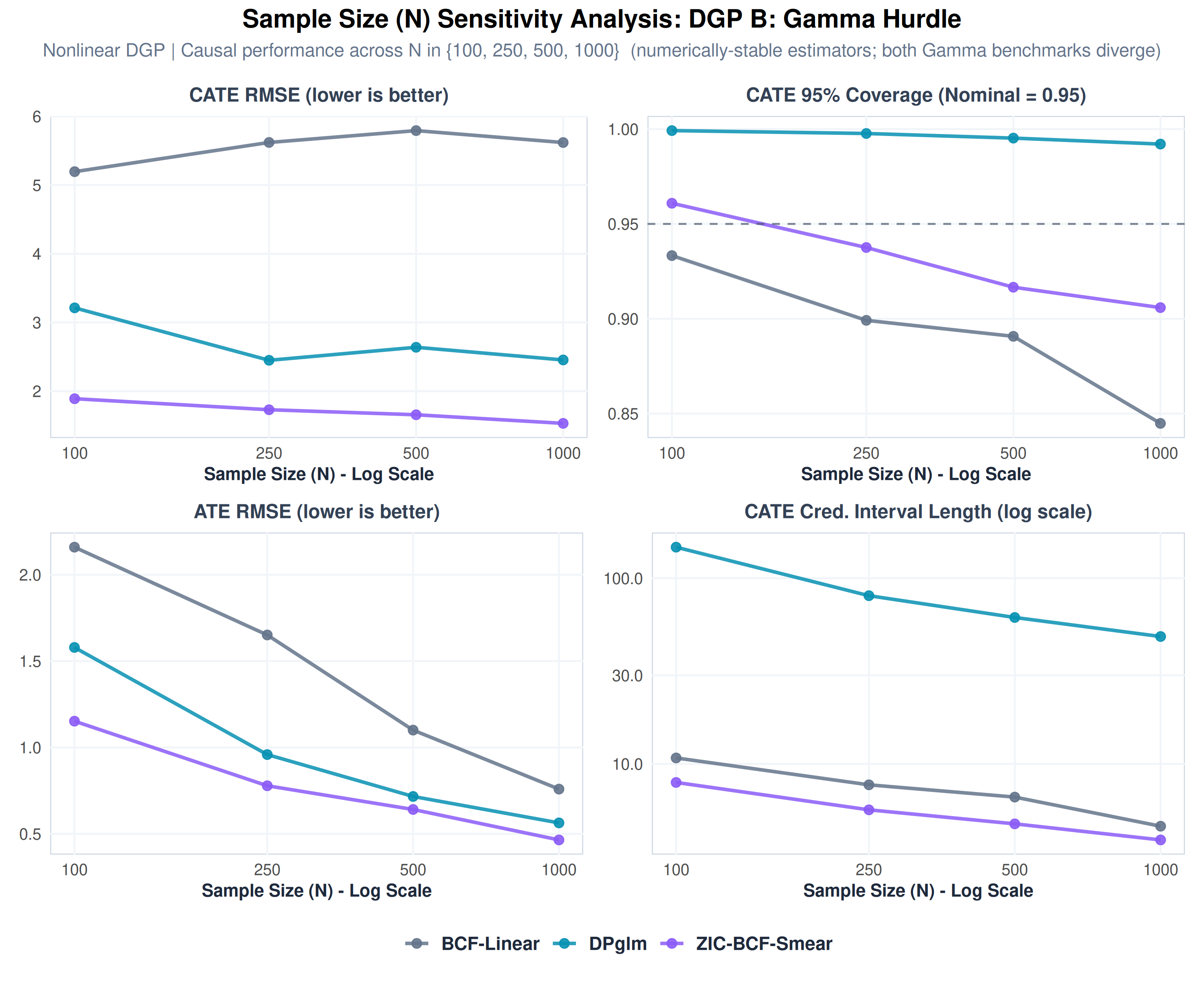} \\
(a) DGP A: Log-Normal Hurdle & (b) DGP B: Gamma Hurdle \\
\multicolumn{2}{c}{\includegraphics[width=0.48\textwidth]{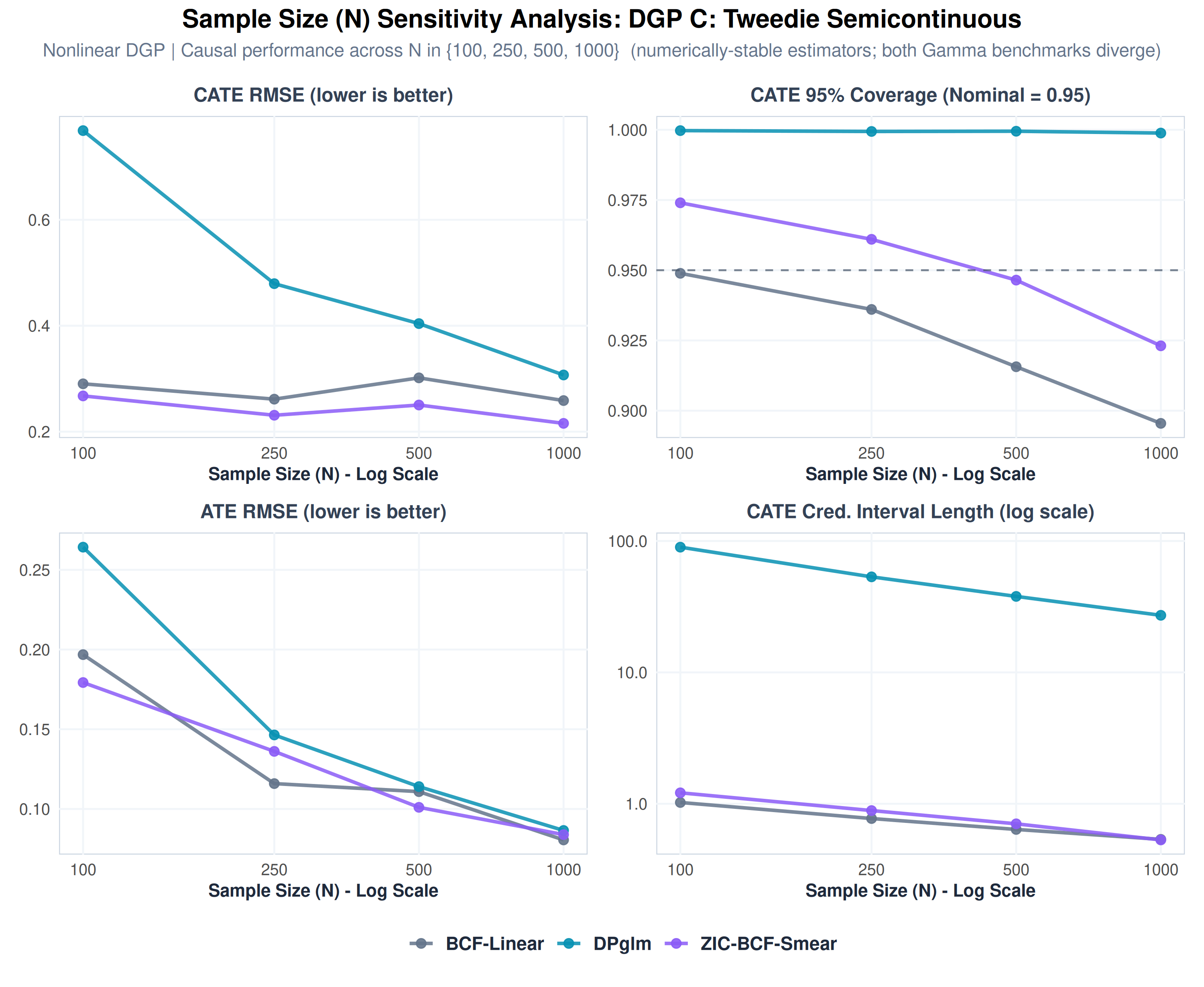}} \\
\multicolumn{2}{c}{(c) DGP C: Tweedie Semicontinuous}
\end{tabular}
\caption{Link-nonlinear $N$-sensitivity analysis restricted to the numerically stable estimators (BCF-Linear, DPglm, and ZIC-BCF-Smear). Reading across the panels, ZIC-BCF-Smear holds the lowest CATE RMSE and the coverage closest to the nominal $95\%$ at every sample size in all three DGPs, while DPglm's credible-interval length stays an order of magnitude wider throughout.}
\label{fig:supp_nonlinear_n_sens_stable}
\end{figure}

\begin{figure}[htbp]
\centering
\begin{tabular}{@{}cc@{}}
\includegraphics[width=0.48\textwidth]{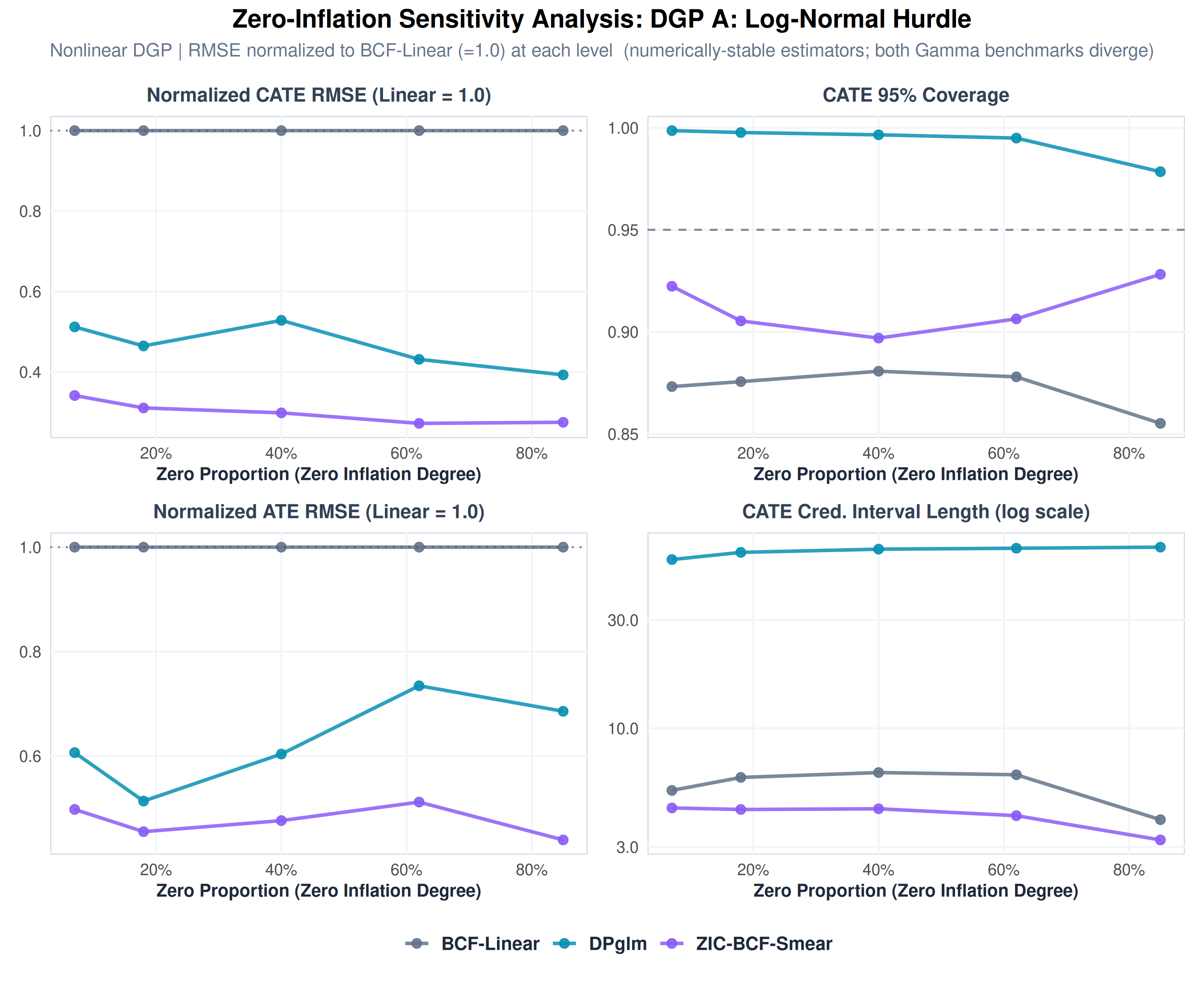} &
\includegraphics[width=0.48\textwidth]{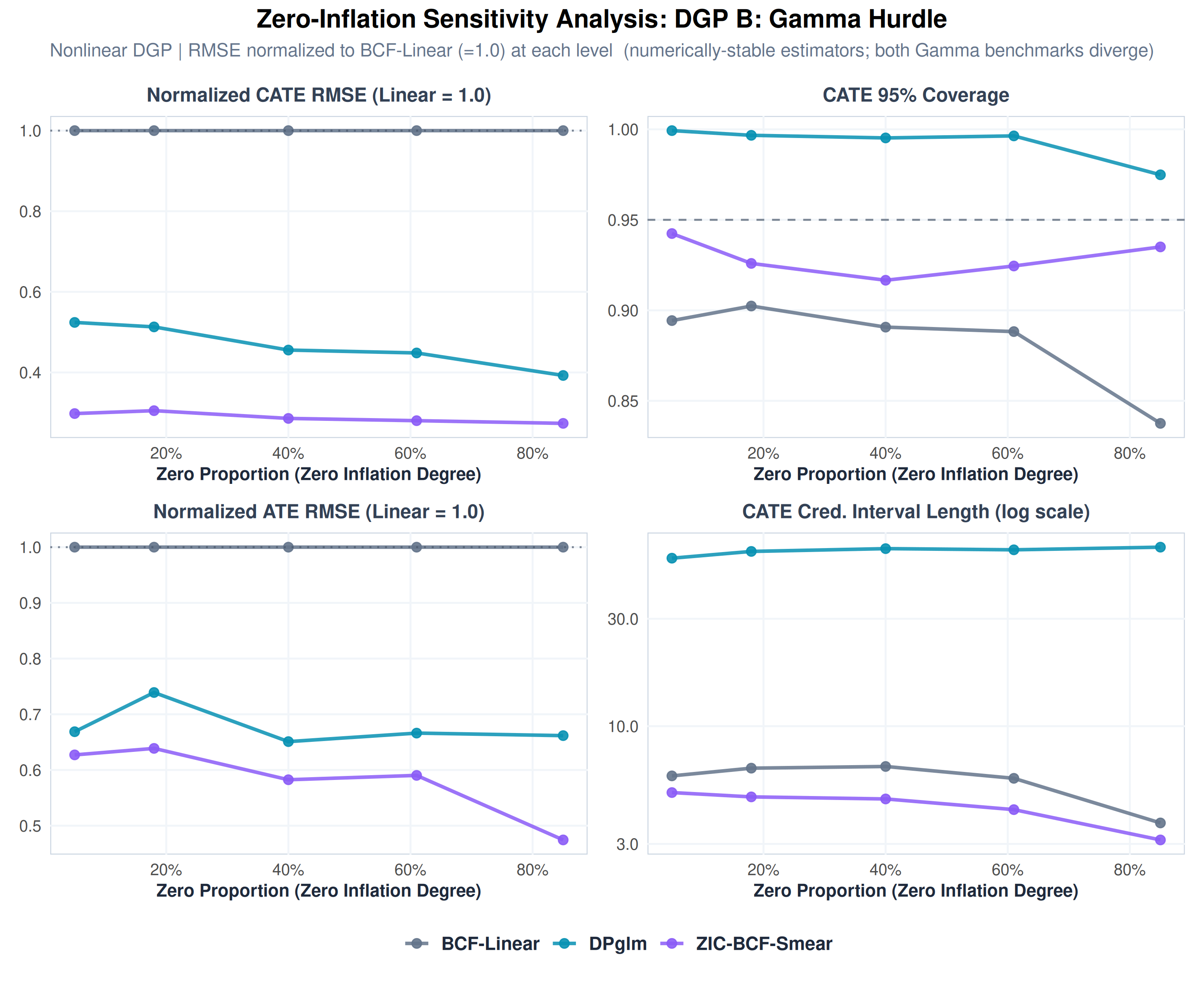} \\
(a) DGP A: Log-Normal Hurdle & (b) DGP B: Gamma Hurdle \\
\multicolumn{2}{c}{\includegraphics[width=0.48\textwidth]{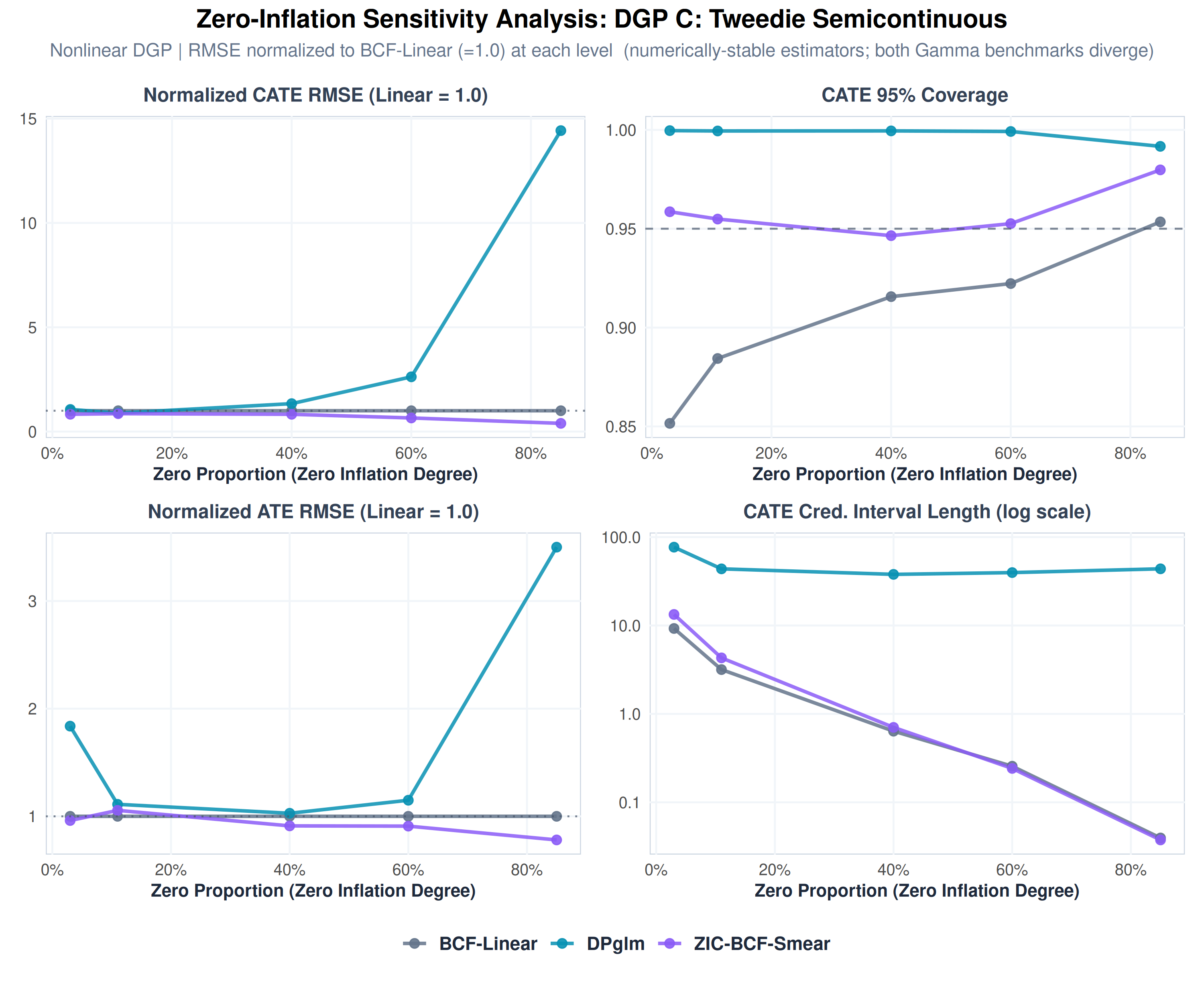}} \\
\multicolumn{2}{c}{(c) DGP C: Tweedie Semicontinuous}
\end{tabular}
\caption{Link-nonlinear zero-inflation sensitivity analysis restricted to the numerically stable estimators. The gap between ZIC-BCF-Smear and the single-part BCF-Linear baseline widens monotonically as the zero proportion rises toward $85\%$, which is the graphical signature of fitting the intensity forest only on units that clear the hurdle.}
\label{fig:supp_nonlinear_zi_sens_stable}
\end{figure}

\clearpage

\subsection{Numerical Stability of the Parametric Benchmarks}
\label{subsec:supp_stability}
Table~\ref{tab:supp_nonlinear_stability} reports the divergence rates of the parametric benchmarks under the link-nonlinear DGPs, together with medians computed across the surviving seeds. The medians are given so that the comparison does not rest solely on means contaminated by overflow. They show that excluding the diverged seeds does not rescue the benchmarks: the CATE correlations remain negative, and the robust median CATE RMSEs remain a factor of $2.0$ to $2.7$ above ZIC-BCF-Smear.

\begin{table}[htbp]
\centering
\caption{Numerical stability of the parametric benchmarks under nonlinear DGPs ($N = 500$, ZI Level 3). Reports divergence rates and robust medians across 100 seeds.}
\label{tab:supp_nonlinear_stability}
\resizebox{\textwidth}{!}{%
\begin{tabular}{llcccc}
\toprule
DGP & Model & Diverged Seeds & CATE RMSE (median) & ATE Abs Err (median) & CATE Corr \\
\midrule
\multirow{3}{*}{DGP A: Log-Normal}
& Gamma Hurdle  & 4 / 100 & 3.960 & 0.759 & -0.039 \\
& Gamma +.01    & 5 / 100 & 4.336 & 0.747 & -0.081 \\
& \textbf{ZIC-BCF-Smear} & \textbf{0 / 100} & \textbf{1.611} & \textbf{0.352} & \textbf{0.686} \\
\midrule
\multirow{3}{*}{DGP B: Gamma}
& Gamma Hurdle  & 5 / 100 & 3.704 & 0.784 & -0.040 \\
& Gamma +.01    & 9 / 100 & 4.145 & 0.692 & -0.068 \\
& \textbf{ZIC-BCF-Smear} & \textbf{0 / 100} & \textbf{1.537} & \textbf{0.494} & \textbf{0.645} \\
\midrule
\multirow{3}{*}{DGP C: Tweedie}
& Gamma Hurdle  & 5 / 100 & 0.369 & 0.122 & -0.060 \\
& Gamma +.01    & 7 / 100 & 0.460 & 0.147 & -0.045 \\
& \textbf{ZIC-BCF-Smear} & \textbf{0 / 100} & \textbf{0.188} & \textbf{0.065} & \textbf{0.175} \\
\bottomrule
\end{tabular}%
}
\end{table}

\subsection{Hurdle Stage Sensitivity Results}
For completeness, the full numerical values of the Hurdle stage sensitivity analyses are provided. The standard configuration ($N = 500$) appears as the middle column of each table, and reproduces the hurdle-stage figures quoted in the main manuscript. Table~\ref{tab:linear_hurdle_sens} displays the Hurdle stage metrics across sample size $N$ and zero-inflation levels for the link-linear study. Table~\ref{tab:nonlinear_hurdle_sens} displays the Hurdle stage metrics for the link-nonlinear study.

\begin{table}[htbp]
\centering
\caption{Linear study Hurdle stage sensitivity results across sample sizes $N$ (under nominal zero-inflation) and zero-inflation levels (under $N=500$).}
\label{tab:linear_hurdle_sens}
\resizebox{\textwidth}{!}{%
\begin{tabular}{lllccccc}
\toprule
DGP & Model & Parameter & $N=100$ & $N=250$ & $N=500$ & $N=1000$ & ZI 85\% (Level 1) \\
\midrule
\multirow{6}{*}{DGP A: Log-Normal} 
& \multirow{2}{*}{DPglm} 
& RMSE & 0.086 & 0.067 & 0.053 & 0.057 & 0.050 \\
&& Coverage & 98.0\% & 99.5\% & 99.7\% & 99.7\% & 86.6\% \\
& \multirow{2}{*}{Gamma Hurdle} 
& RMSE & 0.208 & 0.129 & 0.098 & 0.071 & 0.080 \\
&& Coverage & 95.1\% & 95.8\% & 94.0\% & 94.8\% & 95.1\% \\
& \multirow{2}{*}{\textbf{ZIC-BCF-Smear}} 
& RMSE & \textbf{0.068} & \textbf{0.054} & \textbf{0.046} & \textbf{0.040} & \textbf{0.036} \\
&& Coverage & \textbf{98.9\%} & \textbf{99.0\%} & \textbf{99.3\%} & \textbf{98.6\%} & \textbf{98.3\%} \\
\midrule
\multirow{6}{*}{DGP B: Gamma} 
& \multirow{2}{*}{DPglm} 
& RMSE & 0.096 & 0.065 & 0.055 & 0.056 & 0.047 \\
&& Coverage & 97.9\% & 99.5\% & 99.7\% & 99.6\% & 86.9\% \\
& \multirow{2}{*}{Gamma Hurdle} 
& RMSE & 0.208 & 0.129 & 0.098 & 0.071 & 0.081 \\
&& Coverage & 95.0\% & 95.8\% & 94.0\% & 94.8\% & 95.0\% \\
& \multirow{2}{*}{\textbf{ZIC-BCF-Smear}} 
& RMSE & \textbf{0.067} & \textbf{0.054} & \textbf{0.046} & \textbf{0.040} & \textbf{0.036} \\
&& Coverage & \textbf{98.5\%} & \textbf{98.8\%} & \textbf{99.3\%} & \textbf{98.3\%} & \textbf{98.6\%} \\
\midrule
\multirow{6}{*}{DGP C: Tweedie} 
& \multirow{2}{*}{DPglm} 
& RMSE & 0.070 & 0.052 & 0.044 & 0.036 & 0.059 \\
&& Coverage & 81.3\% & 84.3\% & 87.0\% & 88.3\% & 100.0\% \\
& \multirow{2}{*}{Gamma Hurdle} 
& RMSE & 0.173 & 0.114 & 0.081 & 0.061 & 0.097 \\
&& Coverage & 97.5\% & 96.8\% & 95.0\% & 92.7\% & 94.4\% \\
& \multirow{2}{*}{\textbf{ZIC-BCF-Smear}} 
& RMSE & \textbf{0.044} & \textbf{0.035} & \textbf{0.028} & \textbf{0.026} & \textbf{0.045} \\
&& Coverage & \textbf{95.3\%} & \textbf{98.2\%} & \textbf{99.3\%} & \textbf{97.9\%} & \textbf{97.9\%} \\
\bottomrule
\end{tabular}%
}
\end{table}

\begin{table}[htbp]
\centering
\caption{Nonlinear study Hurdle stage sensitivity results across sample sizes $N$ (under nominal zero-inflation) and zero-inflation levels (under $N=500$).}
\label{tab:nonlinear_hurdle_sens}
\resizebox{\textwidth}{!}{%
\begin{tabular}{lllccccc}
\toprule
DGP & Model & Parameter & $N=100$ & $N=250$ & $N=500$ & $N=1000$ & ZI 85\% (Level 1) \\
\midrule
\multirow{6}{*}{DGP A: Log-Normal} 
& \multirow{2}{*}{DPglm} 
& RMSE & 0.105 & 0.085 & 0.077 & 0.078 & 0.078 \\
&& Coverage & 98.2\% & 98.8\% & 98.6\% & 96.1\% & 96.9\% \\
& \multirow{2}{*}{Gamma Hurdle} 
& RMSE & 0.229 & 0.167 & 0.129 & 0.111 & 0.104 \\
&& Coverage & 92.6\% & 89.9\% & 88.6\% & 82.2\% & 85.9\% \\
& \multirow{2}{*}{\textbf{ZIC-BCF-Smear}} 
& RMSE & \textbf{0.083} & \textbf{0.074} & \textbf{0.065} & \textbf{0.062} & \textbf{0.065} \\
&& Coverage & \textbf{98.0\%} & \textbf{96.4\%} & \textbf{94.7\%} & \textbf{90.8\%} & \textbf{92.8\%} \\
\midrule
\multirow{6}{*}{DGP B: Gamma} 
& \multirow{2}{*}{DPglm} 
& RMSE & 0.101 & 0.088 & 0.077 & 0.077 & 0.080 \\
&& Coverage & 98.4\% & 98.8\% & 98.2\% & 96.6\% & 96.7\% \\
& \multirow{2}{*}{Gamma Hurdle} 
& RMSE & 0.229 & 0.166 & 0.129 & 0.111 & 0.104 \\
&& Coverage & 92.6\% & 90.0\% & 88.5\% & 82.3\% & 86.0\% \\
& \multirow{2}{*}{\textbf{ZIC-BCF-Smear}} 
& RMSE & \textbf{0.083} & \textbf{0.075} & \textbf{0.066} & \textbf{0.062} & \textbf{0.065} \\
&& Coverage & \textbf{98.0\%} & \textbf{96.4\%} & \textbf{94.5\%} & \textbf{90.9\%} & \textbf{92.9\%} \\
\midrule
\multirow{6}{*}{DGP C: Tweedie} 
& \multirow{2}{*}{DPglm} 
& RMSE & 0.099 & 0.067 & 0.057 & 0.049 & 0.051 \\
&& Coverage & 99.6\% & 99.9\% & 99.8\% & 99.6\% & 98.3\% \\
& \multirow{2}{*}{Gamma Hurdle} 
& RMSE & 0.220 & 0.150 & 0.111 & 0.090 & 0.083 \\
&& Coverage & 93.7\% & 93.5\% & 92.4\% & 88.1\% & 93.6\% \\
& \multirow{2}{*}{\textbf{ZIC-BCF-Smear}} 
& RMSE & \textbf{0.069} & \textbf{0.051} & \textbf{0.044} & \textbf{0.040} & \textbf{0.028} \\
&& Coverage & \textbf{99.2\%} & \textbf{99.7\%} & \textbf{98.4\%} & \textbf{96.6\%} & \textbf{98.5\%} \\
\bottomrule
\end{tabular}%
}
\end{table}

\clearpage

\section{Applied Study: Severe Adverse-Event Burden in the SPECTRUM Head-and-Neck Cancer Trial}
\label{sec:supp_onc}

This section is the extended, self-contained version of the oncology applied study summarized in the main manuscript. It develops in full the secondary causal analysis of the SPECTRUM phase III trial \citep{vermorken2013spectrum, nct00460265}, whose individual-participant data were obtained through Project Data Sphere \citep{pds_headne_265}: the analytic cohort, the two outcome definitions and the estimand they identify, the posterior models and their computation, the complete subgroup and contrast results under both endpoints, the reconciliation with the clinical and statistical literature, the limitations, and computational reproducibility.

\subsection{Trial Context and Analytic Cohort}
\label{sec:supp_context}

SPECTRUM was an open-label phase III randomized trial of cisplatin and fluorouracil with or without panitumumab for recurrent or metastatic squamous-cell carcinoma of the head and neck \citep{vermorken2013spectrum,nct00460265}. The parent trial reported 657 randomized patients. The current analytic extract contains 520 participants, balanced at 260 per arm.

\subsubsection{What the Cohort Derivation Can and Cannot Establish}
\label{subsec:supp_cohort_derivation}

The reduction from 657 to 520 patients warrants explicit examination, because randomization grounds causal identification and applies to the full randomized cohort rather than automatically to an analytic subset.

What is checkable from the supplied files is the following. The analytic cohort is the set of subjects in the supplied core-variable file with a codeable treatment arm, one row per subject, and it contains exactly 260 patients per arm. The adverse-event file contains 9,934 records on 514 distinct subjects, 256 in the chemotherapy arm and 258 in the panitumumab arm; the 6 analytic-cohort patients with no adverse-event record contribute an outcome of zero at every horizon. Recomputing the summed-duration endpoint from the adverse-event file reproduces the patient-level outcome for all 520 patients exactly, so the linkage between the two files is sound. That reconciliation, from 520 to 514 and back, is complete.

What is not checkable is the step from 657 to 520. The reduction occurs in the supplied extract itself rather than in any analysis code: the core-variable file distributed with the \nolinkurl{HeadNe_Amgen_2007_265} package contains those 520 subjects, and no file in the package records which parent-trial patients were omitted or why. A reader therefore cannot tell from these materials whether the reduction is a benign completeness filter or a selection that differs by arm, and the exactly balanced 260 per arm is consistent with either a deliberately balanced release or a coincidence. This bears directly on the intention-to-treat characterization: the estimand below is the average treatment effect in the 520-person analytic cohort, and the extent to which that transfers to the parent trial's randomized population is not established here. Resolving it requires the release documentation for the extract, which would have to be requested from Project Data Sphere.

\subsubsection{Treatment Coding and the Adjustment Set}

Treatment is coded as $Z_i=1$ for panitumumab plus chemotherapy and $Z_i=0$ for chemotherapy alone. The baseline covariate vector $X_i$ contains age, ECOG status, sex, diagnosis site, HPV status, disease status, and the patient-level prior-radiotherapy indicator PRRADIO. Because the trial randomized assignment, the design propensity is known to be 0.5 and is supplied directly to ZIC-BCF-Smear rather than estimated from the realized arms. The unadjusted arm difference is therefore already unbiased for the cohort average effect.

Two specific considerations govern the selection and encoding of these baseline covariates.

First, tumor category is omitted because \texttt{TUMCAT} is an exact deterministic recode of diagnosis site (\texttt{DIAGTYPE}) in this extract (larynx and oropharynx map to Y, while hypopharynx and oral cavity map to N, with no exceptions among all 520 patients). Including both would introduce an exact linear dependency into the design matrix. Tumor category is likewise excluded from prespecified subgroup analyses, as it represents a strict coarsening of the diagnosis site variable already analyzed at full resolution.

Second, while overall prior head-and-neck cancer treatment (\texttt{PRHNTRTC}) is retained descriptively, the model adjustment set and prespecified subgroup lists utilize the clinically specific prior-radiotherapy indicator \texttt{PRRADIO}. \texttt{PRHNTRTC} is present for 204 of 260 chemotherapy patients (78.5\%) and 216 of 260 panitumumab patients (83.1\%); \texttt{PRRADIO} is present for 200 chemotherapy patients (76.9\%) and 216 panitumumab patients (83.1\%). Prior radiotherapy is a principal clinical candidate modifier for mucosal and cutaneous toxicity, the outcome being modeled. The two indicators differ for only four patients, all in the chemotherapy arm, so they are not entered together: their near-redundancy destabilized the Gamma benchmark, whereas \texttt{PRRADIO} provides the clinically relevant adjustment and subgroup variable.

Three qualifications attach to these variables. \texttt{PRHNTRTC} is a composite: the data definition table derives it as Y if the subject had prior radiotherapy as per RAHX, or other prior treatment as per TXHX or OTHEPROC \citep{amgen_20050251_ddt}. The baseline demographic records supply \texttt{PRRADIO}, and the trial radiotherapy administration domain (\texttt{radiotx}) supplies 961 detailed radiotherapy records covering all 520 subjects, 857 of them with \texttt{RAANY} $=$ Y at the record level and every patient recorded as \texttt{PRRADIO} $=$ N carrying an explicit \texttt{RAANY} $=$ N record. The subject-level agreement between \texttt{RAANY} and \texttt{PRRADIO} for all 520 patients (416 versus 104) is therefore a genuine coding concordance rather than an artifact of absent files being coalesced to no radiotherapy. The lesion-level tracking domain supplies \texttt{LSPIRRNY}, and that flag does not merely leave gaps relative to \texttt{PRRADIO} but contradicts it often enough to be informative: 42 of the 104 patients recorded as \texttt{PRRADIO} $=$ N (40.4\%) have at least one lesion marked previously irradiated, against 110 of the 416 patients recorded as \texttt{PRRADIO} $=$ Y (26.4\%), and among the 373 patients with any lesion-level information the corresponding rates are 50.0\% and 38.1\%. A flag more often positive among patients recorded as having had no prior radiotherapy is either a different construct, for instance a lesion treated in another field or for another indication, or a merge or level-coding issue; the flag is therefore not used as a covariate and this cross-tabulation is reported here for transparency. The extract also carries \texttt{PRHNTRTI}, derived from the randomization stratification variable, and the two prior-treatment versions agree for 501 of 520 patients (96.3\%). Finally, the treated and untreated strata carry the same exposure imbalance as the cohort as a whole, so a difference between them cannot be attributed to prior treatment or prior radiotherapy alone.

\subsection{Outcome Construction and Estimand}
\label{sec:supp_outcome}

Let $\mathcal{A}_i$ denote the on-study AE records for patient $i$, let $d_{ij}$ be the recorded duration of event $j$, and let $G_{ij}$ be its recorded grade. The initial patient-level outcome is
\begin{equation}
Y_i^{\mathrm{sum}}=\sum_{j\in\mathcal{A}_i}d_{ij}\,\mathbf{1}(G_{ij}\geq 3),
\label{eq:summed_duration}
\end{equation}
where records explicitly marked as pre-treatment are excluded and severe records without a recorded duration contribute zero in the current implementation. Equation~\eqref{eq:summed_duration} is a sum of event-record durations. It is not the union of severe-AE calendar-date intervals. Accordingly, simultaneous severe events can add more than one event-day on the same date, while recurrence adds further event-days. This outcome should be called \emph{summed severe-AE-record duration} or \emph{severe-AE event-day burden}, rather than distinct days in a severe-AE state.

Of 9,934 AE records, 8,518 have a recorded duration. Among 1,718 grade 3+ records, 1,589 have a recorded duration (92.5\%). The missing-duration records are not evidence that the corresponding severe event had zero duration. They remain an outcome-reconstruction limitation. The summed-duration outcome has 123 zero values (23.7\%), 397 positive values, a median positive value of 26, and a maximum of 641 event-days. Such zero-inflated nonnegative continuous outcomes require both the point mass and the skewed positive distribution to be handled explicitly \citep{liu2019review}.

\subsubsection{Distinct Severe-AE Calendar Days}
\label{subsec:supp_distinct_endpoint}

The companion outcome counts the distinct on-study calendar days on which patient $i$ experienced at least one grade 3+ AE. Let $s_{ij}$ and $e_{ij}$ denote the imputed start and end study days for a qualifying record with complete interval boundaries, and let $\mathcal{C}_i$ be the set of these records. The endpoint is
\begin{equation}
Y_i^{\mathrm{cal}}
=
\left|
\bigcup_{j\in\mathcal{C}_i}
\{s_{ij},s_{ij}+1,\ldots,e_{ij}\}
\right|.
\label{eq:distinct_days}
\end{equation}
The interval is inclusive, consistently with the data dictionary definition of AE duration as the end date minus the start date plus one. Within each patient, overlapping or adjacent intervals are merged computationally and each calendar day is counted once. Equation~\eqref{eq:distinct_days} is an outcome-definition sensitivity analysis, not a correction to Equation~\eqref{eq:summed_duration}: summed duration gives additional weight to concurrent severe conditions, whereas distinct days measure observed time with any severe toxicity.

The original subject identifier in the 9,934-record AE file links exactly to the patient-level file containing the 520 randomized patients, of whom 514 have at least one AE record. Recomputing Equation~\eqref{eq:summed_duration} from the AE file reproduces the existing patient-level outcome for all 520 patients. The distinct-day endpoint uses records not explicitly identified as prior to treatment and with severity grade at least 3.

\subsubsection{Accounting for Unresolved Interval Boundaries}
\label{subsec:supp_boundaries}

Of 1,671 qualifying records, 1,547 (92.6\%) have complete imputed start and end study days. The remaining 124 records have an unresolved boundary. Patient disposition records supply \texttt{DSDY} (safety follow-up day) and \texttt{LASTCTDY} (last-contact study day) across all 520 patients. However, these variables reflect contact or protocol milestones rather than imputed adverse-event resolution dates, and only 354 patients completed the formal safety follow-up. Complete intervals reproduce the corresponding recorded durations exactly. Thus, Equation~\eqref{eq:distinct_days} is more directly interpretable than summed record duration but remains an observed-complete-interval endpoint.

Two methodological consequences of this construction warrant explicit documentation. The endpoint builder retains only records with a complete imputed interval and assigns zero to any patient with none, so both counts matter. In all 124 unresolved records the missing boundary is the end day rather than the start day, without exception. Those records fall across 88 patients, 51 on panitumumab and 37 on chemotherapy, whose distinct-day totals therefore omit some severe-adverse-event days. A smaller group of 13 patients has no complete qualifying interval at all and is assigned zero outright, and that group is imbalanced 10 on panitumumab against 3 on chemotherapy.

This should not be over-read. The adverse-event file carries no ongoing-at-cutoff flag, so a missing end day cannot be shown to mean the event was still running when data collection stopped. The disposition records show that 93 complete intervals end after \texttt{DSDY} while only 3 end after \texttt{LASTCTDY}, supporting \texttt{LASTCTDY} as an informative observed upper bound for sensitivity diagnostics rather than proof of complete safety ascertainment. Six unresolved intervals also begin after \texttt{DSDY} itself, which is one further reason \texttt{DSDY} cannot serve directly as a censoring variable while \texttt{LASTCTDY} remains the better observed window. What is measurable is the 10-versus-3 imbalance among the forced zeros. It runs in the direction of understating the panitumumab arm, and because it acts on the zero indicator it lands specifically on the hurdle component, which is the component this study reports as uncertain. The imbalance is small relative to the 520-patient cohort but it is not neutral with respect to the conclusion it touches.

\subsubsection{The Estimand}

For either outcome $k\in\{\mathrm{sum},\mathrm{cal}\}$, the causal estimand is the analytic-cohort average treatment effect
\begin{equation}
\Delta^{k}=\mathbb{E}\{Y_i^{k}(1)-Y_i^{k}(0)\mid i\in\mathcal{S}\},
\label{eq:ate}
\end{equation}
where $\mathcal{S}$ is the 520-person analytic cohort. Under random assignment, consistency, and adequate outcome reconstruction, Equation~\eqref{eq:ate} is the effect of \emph{assignment} on accrued event-day burden and carries an intention-to-treat interpretation within $\mathcal{S}$, subject to the caveat of Section~\ref{subsec:supp_cohort_derivation} about how $\mathcal{S}$ was formed.

To avoid misinterpretation, three properties of the estimand in Equation~\eqref{eq:ate} must be highlighted: 1. it is not a fixed-horizon restricted mean; 2. it is not a competing-risk estimand (death, progression, and discontinuation are post-randomization events that affect accumulated burden and are not modeled); and 3. it is not a rate per unit of treatment exposure. Treatment duration differs sharply between the arms, and structurally so, because panitumumab continues as maintenance after chemotherapy stops. That difference is a consequence of assignment, so it lies on the causal pathway from assignment to accrued burden: it is part of the effect being estimated rather than a nuisance to be removed, and conditioning on it would both block part of the effect and condition on a post-treatment variable. Equation~\eqref{eq:ate} is therefore the effect of assignment under the trial's own differential treatment duration, and Section~\ref{sec:supp_exposure} quantifies how much of the effect travels through that channel.

\begin{table}[htbp]
\centering
\caption{Observed summed-duration outcome by randomized arm. Positive-part summaries condition on $Y_i^{\mathrm{sum}}>0$.}
\label{tab:supp_observed}
\resizebox{\textwidth}{!}{%
\begin{tabular}{lrrrrrr}
\toprule
Arm & $n$ & Zero (\%) & Mean & Positive mean & Positive median & Maximum \\
\midrule
Chemotherapy & 260 & 24.6 & 26.75 & 35.49 & 20.5 & 263 \\
Panitumumab + chemotherapy & 260 & 22.7 & 42.43 & 54.89 & 36.0 & 641 \\
\bottomrule
\end{tabular}
}
\end{table}

\begin{table}[htbp]
\centering
\caption{Descriptive comparison of the summed-duration and distinct-calendar-day endpoints. The unadjusted difference is panitumumab plus chemotherapy minus chemotherapy.}
\label{tab:supp_endpoint_descriptive}
\resizebox{\textwidth}{!}{%
\begin{tabular}{lrrrrr}
\toprule
Endpoint & Total days & Chemotherapy mean & Panitumumab + chemotherapy mean & Unadjusted difference & Zeros \\
\midrule
Summed AE-record days & 17,988 & 26.75 & 42.43 & 15.68 & 123 \\
Distinct severe-AE calendar days & 13,107 & 19.18 & 31.23 & 12.06 & 123 \\
\bottomrule
\end{tabular}
}
\end{table}

The interval union removes 4,881 record-days (27.1\% of the original total) across 217 patients. It reduces the unadjusted arm difference by 3.62 days (23.1\%) without changing the zero-outcome count. Every positive recorded duration in this extract has an associated complete date interval, which explains why the hurdle indicator is unchanged.

\subsection{Treatment Exposure, Common-Horizon Truncation, and Person-Time}
\label{sec:supp_exposure}

The endpoints of Section~\ref{sec:supp_outcome} count days on which a patient was in a severe-adverse-event state. That quantity accrues with time on treatment, so it is interpretable as a between-arm comparison only if the reader is told how far the arms differ in exposure. They differ a great deal, and the difference is structural rather than incidental: panitumumab continues as maintenance after chemotherapy stops. Table~\ref{tab:supp_exposure} reports the distribution of TRTDUR, defined in the data definition table as the last dose date minus the first dose date plus one \citep{amgen_20050251_ddt}, so it measures time on treatment and not length of follow-up.

\begin{table}[htbp]
\centering
\caption{Treatment duration (TRTDUR, days) by randomized arm in the 520-person analytic cohort. TRTDUR is a post-randomization quantity and appears in no adjustment set.}
\label{tab:supp_exposure}
\begin{tabular}{lrrrrrr}
\toprule
Arm & $n$ & Mean & SD & Median & Maximum \\
\midrule
Chemotherapy & 260 & 83.72 & 50.14 & 89.0 & 276 \\
Panitumumab plus chemotherapy & 260 & 130.30 & 127.19 & 112.5 & 1{,}149 \\
\bottomrule
\end{tabular}
\end{table}

Mean time on treatment is 56\% longer under panitumumab, and the maxima differ by a factor of four. The consequence for the endpoint is large. Table~\ref{tab:supp_truncation} truncates every patient's observed accrual at a common study day and recomputes the unadjusted arm difference in distinct severe-adverse-event days. Patient disposition records provide a separate observed-window diagnostic based on last-contact study day (\texttt{LASTCTDY}), reported below, but they do not convert treatment duration into a uniform follow-up measure.

\begin{table}[htbp]
\centering
\caption{Mean distinct severe-adverse-event calendar days by arm under common-horizon truncation, on the 520-person basis with 260 patients per arm. The six patients with no adverse-event record contribute zero at every horizon. The primary row uses the imputed start and end days; the three remaining date variants recorded in the file are shown for comparison. This table is a diagnostic, not a corrected estimate; see the text.}
\label{tab:supp_truncation}
\resizebox{\textwidth}{!}{%
\begin{tabular}{lrrrr}
\toprule
Date variant and horizon & Chemotherapy & Panitumumab & Difference \\
\midrule
\multicolumn{4}{l}{\emph{Imputed start, imputed end (primary)}} \\
\quad Uncapped & 19.18 & 31.23 & $+12.06$ \\
\quad 120 days & 15.55 & 19.65 & $+4.10$ \\
\quad 90 days & 11.96 & 15.33 & $+3.37$ \\
\quad 84 days & 11.25 & 14.39 & $+3.14$ \\
\midrule
\multicolumn{4}{l}{\emph{Raw start, raw end}} \\
\quad Uncapped & 18.83 & 29.27 & $+10.44$ \\
\quad 84 days & 10.90 & 13.35 & $+2.44$ \\
\midrule
\multicolumn{4}{l}{\emph{Raw start, imputed end}} \\
\quad Uncapped & 19.06 & 30.18 & $+11.12$ \\
\quad 84 days & 11.13 & 13.77 & $+2.64$ \\
\midrule
\multicolumn{4}{l}{\emph{Imputed start, raw end}} \\
\quad Uncapped & 18.95 & 29.85 & $+10.90$ \\
\quad 84 days & 11.02 & 13.64 & $+2.62$ \\
\bottomrule
\end{tabular}
}
\end{table}

Two cautions govern how Table~\ref{tab:supp_truncation} is read, and they matter more than the numbers in it.

The first is that this is truncation of observed accrual, not a fixed-horizon estimand. Capping at day 84 removes events recorded after day 84, but it does not create observation for a patient whose follow-up ended before day 84. Patient disposition records supply \texttt{LASTCTDY} for all 520 patients, showing that 47 chemotherapy patients and 41 panitumumab patients have last contact before day 84. Capping complete intervals at the patient-specific last-contact day gives mean observed-window burdens of 19.18 and 30.90 days by arm; at day 84, the corresponding means are 11.25 and 14.33 days, with 213 and 219 patients reaching that contact horizon. The horizon diagnostic checks the first alternative explanation of that collapse, that differential follow-up rather than exposure produces the attenuation, but it is not a corrected estimate. At day 84 the mean among patients whose contact actually reaches the horizon is 11.25 in the chemotherapy arm against 14.93 under panitumumab, a difference of 3.68 days that exceeds the 3.08-day difference over all patients; at day 120 the corresponding difference is 5.13 days (16.12 against 21.25) against 4.03 over all patients. The difference is therefore not smaller among patients observed through the horizon, so simple truncation of the shortest-observed records is not what produces the attenuation. Restricting to patients whose contact reaches the horizon conditions on a post-randomization variable, however, so informative censoring remains possible and the direction of any resulting selection bias cannot be signed in advance. These quantities remain diagnostics, not corrected fixed-horizon estimates, because early contact loss, incomplete event intervals, and competing events require an explicit censoring and state definition. Capping applies only to intervals with complete boundaries, so these diagnostics inherit, in bounded form, the forced-zero problem documented above. The attenuation in Table~\ref{tab:supp_truncation} also reproduces across all four date variants, so it is a property of the data rather than of one imputation convention.

The second is that whether the effect reverses on a per-exposure basis depends on the normalization, and the two natural normalizations disagree. Table~\ref{tab:supp_person_time} reports both.

\begin{table}[htbp]
\centering
\caption{Descriptive per-exposure burden by arm. The patient-level mean weights every patient equally and therefore up-weights patients with short exposure; the aggregate rate pools days over pooled treatment-time. Neither is an adjusted causal effect.}
\label{tab:supp_person_time}
\resizebox{\textwidth}{!}{%
\begin{tabular}{lrrrrr}
\toprule
Arm & Total treatment days & Total severe-AE days & Patient-level mean rate & Aggregate rate \\
\midrule
Chemotherapy & 21{,}766 & 4{,}986 & 53.72 & 22.91 \\
Panitumumab plus chemotherapy & 33{,}877 & 8{,}121 & 49.91 & 23.97 \\
\bottomrule
\end{tabular}
}
\end{table}

The patient-level mean of per-100-treatment-day rates reverses the sign of the arm comparison, at 53.72 against 49.91, while the aggregate person-time rate does not, at 22.91 against 23.97. Fitting the same two models to the per-exposure rate reproduces that ambiguity rather than resolving it: ZIC-BCF-Smear estimates $5.19$ severe-adverse-event days per 100 treatment-days (95\% credible interval $[-4.39, 17.28]$, $P(>0)=0.844$) and the Gamma hurdle $-5.23$ (95\% credible interval $[-37.06, 25.19]$, $P(>0)=0.369$). Both are compatible with no effect. The Gamma benchmark's exposure-rate estimate also exhibits sensitivity to minor adjustment specifications: entering both prior-treatment indicators instead of prior radiotherapy alone shifts the point estimate from $-5.23$ to $-3.47$ (95\% credible interval $[-33.40, 24.51]$), a swing produced entirely by how a covariate differing for four patients is parameterized. This volatility illustrates the fragility of the parametric benchmark under near-redundant linear predictors. The defensible summary is severe attenuation once exposure is accounted for, not reversal.

\subsubsection{Why This Is Reported as Descriptive}

Everything in this subsection is descriptive, and the randomized analysis of Section~\ref{sec:supp_models} is left standing rather than replaced. The reason is not caution for its own sake. Treatment duration is measured after randomization, so dividing by it conditions on a post-treatment variable and can induce selection bias. The direction of that bias is not fixed in advance and is not recoverable from these files. A patient who discontinues early because of toxicity contributes a short denominator alongside a large numerator, which inflates that patient's rate and makes their arm look worse; a patient who discontinues early for an unrelated reason pushes the same arm's rate down. Which of the two dominates depends on why patients discontinued, and the extract records no discontinuation reason. The per-exposure quantities are therefore labeled descriptive rates per treatment-day rather than adjusted causal effects, and the truncation table is labeled a diagnostic.

Three methodological considerations must be kept distinct. First, treatment exposure is directly recorded and reported here. Second, patient-level contact and safety follow-up milestones are recorded, but because safety follow-up was completed by only 354 patients, they support censoring sensitivity diagnostics rather than a fully observed fixed-horizon safety endpoint. Third, the primary estimand is governed by treatment duration as a post-randomization consequence of assignment, which can neither be ignored in interpretation nor naively conditioned on in estimation. Incorporating patient disposition records therefore makes a fixed-horizon analysis feasible in principle, but does not render it assumption-free or automatically identified from the observed endpoint.

The right target remains a prespecified fixed-horizon restricted-mean or multistate formulation. Although disposition records supply last-contact timing, a valid fixed-horizon analysis still requires explicit structural choices regarding censoring rules, intercurrent treatment policies, and competing risks such as death, progression, and discontinuation. If instead a causal quantity that removes the duration channel is wanted from the existing data, it requires an explicit assumption about a treatment-duration policy, stated as such. An offset on treatment duration is not that assumption and should not be presented as one.

\begin{table}[htbp]
\centering
\caption{Disposition and observed follow-up records across randomized arms. \texttt{DSDY} denotes the recorded safety follow-up day, while \texttt{LASTCTDY} denotes the last-contact study day. Neither variable constitutes a fixed-horizon outcome by itself.}
\label{tab:supp_disposition}
\resizebox{\textwidth}{!}{%
\begin{tabular}{lrrrrrr}
\toprule
Arm & $n$ & Completed safety follow-up & Mean DSDY & Mean LASTCTDY & LASTCTDY $<84$ & LASTCTDY $\geq 84$ \\
\midrule
Chemotherapy & 260 & 181 (69.6\%) & 110.7 & 376.8 & 47 & 213 \\
Panitumumab plus chemotherapy & 260 & 173 (66.5\%) & 159.5 & 403.5 & 41 & 219 \\
\bottomrule
\end{tabular}%
}
\end{table}

\subsection{Posterior Models and Computation}
\label{sec:supp_models}

ZIC-BCF-Smear models the hurdle and positive components with separate Bayesian causal forests. The hurdle model yields posterior draws of the probability of a positive reconstructed burden, and the continuous model operates on the positive outcome. A smearing re-transformation returns the total-outcome contrast to the original event-day scale \citep{duan1983smearing}. Each outcome was fit as four independent chains of 4,000 retained draws after 5,000 burn-in iterations, with seeds 1 through 4, and the reported posterior is the pooled set of 16,000 draws. For every posterior draw, the analysis saves the unit-level total-outcome CATE and the hurdle-margin CATE. A subgroup posterior is the average of its members' unit-level effects within each draw. A contrast posterior is formed by subtracting the two subgroup averages within draw. This preserves posterior covariance, rather than comparing unrelated marginal intervals.

The Gamma-hurdle benchmark uses the same treatment coding, analytic cohort, and dummy-encoded baseline covariates for each outcome, drawn from the same shared definition file so that the two specifications cannot drift apart. It models the positive outcome with a Gamma component and the zero versus positive component with a hurdle model. Each fit used four chains of 4,000 draws, matched to the ZIC-BCF-Smear budget, with the same four seeds; the benchmark posterior is a Laplace approximation, so its draws are independent and no burn-in iteration is involved. The Gamma fit also returns posterior hurdle probabilities, allowing direct comparison of the probability-difference margin. Flexible Bayesian models for zero-inflated outcomes have been developed precisely because both distributional features can matter for causal estimation and uncertainty propagation \citep{oganisian2021zeroinflated}.

Convergence diagnostics are computed by the analysis scripts and follow the construction of Section~\ref{subsec:supp_convergence}. On the summed-duration endpoint the largest Gelman--Rubin statistic across all monitored quantities is 1.001 for ZIC-BCF-Smear and 1.000 for the Gamma hurdle, the smallest effective sample sizes are 2,670 and 16,000, and the largest absolute Geweke statistics are 2.03 and 1.42. For the Gamma hurdle the first two values are exact by construction rather than evidence of mixing, its Laplace posterior being sampled independently. On the distinct-calendar-day endpoint the corresponding values are 1.002, 2,777, and 2.21 for ZIC-BCF-Smear and 1.000, 16,000, and 1.42 for the Gamma hurdle. The assessed Geweke set is large, so a plain reading of its maximum needs a multiplicity context: eighteen functionals across four chains give 72 standard normal statistics per fit, whose 5\% family-wise critical value is about 3.38, and the observed maxima of 2.03 and 2.21 sit well below it. The final ZIC-BCF-Smear fits use a 5,000-iteration burn-in because the originally reported 2,000-iteration configuration did not clear this bar: there the largest statistics were 3.67 on the summed-duration endpoint and 4.90 on the distinct-calendar-day endpoint, both attaching to $\sigma_c$, the log-scale residual standard deviation that sets the smearing factor, with family-wise probabilities of about 0.017 for the summed-duration endpoint and about $7 \times 10^{-5}$ for the distinct-calendar-day endpoint under the correction. An investigation by direct replication, archived with the analysis scripts, established both the mechanism and the remedy. Conditional on the forest fit, $\sigma_c$ receives a conjugate inverse-chi-square update, so its individual draws are nearly uncorrelated (lag-one autocorrelation about $0.04$, effective sample size about $93\%$ of nominal), which is why effective sample size and Gelman--Rubin diagnostics look clean throughout; but the joint level of $\sigma_c$ and the fitted log-scale surface equilibrates through a two-way Gibbs coupling whose relaxation time varies by chain, so at the shorter burn-in some retained chains carried a small decaying level transient. Across twenty-four fresh seed combinations at that configuration the per-chain Geweke statistics of $\sigma_c$ had standard deviation about 1.46 against the calibrated null value of 1, and every window pairing touching the early portion of a flagged chain flagged together. Extending burn-in to 5,000 iterations on the identical seeds removed every family-wise-threatening value while moving posterior means by at most $0.15\%$, orders of magnitude below reported precision, so the longer burn-in was adopted for all final models. The two remaining diagnostics cannot detect such transients by construction: $\widehat{R}$ pools split-half means across chains, against which one chain's early tilt is negligible, and effective sample size measures variation around a fixed level, which a one-off decaying shift does not generate, leaving the early-versus-late Geweke contrast as the only of the three sensitive to residual burn-in. Because the quantity involved is the one that sets the smearing factor, Section~\ref{subsec:supp_smearing_test}'s locally smeared sensitivity analysis remains the relevant robustness check, and it bounds the consequence of exactly this kind of residual-scale behavior.

The homoskedasticity assumption behind the smearing factor is tested as described in Section~\ref{subsec:supp_smearing_test}. The studentized Breusch--Pagan test gives $p = 0.122$ on summed duration and $p = 0.157$ on distinct calendar days. The locally smeared sensitivity analysis moves the summed-duration ATE from 14.54 to 11.70 event-days, a 19.5\% reduction that remains within the scalar-smeared posterior interval, and moves the distinct-day ATE from 9.75 to 9.39, a 3.7\% reduction. The conclusions of this study are therefore not carried by the scalar smearing assumption.

\begin{table}[htbp]
\centering
\caption{Posterior average treatment effects. Positive values favor a higher burden or a higher probability of a positive reconstructed severe-AE burden under panitumumab assignment.}
\label{tab:supp_model_comparison}
\resizebox{\textwidth}{!}{%
\begin{tabular}{lrrr}
\toprule
Model and estimand & Estimate & 95\% credible interval & $P(>0)$ \\
\midrule
ZIC-BCF-Smear, summed duration (event-days) & 14.27 & [4.69, 24.23] & 0.999 \\
Gamma Hurdle, summed duration (event-days) & 15.95 & [6.15, 26.30] & 0.999 \\
\emph{Unadjusted arm difference, summed duration} & \emph{15.68} & \emph{--} & \emph{--} \\
ZIC-BCF-Smear, distinct calendar days & 10.00 & [3.57, 16.62] & 0.999 \\
Gamma Hurdle, distinct calendar days & 12.20 & [5.56, 19.23] & 1.000 \\
\emph{Unadjusted arm difference, distinct calendar days} & \emph{12.06} & \emph{--} & \emph{--} \\
ZIC-BCF-Smear, positive reconstructed burden, summed-duration fit & 0.0288 & [$-0.0379$, 0.1019] & 0.786 \\
Gamma Hurdle, positive reconstructed burden, summed-duration fit & 0.0171 & [$-0.0554$, 0.0893] & 0.683 \\
ZIC-BCF-Smear, positive reconstructed burden, distinct-day fit & 0.0273 & [$-0.0392$, 0.1023] & 0.774 \\
Gamma Hurdle, positive reconstructed burden, distinct-day fit & 0.0171 & [$-0.0554$, 0.0893] & 0.683 \\
\bottomrule
\end{tabular}
}
\end{table}

Both models retain a credible positive effect when overlapping events are counted once. Relative to summed duration, the ZIC-BCF-Smear ATE decreases by 4.27 days and the Gamma-hurdle ATE decreases by 3.75 days. This expected attenuation reflects the change in endpoint rather than a reversal of the safety conclusion. The probability-margin estimates remain uncertain and almost unchanged because the zero indicator is unchanged. Thus, the supported conclusion concerns accumulated recorded burden, not a clearly established change in the fraction of patients with a positive reconstructed severe-AE burden. The Gamma-hurdle benchmark now uses the same four-chain, $4{,}000$-draw-per-chain budget as ZIC-BCF-Smear on both endpoints, and every estimate above is within Monte Carlo error of the shorter benchmark run, so the specification comparison no longer mixes unequal posterior sample sizes.

Both models also sit close to the unadjusted arm differences, and under randomization that is the benchmark rather than a weakness. Because assignment was randomized, the raw contrast is itself an unbiased estimate of the average treatment effect in this cohort, so a fitted model whose point estimate departed sharply from it would owe the reader an explanation. The Gamma-hurdle distinct-day estimate of 12.20 days against an unadjusted 12.06 is therefore reassuring rather than uninformative. What the adjusted fits contribute is not bias removal, which randomization already supplies, but precision, the decomposition into participation and intensity, and the resolution of the effect by covariate profile. This point was previously stated only in passing and only in the supplement, and it is now stated in the main text as well.

An independent Gamma-hurdle fit of the distinct-day endpoint with seed 2026, outside the four seeds used by the main analysis, is not used for the updated primary summary. The four-chain result above has $\widehat{R}=1.000$ and a stable positive posterior, which is sufficient for the reported specification comparison.

\begin{figure}[htbp]
\centering
\includegraphics[width=0.96\textwidth]{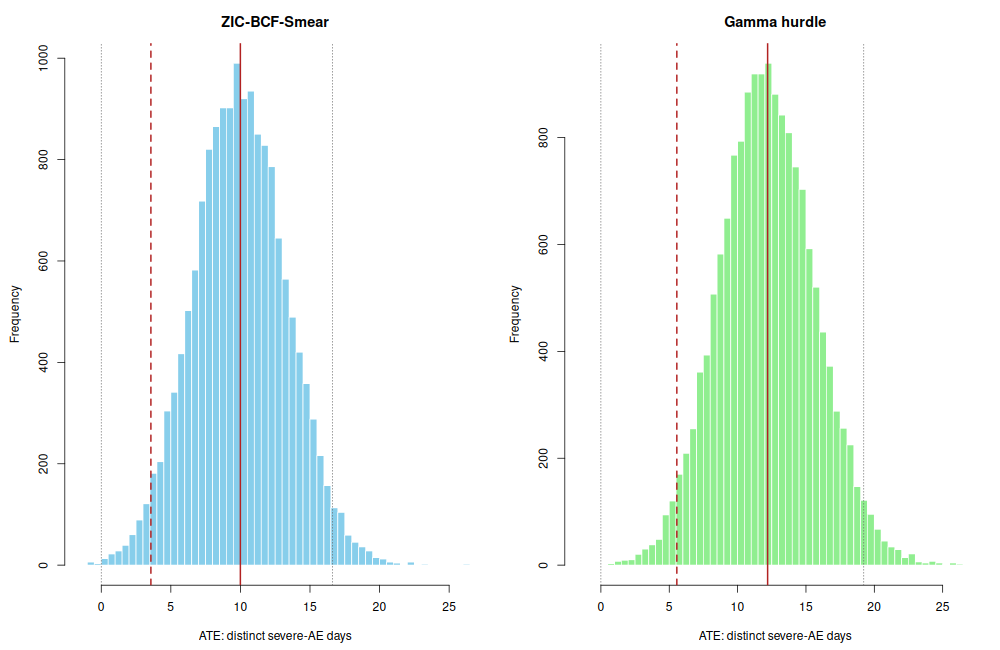}
\caption{Posterior distributions of the average treatment effect on distinct severe-AE calendar days under ZIC-BCF-Smear (left) and the Gamma-hurdle benchmark (right). In each panel the solid vertical line marks the posterior mean, the dashed line the 2.5\% posterior quantile, and the dotted lines the zero reference and the 97.5\% posterior quantile.}
\label{fig:supp_distinct_ate}
\end{figure}

Figure~\ref{fig:supp_distinct_ate} displays the two posterior distributions of the distinct-calendar-day ATE that underlie the third and fourth rows of Table~\ref{tab:supp_model_comparison}. Three features are worth noting. First, both posteriors place essentially all of their mass to the right of zero, which is the graphical counterpart of $P(\mathrm{ATE}>0)=0.999$ for ZIC-BCF-Smear and $1.000$ for the Gamma hurdle; the zero reference sits in the extreme left tail rather than outside the plotted support, so the evidence is strong without being degenerate. Second, the ZIC-BCF-Smear posterior is centred lower (10.00 days) and is mildly right-skewed, with a longer upper tail extending beyond 20 days. That shape is expected from the smearing re-transformation, which propagates uncertainty in the positive-part residual distribution onto the original day scale and therefore does not force a symmetric effect posterior. Third, the two posteriors overlap over almost their entire range, so the 2.20-day gap between the ZIC-BCF-Smear and Gamma-hurdle point estimates is small relative to the width of either distribution and should not be read as a substantive disagreement about the average effect.

\subsection{Subgroup Results and Benchmark Sensitivity}
\label{sec:supp_subgroups}

Table~\ref{tab:supp_zic_subgroups} reports the ZIC-BCF-Smear CATEs. The subgroup effects are all positive in posterior mean, but the width of the intervals grows in small strata. No pairwise total-outcome contrast has a 95\% credible interval excluding zero. The strongest ZIC-BCF-Smear signals, ECOG 0 versus ECOG 1 and HPV-negative versus HPV-not-recorded, remain compatible with no contrast.

\begingroup
\small
\begin{longtable}{llrrrr}
\caption{ZIC-BCF-Smear subgroup CATEs for summed severe-AE-record duration.}\label{tab:supp_zic_subgroups}\\
\toprule
Covariate & Level & $n$ & CATE & 95\% credible interval & $P(>0)$ \\
\midrule
\endfirsthead
\toprule
Covariate & Level & $n$ & CATE & 95\% credible interval & $P(>0)$ \\
\midrule
\endhead
Overall & All & 520 & 14.27 & [4.69, 24.23] & 0.999 \\
Sex & Male & 455 & 14.20 & [4.53, 24.37] & 0.999 \\
Sex & Female & 65 & 14.76 & [2.28, 29.89] & 0.990 \\
ECOG status & ECOG 0 & 150 & 18.74 & [5.76, 35.51] & 1.000 \\
ECOG status & ECOG 1 & 370 & 12.46 & [3.27, 22.47] & 0.996 \\
Diagnosis site & Hypopharynx & 75 & 13.37 & [3.56, 24.47] & 0.997 \\
Diagnosis site & Larynx & 157 & 14.00 & [3.49, 26.50] & 0.997 \\
Diagnosis site & Oral cavity & 142 & 14.57 & [3.22, 28.53] & 0.995 \\
Diagnosis site & Oropharynx & 146 & 14.72 & [3.92, 27.77] & 0.998 \\
HPV status & Negative & 270 & 16.85 & [5.88, 28.96] & 1.000 \\
HPV status & Positive & 69 & 18.17 & [4.50, 36.73] & 0.997 \\
HPV status & Not recorded & 181 & 8.93 & [$-$1.02, 19.27] & 0.960 \\
Disease status & Newly diagnosed & 113 & 16.97 & [4.43, 33.48] & 0.998 \\
Disease status & Recurrent & 407 & 13.52 & [4.14, 23.54] & 0.999 \\
Prior radiotherapy & No prior radiotherapy & 104 & 16.87 & [4.10, 33.76] & 0.997 \\
Prior radiotherapy & Prior radiotherapy & 416 & 13.62 & [4.17, 23.65] & 0.999 \\
Age group & $<50$ & 72 & 15.83 & [4.87, 29.76] & 0.999 \\
Age group & 50--59 & 240 & 13.92 & [4.45, 23.90] & 0.999 \\
Age group & 60--69 & 172 & 13.92 & [4.35, 24.31] & 0.998 \\
Age group & $70+$ & 36 & 15.15 & [4.51, 27.95] & 0.998 \\
\bottomrule
\end{longtable}
\endgroup

The Gamma-hurdle benchmark is deliberately included as a specification sensitivity analysis, not as a selector of a clinically preferred subgroup claim. Its total-duration ECOG 0 minus ECOG 1 contrast is 26.94 event-days (95\% credible interval [4.74, 51.94]), whereas the corresponding ZIC-BCF-Smear contrast is 6.28 event-days (95\% credible interval [$-2.57$, 22.20]). The cross-model disagreement remains substantive, and prevents a robust claim that baseline ECOG modifies the total event-day burden.

The models also disagree on certain hurdle-margin contrasts. Gamma Hurdle estimates ECOG 0 minus ECOG 1 at 0.169 (95\% credible interval [0.014, 0.323]) and HPV-negative minus HPV-not-recorded at 0.224 (95\% credible interval [0.066, 0.387]). The corresponding ZIC-BCF-Smear contrasts are 0.041 (95\% credible interval [$-0.026$, 0.176]) and 0.067 (95\% credible interval [$-0.020$, 0.238]), and no ZIC-BCF-Smear hurdle-margin contrast excludes zero. Since the average hurdle effect is uncertain for both models, these isolated benchmark-dependent findings are hypothesis-generating only.

\subsubsection{Two Contrasts That Need Naming}
\label{subsec:supp_named_contrasts}

Two of the contrasts in these tables should be described for what they measure rather than for the covariate they are labeled with.

The HPV-not-recorded category is created by coalescing patients with no recorded baseline HPV result, and it contains 181 of the 520 patients. Any contrast against it is therefore a missingness contrast, not a comparison of two measured biological states, and the HPV-negative minus HPV-not-recorded contrast that the Gamma benchmark reports as credible on the participation margin is a statement about which patients had a recorded result rather than about HPV biology. It is retained in the tables because dropping the category would silently drop those 181 patients from the cohort, but it should not be read as evidence about human papillomavirus.

Prior radiotherapy is the clinically motivated candidate modifier, and it behaves like a candidate rather than an established one. On the summed-duration endpoint the effect is larger among the 104 patients without prior radiotherapy (16.87 event-days, 95\% credible interval [4.10, 33.76]) than among the 416 with prior radiotherapy (13.62, [4.17, 23.65]), with a contrast of 3.25 event-days (95\% credible interval [$-7.61$, 17.82]); on the distinct-day endpoint the corresponding estimates are 13.09 and 9.22 days, with contrast 3.87 days (95\% credible interval [$-4.10$, 15.62]). These contrasts are not credible, and the smaller stratum is limited. The composite PRHNTRTC differs from PRRADIO for only four patients, all in the chemotherapy arm, so its subgroup contrast is a near-duplicate of the one reported here; it is omitted from the prespecified lists because it would add an unadjusted comparison that carries no additional information. The radiotx domain confirms PRRADIO exactly at the subject level for all 520 patients, none of them absent from the domain; the lesion-level flag is not a substitute because it is lesion-specific and, as documented above, contradicts PRRADIO for 42 of the 104 unirradiated patients.

\begin{figure}[htbp]
\centering
\includegraphics[width=0.78\textwidth]{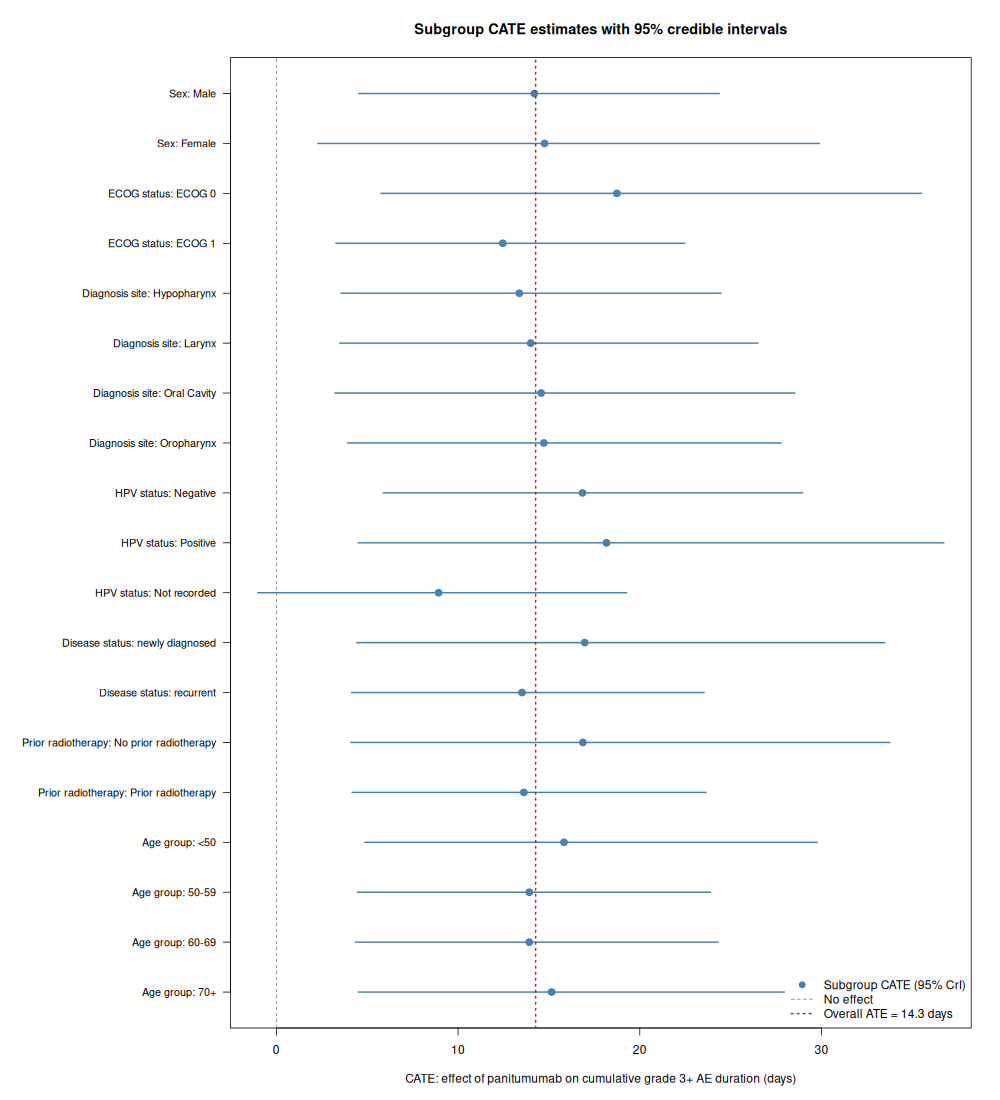}
\caption{ZIC-BCF-Smear subgroup CATEs for summed severe-AE-record duration, with 95\% credible intervals. The dashed vertical line at zero is the no-effect reference and the dotted vertical line is the overall ATE of 14.3 event-days.}
\label{fig:supp_zic_cates}
\end{figure}

\begin{figure}[htbp]
\centering
\includegraphics[width=0.78\textwidth]{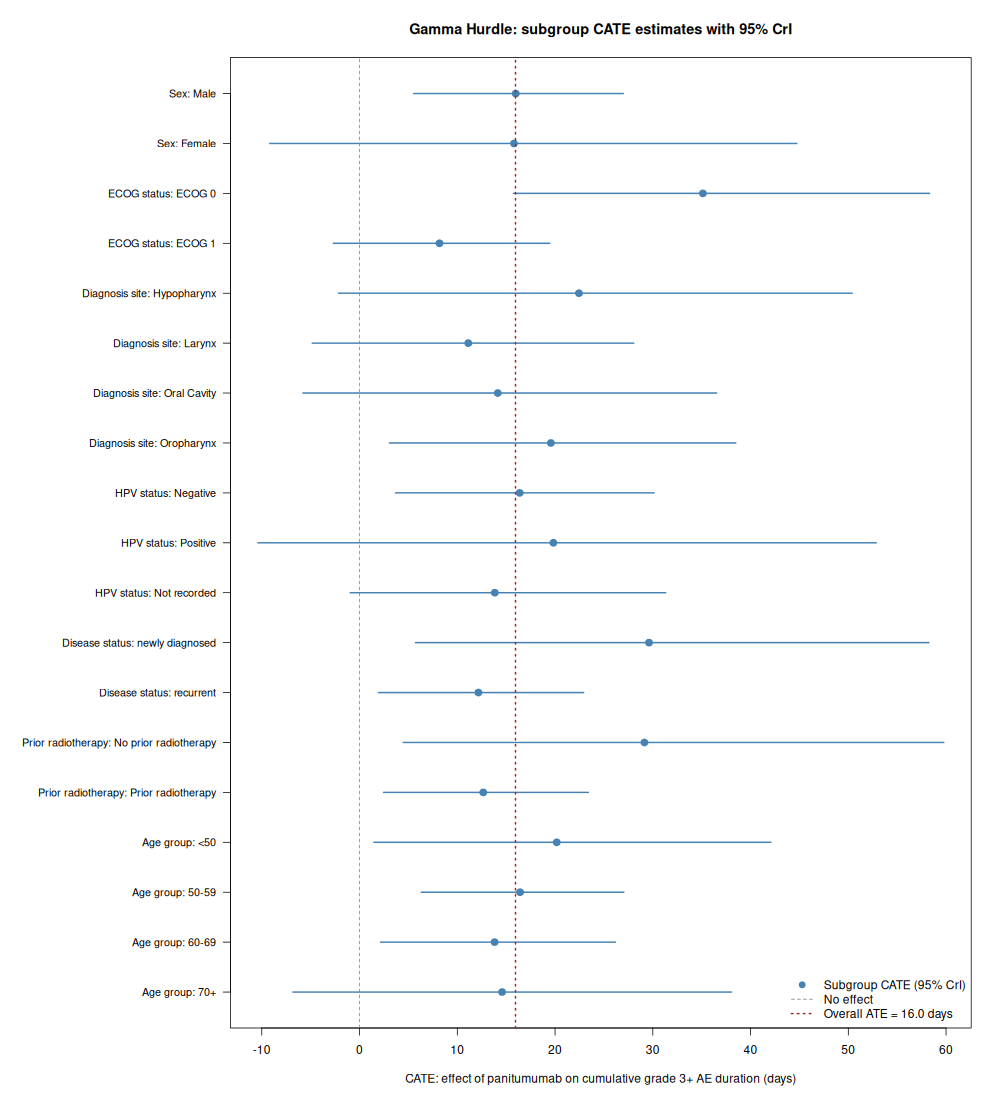}
\caption{Gamma-hurdle subgroup CATEs for the same summed severe-AE-record duration outcome, on the same 95\% credible-interval scale. The dotted vertical line is the Gamma-hurdle overall ATE of 16.1 event-days. Note that the horizontal axis spans roughly twice the range of Figure~\ref{fig:supp_zic_cates}.}
\label{fig:supp_gamma_cates}
\end{figure}

Figure~\ref{fig:supp_zic_cates} presents the same estimates graphically and makes the shrinkage behaviour of the causal forest visible. Every subgroup interval lies to the right of zero except HPV-not-recorded, and every posterior mean falls within roughly five event-days of the overall ATE, so the ZIC-BCF-Smear fit describes a broadly homogeneous effect whose apparent subgroup variation is small relative to within-subgroup uncertainty. Interval width tracks subgroup size almost monotonically, with the widest intervals for HPV-positive patients ($n=69$), women ($n=65$), and patients aged 70 years or older ($n=36$), which is the expected consequence of borrowing strength across a tree ensemble rather than fitting separate stratum-specific models.

Figure~\ref{fig:supp_gamma_cates} shows the corresponding Gamma-hurdle panel and, read against Figure~\ref{fig:supp_zic_cates}, isolates where the two specifications actually part company. The Gamma-hurdle estimates are more dispersed: several intervals now cross zero (women, hypopharynx, larynx, oral cavity, HPV-positive, HPV-not-recorded, and patients aged 70 years or older), and the ECOG contrast is driven by an ECOG 0 posterior mean of 35.2 event-days against 8.3 event-days for ECOG 1, a separation with no counterpart in the ZIC-BCF-Smear panel. Because both models are fit to identical data, covariates, and treatment coding, the difference is attributable to the specification: the parametric Gamma positive component with linear dummy-encoded covariates extrapolates subgroup means without the regularization that the causal forest imposes. The comparison therefore supports the interpretation of the ECOG signal as specification sensitivity rather than as corroborated effect modification.

\subsubsection{Distinct-Calendar-Day Subgroup Sensitivity}
\label{subsec:supp_distinct_subgroups}

Table~\ref{tab:supp_distinct_zic_subgroups} reports the ZIC-BCF-Smear CATEs for the overlap-adjusted endpoint. Posterior means are positive in every prespecified subgroup, but no between-group contrast has a 95\% credible interval excluding zero. The largest contrast is ECOG 0 minus ECOG 1 at 6.69 days (95\% credible interval [$-0.83$, 20.43]). The corresponding Gamma-hurdle contrast is 23.34 days (95\% credible interval [7.83, 40.98]). Because the ZIC-BCF-Smear interval includes zero, this apparent ECOG heterogeneity remains model-dependent and should not be presented as a substantive subgroup finding.

\begingroup
\small
\begin{longtable}{llrrrr}
\caption{ZIC-BCF-Smear subgroup CATEs for distinct severe-AE calendar days.}\label{tab:supp_distinct_zic_subgroups}\\
\toprule
Covariate & Level & $n$ & CATE & 95\% credible interval & $P(>0)$ \\
\midrule
\endfirsthead
\toprule
Covariate & Level & $n$ & CATE & 95\% credible interval & $P(>0)$ \\
\midrule
\endhead
Overall & All & 520 & 10.00 & [3.57, 16.62] & 0.999 \\
Sex & Male & 455 & 10.12 & [3.55, 17.01] & 0.999 \\
Sex & Female & 65 & 9.10 & [0.40, 18.63] & 0.980 \\
ECOG status & ECOG 0 & 150 & 14.76 & [4.86, 28.25] & 1.000 \\
ECOG status & ECOG 1 & 370 & 8.07 & [1.85, 14.58] & 0.994 \\
Diagnosis site & Hypopharynx & 75 & 9.55 & [2.65, 17.22] & 0.997 \\
Diagnosis site & Larynx & 157 & 10.88 & [3.15, 20.34] & 0.998 \\
Diagnosis site & Oral cavity & 142 & 9.46 & [1.82, 18.30] & 0.993 \\
Diagnosis site & Oropharynx & 146 & 9.80 & [2.37, 18.36] & 0.997 \\
HPV status & Negative & 270 & 11.92 & [4.36, 20.05] & 1.000 \\
HPV status & Positive & 69 & 12.22 & [3.17, 23.74] & 0.997 \\
HPV status & Not recorded & 181 & 6.28 & [$-$0.73, 13.50] & 0.961 \\
Disease status & Newly diagnosed & 113 & 13.02 & [3.51, 25.67] & 0.999 \\
Disease status & Recurrent & 407 & 9.16 & [2.89, 15.61] & 0.999 \\
Prior radiotherapy & No prior radiotherapy & 104 & 13.09 & [3.40, 26.02] & 0.998 \\
Prior radiotherapy & Prior radiotherapy & 416 & 9.22 & [2.97, 15.66] & 0.999 \\
Age group & $<50$ & 72 & 11.70 & [3.86, 21.60] & 0.999 \\
Age group & 50--59 & 240 & 9.85 & [3.41, 16.54] & 0.999 \\
Age group & 60--69 & 172 & 9.50 & [3.09, 16.26] & 0.999 \\
Age group & $70+$ & 36 & 9.93 & [2.81, 18.24] & 0.997 \\
\bottomrule
\end{longtable}
\endgroup

\begin{figure}[htbp]
\centering
\includegraphics[width=0.78\textwidth]{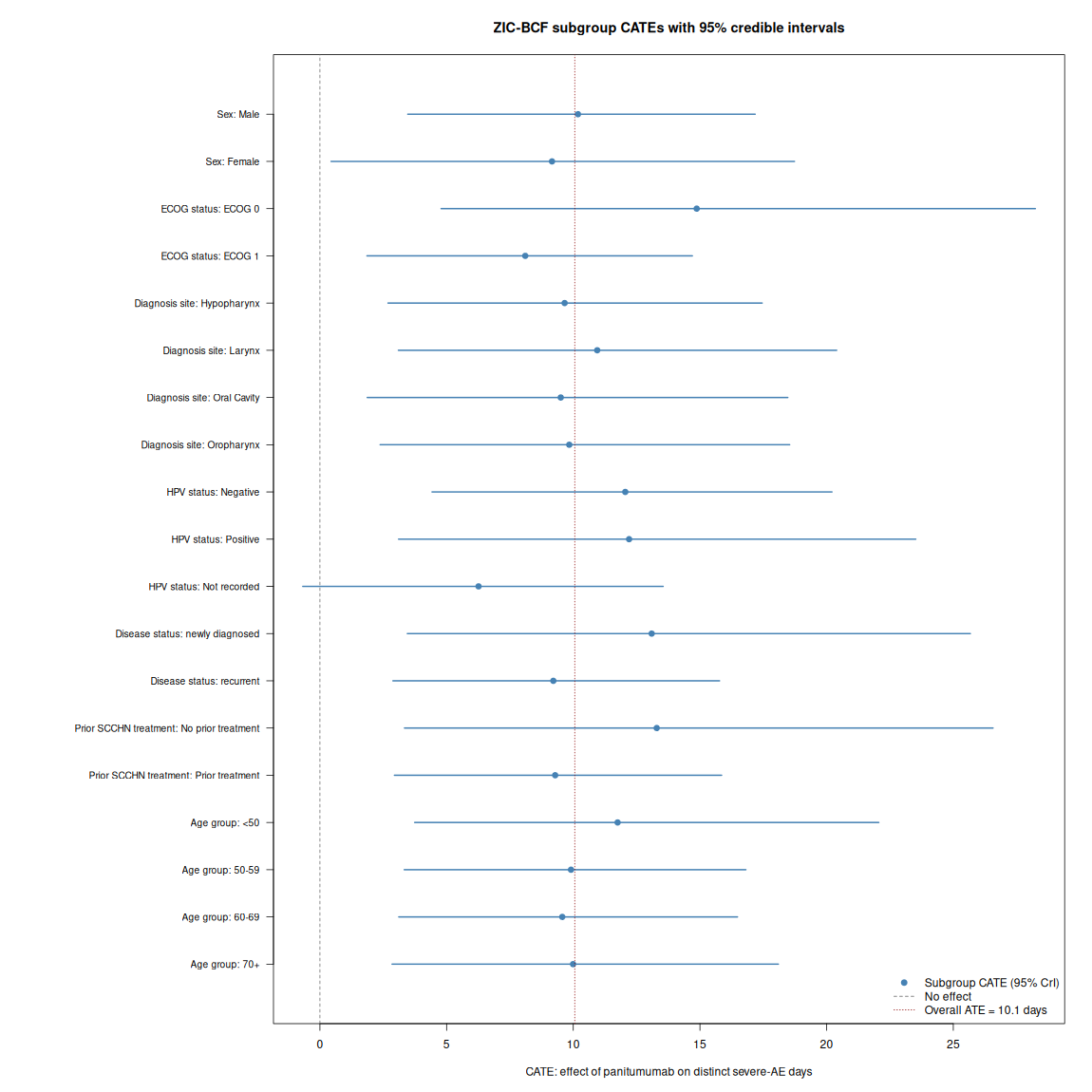}
\caption{ZIC-BCF-Smear subgroup CATEs for distinct severe-AE calendar days, with 95\% credible intervals. The dashed vertical line at zero is the no-effect reference and the dotted vertical line is the overall ATE of 10.1 days.}
\label{fig:supp_distinct_zic_cates}
\end{figure}

Figure~\ref{fig:supp_distinct_zic_cates} plots these estimates and should be compared directly with Figure~\ref{fig:supp_zic_cates}. The two panels have the same ordering of subgroups and nearly the same shape, but the distinct-day axis is compressed by roughly a third, which is the graphical signature of counting each calendar day once. No subgroup changes sign, and the only interval that still covers zero is HPV-not-recorded, exactly as in the summed-duration fit. The ECOG 0 estimate remains the largest and the widest, and its interval visibly overlaps the ECOG 1 interval, which is why the ECOG contrast in this endpoint also fails to exclude zero. Taken together, the two figures show that the overlap adjustment rescales the ZIC-BCF-Smear effect surface without reorganizing it, so the conclusions about heterogeneity are not artefacts of double-counting concurrent events.

\subsection{Position in the Clinical and Statistical Literature}
\label{sec:supp_literature}

The clinical trial literature supports four distinct findings. First, the direction of the average burden contrasts is clinically coherent. The primary SPECTRUM report found that the panitumumab regimen had more severe toxicity in several conventional safety categories \citep{vermorken2013spectrum}. Second, neither numerical ATE is directly comparable with a published SPECTRUM effect, because standard safety tables summarize event incidence, grade, and category rather than summed record durations or overlap-adjusted calendar days. Third, previous SPECTRUM subgroup discussions focused principally on efficacy and HPV or p16 status; the available HPV-specific safety analysis is a conference abstract and does not estimate baseline-modified causal effects on either burden outcome \citep{stohlmacher2012hpv}. Fourth, the review identified no prior study conducting the present secondary causal analysis of the SPECTRUM data.

The novel contribution is therefore specific and clinically interpretable. It combines 1. a patient-centred endpoint defined as distinct calendar days with at least one grade 3+ AE rather than AE counts or summed record durations, 2. an ATE of randomized panitumumab assignment estimated on that day scale, 3. a direct sensitivity analysis showing that the credibly positive overall conclusion persists after concurrent events are not double-counted, and 4. flexible, uncertainty-aware subgroup-effect estimation for a semicontinuous outcome. To our knowledge, this is the first secondary causal analysis of the SPECTRUM data using an overlap-adjusted severe-AE-day endpoint. This claim does not recast the known clinical direction as a new discovery and does not claim priority over every cumulative-toxicity analysis in oncology. The method is motivated by a broad statistical literature on semicontinuous outcomes and flexible causal modeling \citep{liu2019review,oganisian2021zeroinflated}.

Against this background, the primary question is how far the ZIC-BCF-Smear results themselves converge with what the clinical literature already supports. The convergence is strongest, and is essentially complete, for the direction and stability of the average burden effect. ZIC-BCF-Smear estimates 14.27 additional event-days (95\% credible interval [4.69, 24.23], $P(>0)=0.999$) on the summed-duration scale and 10.00 additional days (95\% credible interval [3.57, 16.62], $P(>0)=0.999$) on the overlap-adjusted calendar-day scale. Both sign and near-certainty match the qualitative SPECTRUM safety finding that the panitumumab-containing regimen produced more severe toxicity \citep{vermorken2013spectrum}, and the conclusion survives the one outcome-definition change that could plausibly have manufactured it, namely the double-counting of concurrent severe events. What the literature does not provide is a numerical benchmark: no published SPECTRUM analysis reports an effect on either duration scale, so external agreement can be assessed for direction and for robustness, but not for magnitude. The 4.27-day attenuation between the two endpoints is best read as a statement about how much of the summed-duration burden reflects simultaneity, and it is internal evidence rather than a replication of any published quantity.

Convergence is weaker, though not contradictory, at the level of mechanism. ZIC-BCF-Smear finds no credible population-average effect on the probability of a positive reconstructed severe-AE burden, estimating a 2.88 percentage-point increase (95\% credible interval $-3.79$ to 10.19) from the summed-duration fit and 2.73 percentage points (95\% credible interval $-3.92$ to 10.23) from the distinct-day fit. That margin is not literally the probability of experiencing any severe adverse event, since 13 patients with severe events enter the data with a zero burden because their observation intervals were incomplete. The published safety evidence emphasized particular adverse-event categories and grades rather than an aggregate any-severe-AE indicator, so there is no literature quantity that this margin either confirms or refutes. The substantive implication is internal to the model decomposition: the credible effect is carried by the positive component, that is, by how much severe-AE time accumulates among patients whose reconstructed burden is positive, rather than by the hurdle component that governs whether that burden is positive at all. Since 76.3\% of the cohort has a positive burden in both arms, the hurdle margin has limited room to move, and the supported claim concerns positive intensity rather than occurrence. That intensity can reflect event duration, event multiplicity, overlap, treatment exposure, and incomplete interval reconstruction, so it should not be labeled persistence specifically.

Convergence is weakest for effect modification, and here the ZIC-BCF-Smear results are best described as literature-consistent but uninformative. Posterior means are positive in every prespecified subgroup on both endpoints, no pairwise contrast excludes zero, and the strongest signals, ECOG 0 versus ECOG 1 and HPV-negative versus HPV-not-recorded, remain compatible with no contrast. The SPECTRUM literature does not establish either ECOG or HPV status as a modifier of the causal burden of panitumumab-associated severe toxicity, and the available HPV-specific analysis is a conference abstract that does not analyze either duration endpoint \citep{stohlmacher2012hpv}. A regularized model returning no credible heterogeneity is therefore concordant with the absence of established modifiers.

The Gamma-hurdle benchmark converges with the clinical literature on the same principal average-effect result. Its estimated effect of 15.95 event-days (95\% credible interval [6.15, 26.30]) for summed record durations, and 12.20 days (95\% credible interval [5.56, 19.23]) for distinct calendar days, agree in sign and broad magnitude with the corresponding ZIC-BCF-Smear estimates, and its any-severe-AE margin of 1.71 percentage points (95\% credible interval $-5.54$ to 8.93) is likewise uncertain. Agreement between a flexible nonparametric and a parametric specification on both endpoints strengthens the average-effect conclusion, since it is not an artefact of either functional form. The Gamma-hurdle point estimates are close to the unadjusted arm differences, and under randomization that is the internal benchmark a model should meet rather than a shortcoming: the raw contrast is itself unbiased for the cohort average effect. It is not external validation, since the published trial report did not analyze either duration endpoint, but neither is it a weakness of the fit.

The Gamma-hurdle subgroup findings are less literature-convergent, and this is where the two specifications diverge. For summed duration, Gamma Hurdle estimates an ECOG 0 minus ECOG 1 contrast of 26.94 event-days (95\% credible interval [4.74, 51.94]), while the corresponding ZIC-BCF-Smear contrast of 6.28 event-days has a credible interval that includes zero. For distinct calendar days, the contrast is 23.34 days (95\% credible interval [7.83, 40.98]) under Gamma hurdle and 6.69 days (95\% credible interval [$-0.83$, 20.43]) under ZIC-BCF-Smear. Gamma Hurdle also gives positive hurdle-margin contrasts for ECOG 0 versus ECOG 1 and HPV-negative versus HPV-not-recorded, whereas the corresponding ZIC-BCF-Smear contrasts remain uncertain. Given that the clinical literature establishes neither modifier, these Gamma-specific signals are best interpreted as model-sensitive hypotheses rather than as externally corroborated clinical effect modification.

\subsection{Limitations}
\label{sec:supp_limitations}

The first limitation is outcome reconstruction. The sum in Equation~\eqref{eq:summed_duration} can double count concurrent events, and 7.5\% of severe AE records lack its duration field. Equation~\eqref{eq:distinct_days} directly addresses overlap and confirms the positive overall conclusion, but 124 qualifying records have an unresolved end boundary, affecting 88 patients, and 13 patients with no complete qualifying interval are assigned zero outright with a 10-to-3 imbalance against panitumumab, as detailed in Section~\ref{subsec:supp_boundaries}. Patient disposition records provide \texttt{DSDY} and \texttt{LASTCTDY} across all participants, but these do not resolve missing event end dates; the companion endpoint should therefore be interpreted as observed distinct calendar days represented by complete intervals, rather than total toxicity time through a common follow-up horizon. Neither endpoint silently replaces the other because they answer different clinical questions.

The second limitation is differential treatment exposure, and it is the most consequential one. Mean time on treatment is 56\% longer under panitumumab, structurally so because the drug continues as maintenance after chemotherapy stops, and truncating accrual at a common study day collapses the unadjusted arm difference from 12.06 days to about 3 days. Section~\ref{sec:supp_exposure} develops this in full. The estimand is coherent, and under randomization it is identified, but it is the effect of assignment under the trial's own differential duration and not a per-exposure or fixed-horizon quantity. Because treatment duration is post-randomization it can be neither ignored in interpretation nor conditioned on in estimation, and the per-exposure analyses are reported as descriptive for that reason. A larger accumulated burden may therefore arise through a mixture of more toxicity, longer event duration, more co-occurring events, incomplete intervals, and more time on treatment, and this analysis does not separate those channels.

The third limitation is incomplete follow-up ascertainment. Trial disposition records document \texttt{DSDY} and \texttt{LASTCTDY} for all 520 patients, but only 354 completed protocol-specified safety follow-up, and last contact does not ensure comprehensive safety assessment. The analysis therefore does not standardize follow-up, model competing death or progression, or estimate a censoring-adjusted fixed-horizon endpoint. A prespecified fixed-horizon restricted-mean or multistate analysis is now feasible in principle, but it would require an explicit censoring rule, treatment policy, and competing-event definition. The present manuscript reports the last-contact-capped endpoint only as a descriptive diagnostic.

The fourth limitation concerns the analytic cohort. Randomization licenses the causal reading, and randomization attaches to the parent trial's 657 patients rather than automatically to the 520-person extract analyzed here. Section~\ref{subsec:supp_cohort_derivation} sets out what the supplied files do and do not establish about that reduction; the 520-to-514 reconciliation is complete, and the 657-to-520 step is not documented in any supplied file.

The fifth limitation concerns heterogeneity. Several subgroups are small, especially women, HPV-positive patients, patients aged 70 years or older, and the 100 patients with no prior head-and-neck treatment. The subgroup comparisons are exploratory and are not multiplicity-adjusted. Two of them measure something other than the covariate they are named for or rest on a stratum too small to resolve, as Section~\ref{subsec:supp_named_contrasts} sets out. And the Gamma-hurdle versus ZIC-BCF-Smear disagreement for the ECOG contrast persists under the distinct-day endpoint. This is useful sensitivity evidence, but it reinforces that no baseline effect-modification claim should be based on one model alone.

The sixth limitation concerns the resolution of prior-treatment information. \texttt{PRHNTRTC} remains a composite of prior radiotherapy and other prior treatments, whereas the baseline demographic file supplies the specific patient-level \texttt{PRRADIO} indicator alongside the detailed \texttt{radiotx} domain. \texttt{PRRADIO} agrees exactly with subject-level radiotherapy administration records and is utilized in the primary adjustment set and subgroup specifications. The lesion-level \texttt{LSPIRRNY} flag contradicts \texttt{PRRADIO} for 42 of the 104 patients recorded as not irradiated (Section~\ref{sec:supp_context}), so it cannot be used to recover a more granular patient-level radiation history until that discrepancy is resolved. Radiation-site and lesion-specific effect modification therefore remains an extension rather than an established result.

\subsection{Reproducibility}
\label{sec:supp_reproducibility}

\begingroup
\raggedright
All scripts are located in the \path{applied_study_oncology/} directory of the replication repository accompanying the \texttt{zicbcf} package \citep{zicbcf_package}. Data extraction and cohort preparation are documented in \path{prepare_analysis_data.R}, with domain reconciliation and censoring diagnostics executed by \path{New_Data_v2/run_domain_diagnostics.R}. Shared adjustment sets and subgroup definitions are specified in \path{oncology_common.R}, with the randomized propensity fixed at 0.5. Primary model fits, convergence metrics, and smearing diagnostics are generated by \path{run_zicbcf_oncology.R} and \path{run_gamma_hurdle_oncology.R}, while distinct-day endpoint construction and subgroup analyses are implemented in \path{New_Data/build_nonoverlap_calendar_day_endpoint.R} and \path{New_Data/run_nonoverlap_calendar_day_analysis.R}. Detailed execution scripts and output data objects are archived in the repository.

\subsection{Data Sharing}
\label{sec:supp_onc_sharing}

The individual-participant data analyzed in this section cannot be redistributed, and the reason is worth stating precisely because the usual intuition does not apply. Section III.B(ii) of the Project Data Sphere Platform Agreement prohibits disclosure of the data to any third party who is not a registered Authorized User, and Article VIII requires that they be held in strict confidence \citep{pds_platform_agreement}. This is a contractual confidentiality obligation rather than a patient-privacy one, so the instinct that further de-identification makes a file safe to post does not automatically hold here.

We state the practical position rather than a legal one. The agreement does not expressly define the point at which a derived quantity ceases to be ``the Data'', and Article VII does permit disclosure of what is reasonably necessary to publish or report research results, which is the basis on which the aggregate estimates, tables and figures in this section appear at all. Our operating rule is therefore that per-patient derivatives, including patient-level endpoint files and unit-level posterior draws, are not posted without written confirmation from Project Data Sphere or Amgen. The repository accordingly excludes every patient-level file and every per-patient derivative for this study, while committing all analysis scripts, so that a registered Authorized User can reproduce the study end to end from the supplied package.

Two consequences follow for the authorship and reporting of this work. Access is granted per individual, so every co-author who handled the patient-level data must be separately registered as an Authorized User. And Article VII makes an acknowledgement mandatory in prescribed wording; that sentence appears verbatim in the acknowledgements of the main manuscript and is not paraphrased. Two matters cannot be settled from outside the download record and should be confirmed against it: whether the accepted agreement is the same 2020 version cited here, and whether the Amgen 20050251 package carried dataset-specific Supplemental Terms, which Section III.A(ii) of the agreement permits.
\par
\endgroup

\clearpage

\section{Applied Study: Dental Insurance and Dental Expenditures in the Medical Expenditure Panel Survey}
\label{sec:supp_meps}

This section is the extended, self-contained version of the health-economics applied study referred to in the main manuscript. It develops in full the application to the 2023 Medical Expenditure Panel Survey: the data source and cleaning, the descriptive structure of the semicontinuous dental-expenditure outcome, the estimation pipeline and MCMC configuration, the draw-by-draw construction of subgroup effects, the complete set of subgroup CATE estimates and pairwise contrasts, and a detailed reconciliation of every finding against the applied dental health-economics literature, including the policy implications of the estimated heterogeneity. Because that literature is mature and quantitatively specific, this study functions as an external check on ZIC-BCF-Smear: it asks whether the estimator recovers gradients that decades of applied dental health-economics research have already established.

\subsection{Data Source and Outcome Construction}
\label{sec:supp_applied_data}

\subsubsection{The Medical Expenditure Panel Survey}
The Medical Expenditure Panel Survey (MEPS) is a nationally representative survey of the U.S.\ civilian non-institutionalized population, administered by the Agency for Healthcare Research and Quality, and is a standard data source for studying healthcare utilization and expenditure. The 2023 Full-Year Consolidated File (HC-251) is used here \citep{meps2023file}. Dental expenditures recorded in MEPS are a canonical semicontinuous outcome and have been used repeatedly to study dental-coverage effects \citep{manski2002private, christian2014determinants, luo2020rural, naib2026dental}.

\subsubsection{Variables}
The analysis uses the following variables. The outcome $Y$ is \texttt{DVTEXP23}, total annual dental expenditure in U.S.\ dollars, which is semicontinuous: a large point mass at exactly zero (non-utilizers) plus a heavily right-skewed continuous distribution among users. The treatment $Z$ is \nolinkurl{DNTINS23_M23}, a dental-insurance coverage indicator coded $1$ for covered and $0$ otherwise. The covariates $X$ are age, sex, race and ethnicity, family income, poverty category, census region, and marital status, which are standard determinants of dental utilization and expenditure in the MEPS literature \citep{manski1998demographic, christian2014determinants, zardak2023utilization}. The design variables are the person weight \texttt{PERWT23F}, the variance stratum \texttt{VARSTR}, and the variance primary sampling unit \texttt{VARPSU}. After excluding invalid, negative, and inapplicable survey responses (for example, out-of-scope codes on the treatment and outcome fields), the causal analysis retains $N = 18{,}763$ records.

Race and ethnicity are taken from \texttt{RACETHX}, the standard MEPS analytic variable, whose five levels are Hispanic, non-Hispanic White, non-Hispanic Black, non-Hispanic Asian, and non-Hispanic other or multiple race. We intentionally avoid race-only classifications (such as \texttt{RACEV2X}) because they lack an ethnicity dimension and would misclassify 3, 758 Hispanic respondents as White and 144 as Black in this sample. Utilizing \texttt{RACETHX} ensures that estimated racial and ethnic contrasts conform to established standards in health-services research.

\subsubsection{Encoding of Nominal Covariates}
\label{subsec:supp_meps_encoding}

The design matrix is constructed by expanding all nominal categorical factors into indicator columns using \texttt{model.matrix}. Expanding categorical variables is essential for tree-based ensemble estimators; naive numeric coercion would impose an artificial ordinal ranking on nominal attributes (such as census region or race and ethnicity), forcing decision trees to split along arbitrary administrative code orderings rather than evaluating categorical subsets.

\subsubsection{Descriptive Structure of the Outcome}
The dental-expenditure outcome is exactly the regime that motivates a two-part, distribution-free treatment. In the analytic sample, $54.9\%$ of respondents report zero dental spending, and $54.5\%$ on a weighted basis. Among the positive (active) subpopulation, the distribution is severely right-skewed: the median positive expenditure is $\$361$, the mean is $\$1{,}038$, the 90th and 95th percentiles are $\$2{,}505$ and $\$4{,}511$, and the maximum is $\$40{,}350$. Figure~\ref{fig:supp_zero_pos} shows the zero-versus-positive decomposition, Figure~\ref{fig:supp_loghist} the approximately log-normal shape of the positive part on the log scale, and Figure~\ref{fig:supp_any_by_ins} the raw association between coverage and the probability of any dental spending.

\begin{figure}[htbp]
\centering
\includegraphics[width=0.6\textwidth]{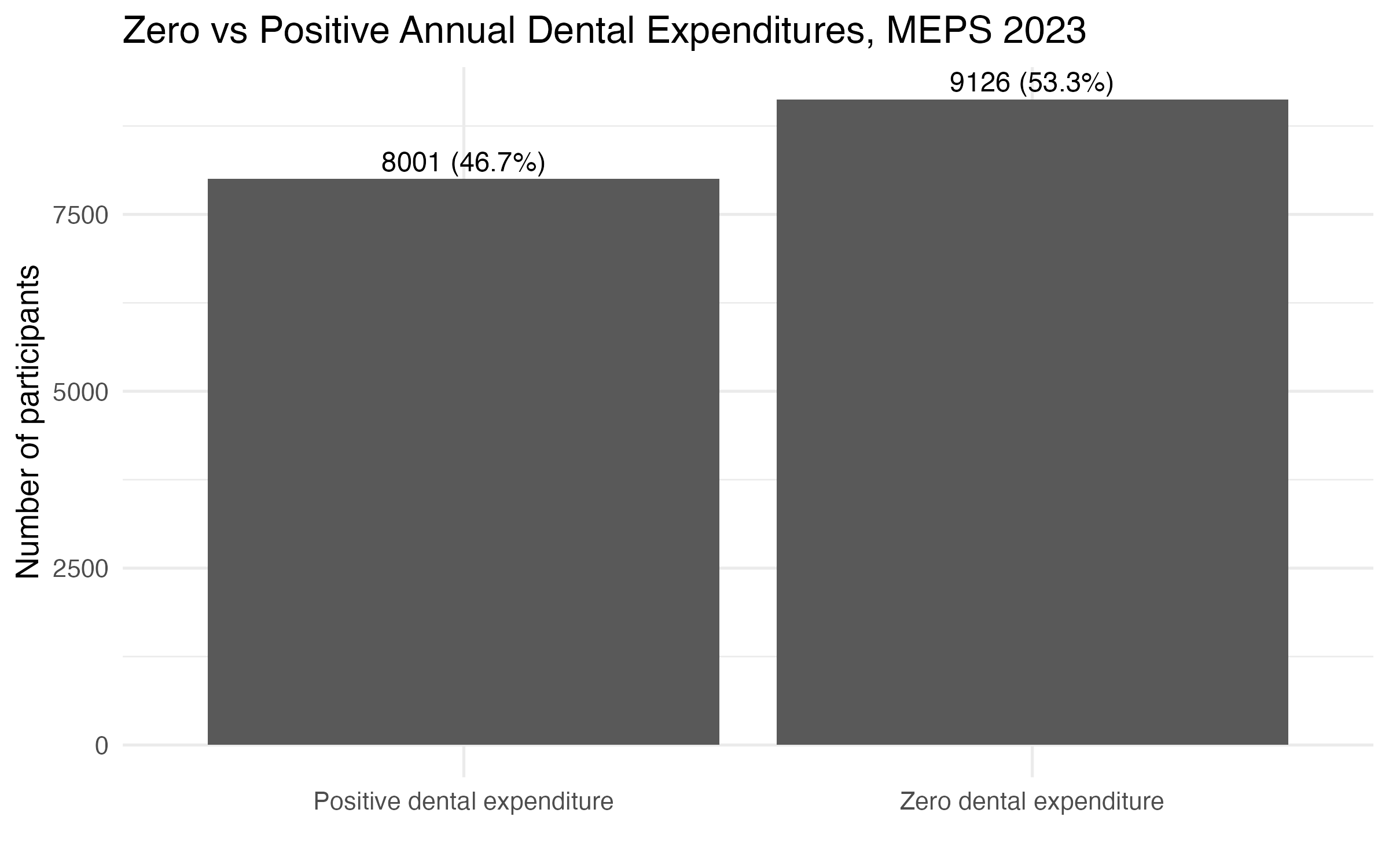}
\caption{Decomposition of the dental-expenditure outcome into the zero point mass ($\approx 53\%$ of records) and the positive continuous part.}
\label{fig:supp_zero_pos}
\end{figure}

\begin{figure}[htbp]
\centering
\includegraphics[width=0.6\textwidth]{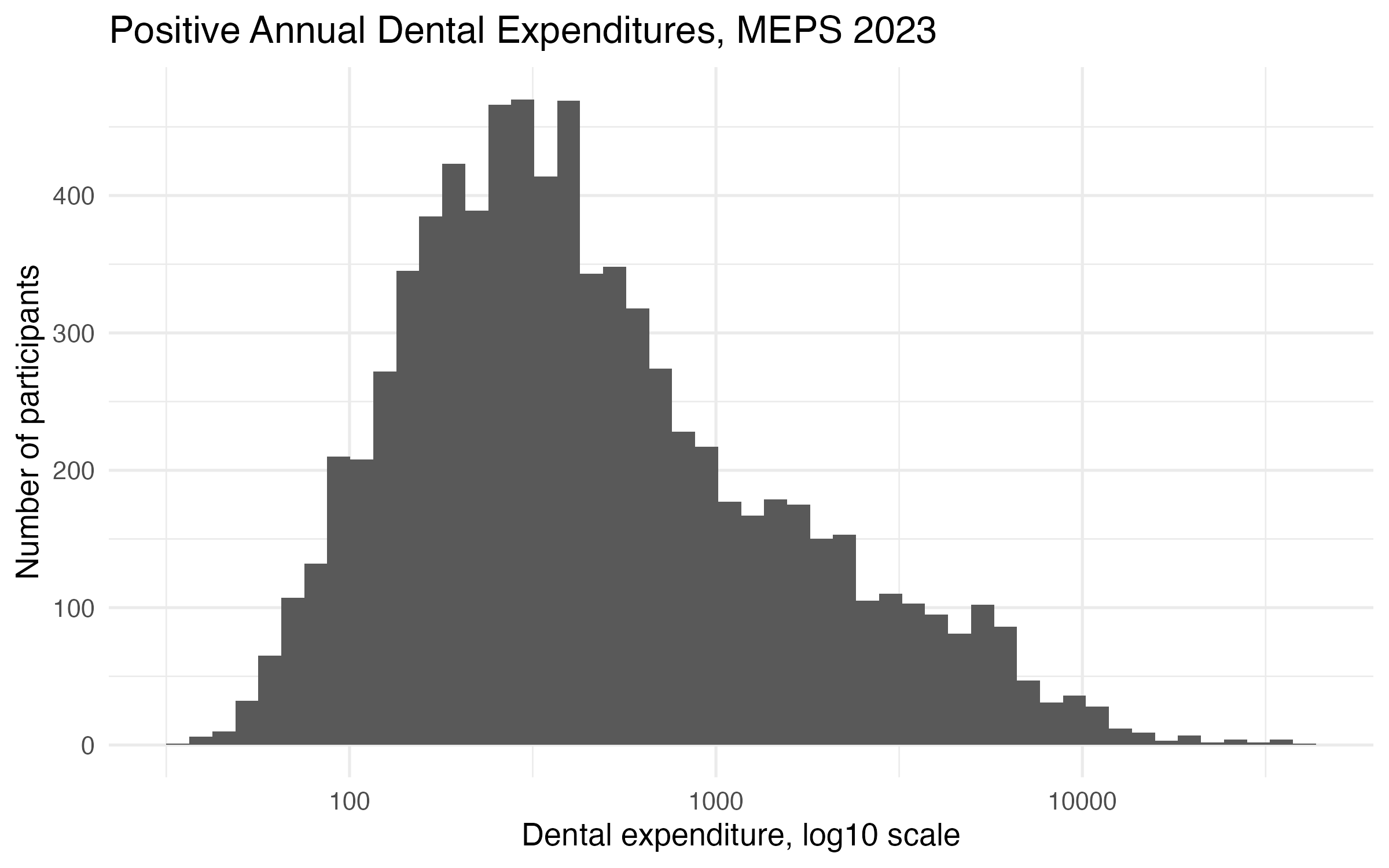}
\caption{Histogram of positive dental expenditures on the log scale, showing the strong right skew on the raw scale and the approximate log-normality that motivates the log-link continuous stage combined with a nonparametric (smearing) re-transformation.}
\label{fig:supp_loghist}
\end{figure}

\begin{figure}[htbp]
\centering
\includegraphics[width=0.6\textwidth]{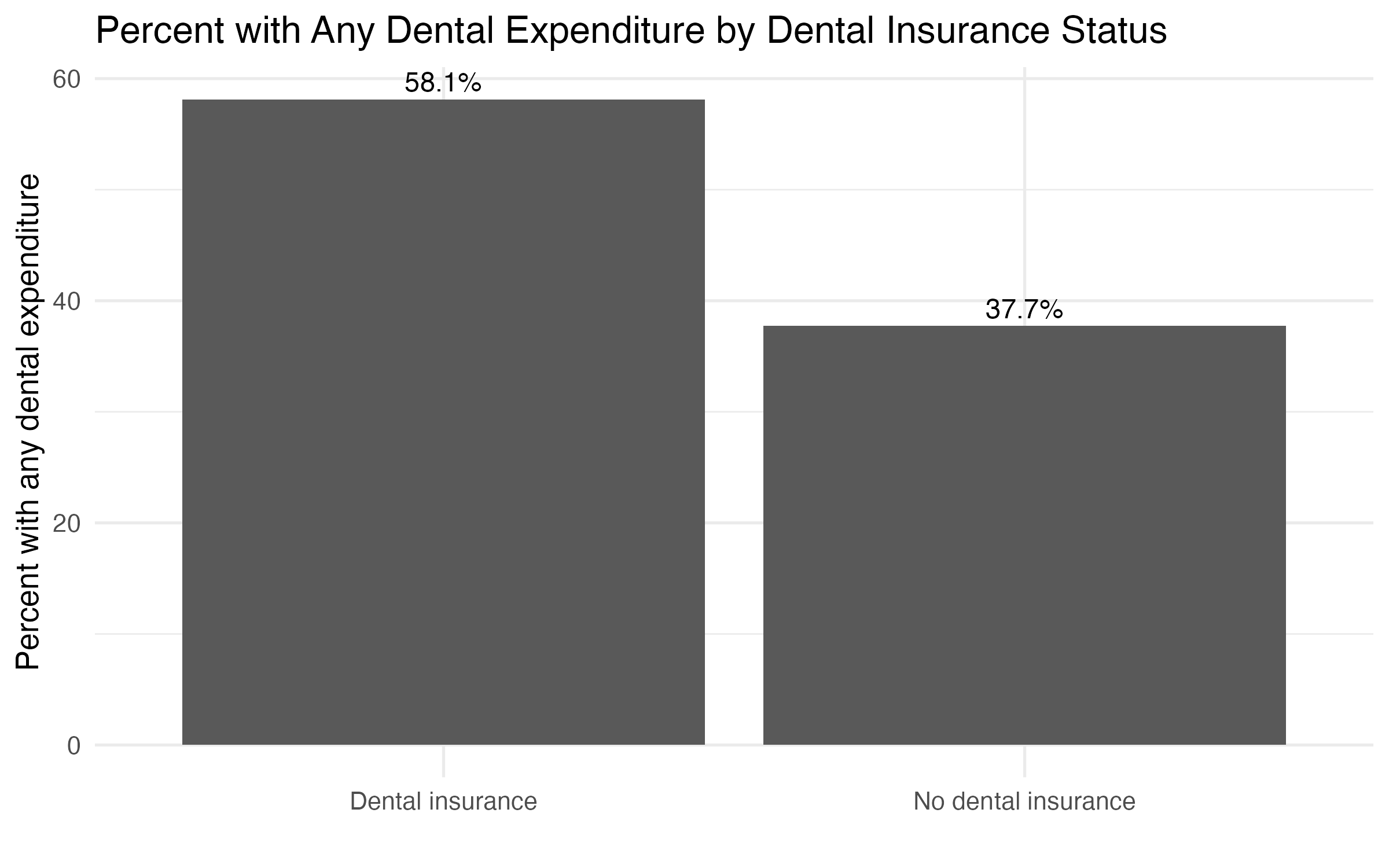}
\caption{Unadjusted probability of any dental expenditure by dental-insurance status. Coverage is strongly associated with clearing the zero hurdle; the causal analysis adjusts this association for confounding and separates it from the intensity effect.}
\label{fig:supp_any_by_ins}
\end{figure}

\subsection{Estimation Pipeline}
\label{sec:supp_applied_est}

\subsubsection{Propensity Models}
Two propensity models are fit. A full-sample logistic regression estimates $\widehat{\pi}(X) = P(Z=1 \mid X)$, which is supplied to the hurdle-stage prognostic forest. A second logistic regression is fit strictly on the active subpopulation $\{i : Y_i > 0\}$ to estimate the subpopulation propensity $\widehat{\pi}^{+}(X) = P(Z=1 \mid X, Y>0)$, which is supplied to the continuous-stage prognostic forest to implement the Subpopulation Propensity Adjustment (SPA). SPA is required because conditioning on the post-treatment event $\{Y>0\}$ opens a collider path; adjusting for $\widehat{\pi}^{+}$ within the active set restores identification of the intensity treatment effect under Assumption~\ref{as:supp_spa}.

\begin{assumption}[Subpopulation strong ignorability]
\label{as:supp_spa}
Conditional on covariates $X$ and active status $\{Y>0\}$, treatment $Z$ is independent of the potential positive-intensity outcomes, $(Y^{+}(0), Y^{+}(1)) \perp Z \mid X, Y>0$. As established in the main manuscript, this holds under full-sample strong ignorability together with conditional two-part independence of the participation and intensity processes.
\end{assumption}

\subsubsection{ZIC-BCF-Smear Configuration}
The \nolinkurl{zicbcf_smear} routine of the \texttt{zicbcf} package \citep{zicbcf_package} is used. The hurdle stage is a Probit BCF fit on all $N$ records (participation), and the continuous stage is a Gaussian BCF fit on $\log(Y^{+})$ over the active subset (intensity), each following the BCF construction of \citet{hahn2020bcf} with BART components \citep{chipman2010bart, hill2011bayesian}. Log-scale predictions are mapped to the dollar scale draw-by-draw with Duan's nonparametric smearing factor \citep{duan1983smearing}, which avoids the parametric multiplier $\exp(\sigma^2/2)$ and its misspecification bias \citep{manning2001estimating}. Four independent chains are run with seeds 1 through 4, each for $2{,}000$ retained draws after $5{,}000$ burn-in iterations, giving $8{,}000$ pooled draws. Chains are run and summarized one at a time and then discarded, so that peak memory stays near the size of a single fit on an $18{,}763$-record sample; what is retained is the per-draw average of unit-level effects within each reported domain, which is sufficient for every subgroup effect and every contrast reported below, together with the unit-level posterior means used by the marginal figures.

\subsubsection{Survey Design and the Two Reported Estimands}
\label{subsec:supp_meps_design}

MEPS is a stratified, clustered, unequal-probability sample of the U.S.\ civilian non-institutionalized population \citep{meps_survey_estimation}. The analytic sample spans 105 variance strata and 262 primary sampling units, every stratum containing at least two, and the person weights sum to $331{,}946{,}662$ persons. An unweighted average over the $18{,}763$ records represents a sample-specific quantity over an oversampled design, where posterior uncertainty reflects superpopulation sampling rather than finite-population survey design.

Both estimands are reported side by side for every effect, because they answer distinct scientific questions and their comparison is methodologically informative.

The \emph{analytic-sample average} is the unweighted average of unit-level effects, $|G|^{-1}\sum_{i \in G}\tau^{(s)}(X_i)$, summarized by its posterior. It describes the sample.

The \emph{survey-weighted population target} is the weighted average $\left(\sum_{i \in G} w_i\right)^{-1}\sum_{i \in G} w_i \tau^{(s)}(X_i)$ with $w_i = \texttt{PERWT23F}$, taken within each posterior draw. Its interval adds a design-based variance component to the posterior variance. The design component is obtained by Taylor linearization \citep{binder1983variances, wolter2007introduction}: for a domain mean the linearized value of unit $i$ is $u_i = w_i d_i (\tau_i - \widehat{\theta}_G) / \sum_j w_j d_j$ with $d_i$ the domain indicator, these are summed within primary sampling unit, and the variance is the usual between-PSU within-stratum sum of squares $\sum_h \frac{n_h}{n_h - 1}\sum_{a \in h}(u_{ha} - \bar{u}_h)^2$. Units outside a domain contribute zero but remain attached to their sampling unit, so domain estimation is handled correctly. For a between-domain contrast the linearized values of the two domains are differenced before aggregation, which retains the covariance induced by their sharing primary sampling units. The reported interval is the posterior mean plus or minus $1.96$ times $\sqrt{V_{\mathrm{posterior}} + V_{\mathrm{design}}}$.

Two properties of this construction warrant explicit mention. First, it treats sampling variability and posterior uncertainty as approximately independent, which represents the standard decomposition when carrying model-based estimates onto complex survey designs \citep{binder1983variances}. Second, because the Bayesian forest algorithm does not incorporate survey weights directly during tree splitting, weighting enters through the linear aggregation of unit-level causal effects. Accordingly, income tertile cutpoints are defined using survey-weighted population quantiles ($\$56{,}133$ and $\$125{,}320$, compared to unweighted sample quantiles of $\$45{,}400$ and $\$110{,}600$) to reflect national economic thresholds accurately.

It is not knowable in advance whether respecting the design widens or narrows an interval, since clustering ordinarily widens and stratification ordinarily narrows. In this application the design component turns out to be small relative to the posterior: for the overall dollar-scale effect the design standard error is $\$1.34$ against a posterior standard deviation of $\$15.73$. That is a finding rather than an assumption, and it did not have to come out that way.

\subsubsection{Draw-by-Draw Subgroup and Contrast Estimation}
All subgroup summaries are computed from the posterior draw matrix, so that uncertainty propagates correctly.
\begin{enumerate}
\item \textbf{Overall ATE.} For each draw $s$, average the unit-level effects over all $N$ units, $\mathrm{ATE}^{(s)} = N^{-1}\sum_i \tau^{(s)}_{\mathrm{resp}}(X_i)$; the posterior mean and $2.5/97.5\%$ quantiles give the point estimate and $95\%$ credible interval.
\item \textbf{Subgroup CATE.} For a subgroup $G$, average within each draw over its members only, $\mathrm{CATE}^{(s)}_G = |G|^{-1}\sum_{i \in G}\tau^{(s)}_{\mathrm{resp}}(X_i)$, giving a full posterior for the subgroup effect.
\item \textbf{Contrasts.} For two subgroups $G_1, G_2$, form the per-draw difference $\Delta^{(s)} = \mathrm{CATE}^{(s)}_{G_1} - \mathrm{CATE}^{(s)}_{G_2}$; the posterior of $\Delta$ yields a credible interval and the posterior probability $P(\Delta > 0)$.
\end{enumerate}
This is the same logic used for the overall ATE, restricted to the units defining each subgroup, and it is what distinguishes a proper subgroup CATE analysis from inspecting the marginal spread of unit-level point estimates.

\subsection{Results}
\label{sec:supp_applied_results}

\subsubsection{Average Treatment Effect}
The estimated ATE of dental insurance on annual dental expenditure is $\$188.02$ as an analytic-sample average, with a $95\%$ credible interval of $[\$156.95,\ \$220.67]$, and $\$184.67$ as a survey-weighted population target, with a design-aware $95\%$ interval of $[\$154.63,\ \$214.70]$. The posterior probability of a positive effect is $1.00$ under both. Figure~\ref{fig:supp_ate} shows the ATE posterior and Figure~\ref{fig:supp_unit_hist} the marginal distribution of the $18{,}763$ unit-level posterior-mean CATEs.

Convergence diagnostics follow Section~\ref{subsec:supp_convergence}. Across the seven monitored functionals the largest $\widehat{R}$ is $1.007$, the smallest effective sample size is $1{,}313$ out of $8{,}000$ pooled draws, and the largest absolute Geweke statistic is $2.80$. The multiplicity context of Section~\ref{subsec:supp_convergence} applies here as well: seven functionals across four chains give 28 standard normal statistics, whose 5\% family-wise critical value is about 3.12, and the observed maximum of 2.80 sits below it, so no correction-flagged value remains. The maximum moved rather than merely shrank: under the longer burn-in $\sigma_c$ falls from $3.672$ to $0.556$ while the hurdle effect rises from $0.899$ to $2.804$, so the surviving maximum is a different functional rather than a uniformly quieter posterior. The MEPS fits share the sampler core, and therefore the $\sigma_c$ equilibration behavior, with the SPECTRUM fits; the final models accordingly use the same 5,000-iteration burn-in motivated there, which the investigation reported in that section showed removes the early-window transient that produced the originally flagged value. Because $\sigma_c$ sets the smearing factor that rescales every dollar-scale summary, the locally smeared sensitivity analysis below remains the check that bounds what residual-scale behavior can do to the reported estimate. The Monte Carlo standard error of the dollar-scale ATE is $\$0.40$, two orders of magnitude below its posterior standard deviation, so the reported summaries are not limited by sampler error.

The homoskedasticity assumption behind the smearing factor is tested as described in Section~\ref{subsec:supp_smearing_test}, and here it fails decisively. The studentized Breusch--Pagan statistic is $160.3$ on 20 degrees of freedom, $p = 6.4 \times 10^{-24}$, and the draw-wise version rejects at the $1\%$ level in all 200 draws examined. The residual standard deviation rises monotonically across the five strata of the fitted log-scale mean, from $1.005$ in the lowest to $1.192$ in the highest, which is the expected pattern for expenditure data. What matters for the reported estimate, however, is the locally smeared sensitivity analysis, and it moves the pooled-average effect only from $\$188.02$ to $\$193.76$, a change of $3.0\%$ against a posterior standard deviation of $\$16.2$. The conclusion is therefore that the assumption is violated and that the violation does not materially drive the estimate, which are two separate findings and both worth reporting.

\paragraph{Locally smeared subgroup sensitivity.}
Because the smearing factor rises by roughly $30\%$ between the extreme strata of the fitted mean, the quantity most exposed to that heterogeneity is not the cohort average, over which stratum-specific rescalings partly cancel, but the subgroup ordering on which any disparity statement rests. The locally smeared sensitivity analysis was therefore repeated for every reported subgroup: within each posterior draw, the single scalar factor is replaced by the smearing factor of the stratum each unit's fitted log-scale mean falls into, the unit-level response-scale CATEs are recombined on that basis, and each subgroup posterior is recomputed draw-by-draw exactly as in Table~\ref{tab:supp_subgroups}. Both the locally smeared estimates and their scalar-smeared counterparts below come from the same pooled four-chain run, so the comparison isolates the re-transformation rather than Monte Carlo noise. Table~\ref{tab:supp_local_smear} reports the result, ordered by sensitivity.

\begingroup
\footnotesize
\setlength{\tabcolsep}{3pt}
\begin{longtable}{p{2.45in}rrrrr}
\caption{Locally smeared subgroup sensitivity for the dollar-scale analytic-sample average: scalar Duan versus locally smeared subgroup CATEs, ordered by relative change. Both columns derive from the same pooled posterior draws; every locally smeared point estimate lies inside its scalar-smear 95\% credible interval.}\label{tab:supp_local_smear}\\
\toprule
Domain & $n$ & Scalar CATE (\$) & Locally smeared & Local 95\% CrI & Change \\
\midrule
\endfirsthead
\toprule
Domain & $n$ & Scalar CATE (\$) & Locally smeared & Local 95\% CrI & Change \\
\midrule
\endhead
Age group : 35-49 & 3468 & 175.27 & 183.95 & [150.47, 218.70] & 5.0\% \\
Age group : 18-34 & 3352 & 169.06 & 176.81 & [142.45, 214.46] & 4.6\% \\
Family income (weighted tertile) : High-income tertile & 5330 & 188.25 & 196.54 & [149.30, 242.36] & 4.4\% \\
Marital status (detailed) : Never married & 4579 & 191.74 & 199.91 & [160.15, 246.77] & 4.3\% \\
Region : Midwest & 3837 & 195.04 & 202.46 & [153.53, 259.09] & 3.8\% \\
Poverty cat. : High-income & 7351 & 206.16 & 213.89 & [164.10, 260.06] & 3.7\% \\
Race and ethnicity : Asian, non-Hispanic & 1143 & 208.13 & 215.66 & [156.84, 287.64] & 3.6\% \\
Marital status : Married & 7497 & 199.66 & 206.47 & [163.18, 248.00] & 3.4\% \\
Marital status (detailed) : Married & 7497 & 199.66 & 206.47 & [163.18, 248.00] & 3.4\% \\
Race and ethnicity : Other or multiple, non-Hispanic & 717 & 149.27 & 154.27 & [68.09, 229.77] & 3.3\% \\
Race and ethnicity : Black, non-Hispanic & 2495 & 138.37 & 142.82 & [86.76, 199.90] & 3.2\% \\
Age group : 50-64 & 3760 & 202.35 & 208.84 & [167.97, 247.03] & 3.2\% \\
Region : Northeast & 2892 & 204.39 & 210.87 & [164.04, 257.12] & 3.2\% \\
Race and ethnicity : White, non-Hispanic & 10243 & 218.16 & 225.03 & [181.96, 267.59] & 3.2\% \\
Poverty cat. : Low-income & 2458 & 126.79 & 130.77 & [72.21, 205.38] & 3.1\% \\
Sex : Female & 9835 & 189.99 & 195.93 & [153.55, 238.20] & 3.1\% \\
Overall : All & 18763 & 188.02 & 193.76 & [160.26, 227.56] & 3.0\% \\
Sex : Male & 8928 & 185.86 & 191.37 & [156.99, 227.75] & 3.0\% \\
Family income (weighted tertile) : Middle-income tertile & 5922 & 207.89 & 214.00 & [178.60, 249.27] & 2.9\% \\
Poverty cat. : Middle-income & 5243 & 212.02 & 218.09 & [175.39, 262.89] & 2.9\% \\
Region : South & 7375 & 161.28 & 165.88 & [130.87, 203.88] & 2.9\% \\
Age group : 0-17 & 3647 & 127.04 & 130.61 & [84.92, 184.24] & 2.8\% \\
Marital status : Not married & 11266 & 180.28 & 185.30 & [151.60, 221.13] & 2.8\% \\
Region : West & 4659 & 214.42 & 220.11 & [170.85, 270.78] & 2.7\% \\
Race and ethnicity : Hispanic & 4165 & 144.82 & 148.15 & [115.04, 184.03] & 2.3\% \\
Family income (weighted tertile) : Low-income tertile & 7511 & 172.20 & 175.83 & [139.34, 217.24] & 2.1\% \\
Marital status (detailed) : Divorced & 1933 & 226.30 & 230.70 & [147.61, 297.80] & 1.9\% \\
Marital status (detailed) : Under 16 & 3151 & 113.70 & 115.67 & [67.98, 172.64] & 1.7\% \\
Poverty cat. : Poor & 2834 & 155.61 & 158.24 & [115.11, 202.15] & 1.7\% \\
Marital status (detailed) : Widowed & 1261 & 244.47 & 247.56 & [136.59, 353.29] & 1.3\% \\
Age group : 65+ & 4536 & 248.95 & 252.05 & [181.06, 315.55] & 1.2\% \\
Poverty cat. : Near-poor & 877 & 168.89 & 170.96 & [104.13, 257.55] & 1.2\% \\
Marital status (detailed) : Separated & 342 & 143.67 & 145.08 & [79.51, 234.29] & 1.0\% \\

\bottomrule
\end{longtable}
\endgroup

Across all 33 reported levels (the overall sample plus the 32 subgroup levels), and on both the unweighted and the design-weighted footing, every locally smeared estimate exceeds its scalar counterpart, by $1.0$ to $5.0\%$ on the analytic sample and $1.1$ to $5.2\%$ on the population basis, with no sign change and a rank correlation of $0.999$ between the two orderings. No covariate family changes its internal ordering, and the largest gap between levels widens rather than narrows, from $\$135.3$ to $\$136.4$ sample-side and from $\$137.1$ to $\$137.8$ population-side. The subgroup structure that carries the racial and ethnic gradient discussed above is therefore preserved under the more conservative re-transformation: non-Hispanic White remains the largest racial and ethnic effect at $\$218.2$ (locally smeared $\$225.0$, a $3.2\%$ shift), and every locally smeared point estimate falls inside the scalar-smear credible interval it replaces. The subgroup conclusion is, if anything, marginally more exposed to residual-scale misscaling than the average effect ($5.0\%$ against $3.0\%$), but the exposure is bounded and small relative to the width of every interval involved.

\paragraph{Comparison with the unadjusted association.}
The unadjusted difference in mean dental expenditure between insured and uninsured respondents is $\$185.9$, or $\$190.6$ weighted. The adjusted dollar-scale effect ($\$184.67$ population-weighted) is essentially equal to the raw association rather than attenuated. By comparison, a linear regression of expenditure on coverage and the same covariates attenuates the raw gap to $\$117.0$, and the Gamma hurdle benchmark reports $\$147.9$, demonstrating that preserving the full expenditure gap is specific to the flexible semicontinuous model rather than an artifact of simpler parametric adjustment.

The attenuation that the selection story predicts does appear, but on the participation margin rather than on the dollar scale. The unadjusted difference in the probability of any dental spending is $0.201$, and the adjusted participation-margin effect is $0.139$ with a design-aware interval of $[0.124, 0.155]$, so roughly three tenths of the raw participation association is absorbed by the covariates. On the dollar scale almost none of it is. The natural reading is that the available covariates capture the demographic and economic correlates of coverage but not the latent oral-health demand that drives selection into it, so the dollar-scale estimate should be regarded as more exposed to residual confounding than its interval suggests. This is a limitation of the adjustment set rather than a property of the estimator, and it is carried forward into Section~\ref{sec:supp_applied_limits}.

\begin{figure}[htbp]
\centering
\includegraphics[width=0.6\textwidth]{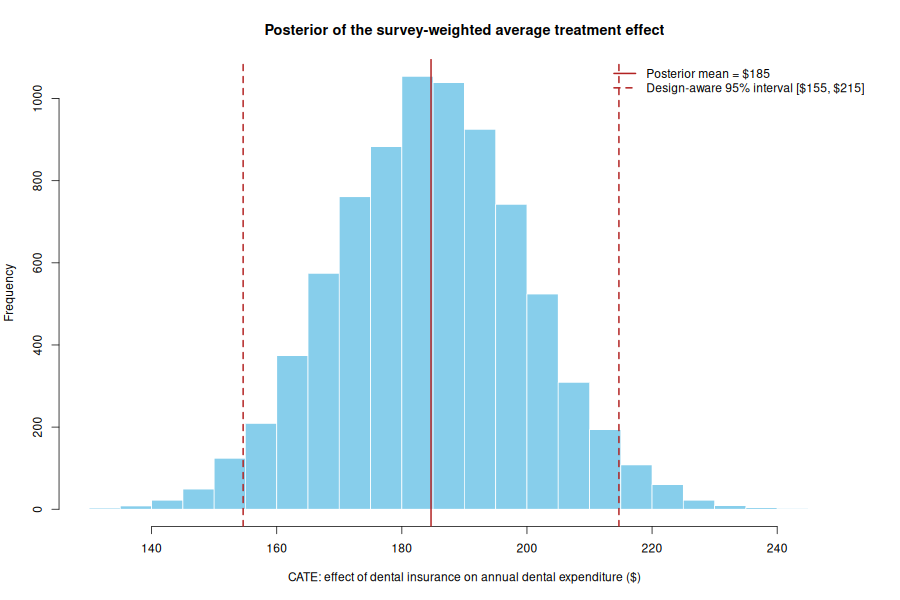}
\caption{Posterior distribution of the survey-weighted ATE, with the design-aware $95\%$ interval marked. The posterior concentrates near $\$185$ and lies entirely above zero.}
\label{fig:supp_ate}
\end{figure}

\begin{figure}[htbp]
\centering
\includegraphics[width=0.6\textwidth]{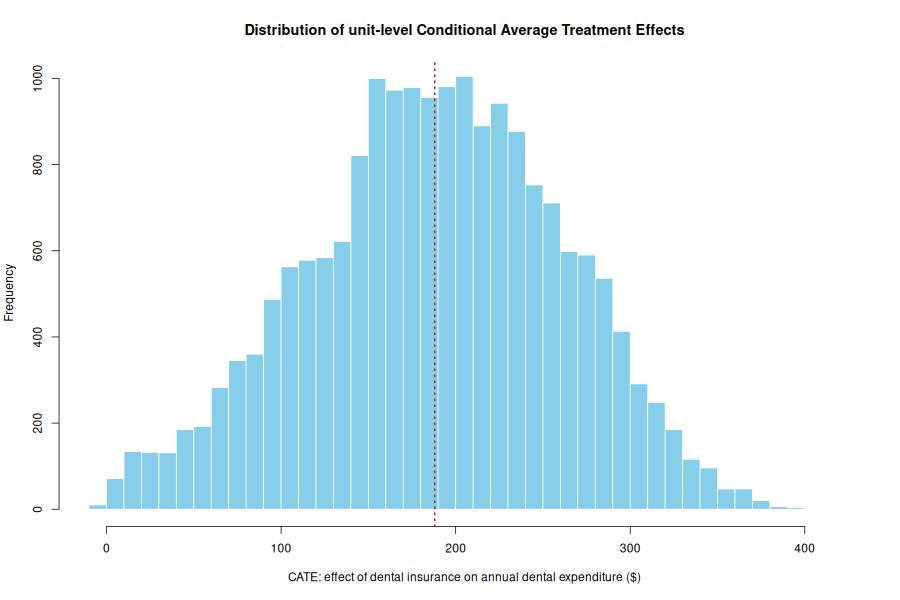}
\caption{Marginal distribution of unit-level posterior-mean CATEs across all $18{,}763$ individuals. This marginal view is shown for reference only; all inferential subgroup summaries use the draw-by-draw construction of Section~\ref{sec:supp_applied_est}.}
\label{fig:supp_unit_hist}
\end{figure}

\paragraph{The ZIC-BCF-Smear Hurdle-Margin ATE}
\label{subsubsec:supp_hurdle_margin}
Although its primary estimand is the overall dollar-scale CATE, ZIC-BCF-Smear is a two-part model and therefore also delivers a participation (hurdle-margin) treatment effect from its probit hurdle stage. Let $\mu_b(x)$ and $\tau_b(x)$ denote the probit-scale prognostic and treatment-effect draws returned by the hurdle stage (\nolinkurl{fit$mu_b} and \nolinkurl{fit$tau_b}). The covariate-specific participation probabilities under control and treatment are $p_0(x) = \Phi(\mu_b(x))$ and $p_1(x) = \Phi(\mu_b(x) + \tau_b(x))$, and the hurdle-margin CATE is $p_1(x) - p_0(x)$, formed draw-by-draw and then summarized exactly as the dollar-scale CATE. The overall hurdle-margin ATE is $0.139$ as a sample average and $0.139$ as a population target, with a design-aware $95\%$ interval of $[0.124,\ 0.155]$: dental insurance raises the probability of any dental expenditure by about $14$ percentage points. Reporting this margin makes explicit that the participation channel is a native output of ZIC-BCF-Smear and not only of the Gamma hurdle benchmark.

\paragraph{Comparison with a Gamma Hurdle Benchmark}
As a benchmark, the Bayesian Gamma hurdle specification of \citet{oganisian2021zeroinflated} was fit on the same analytic sample, with the same design matrix, the same four chains, and the same weighting and variance treatment. This model recombines the participation and positive-intensity pieces draw-by-draw. The resulting response-scale ATE is $\$147.87$ as a sample average, with a $95\%$ credible interval from $\$90.06$ to $\$208.44$, and $\$143.20$ as a population target with a design-aware interval from $\$88.95$ to $\$197.45$. Its hurdle-margin ATE is $0.118$ (design-aware $95\%$ interval $[0.101,\ 0.134]$), so it estimates that dental insurance raises the probability of any dental expenditure by about $11.8$ percentage points. Its convergence is unproblematic, with $\widehat{R}$ at most $1.001$ and effective sample sizes above $3{,}900$. These ATE-only benchmark results show that the positive average effect is not an artifact of ZIC-BCF-Smear: the Gamma hurdle model also finds a credibly positive response-scale effect. The ZIC-BCF-Smear ATE is nevertheless larger and much more tightly concentrated, with a posterior standard deviation of $\$15.7$ against $\$27.6$, consistent with the simulation finding that flexible smearing-based re-transformations are less fragile under heavy right skew and positive-intensity misspecification than a parametric Gamma positive-intensity likelihood.

\subsubsection{Full Subgroup CATE Table}
Table~\ref{tab:supp_subgroups} reports every subgroup CATE with both estimands and their intervals. All subgroups have posterior probability of a positive effect equal to $1.00$. That uniformity is a statement about location rather than a sign of understated uncertainty: the overall effect of about $\$185$ is roughly twelve posterior standard deviations from zero, and no subgroup estimate falls below $\$124$, so a posterior mass below zero would be surprising. The informative question is therefore not whether each subgroup effect is positive but whether subgroups differ, and that question is answered by the contrasts below, of which 37 of 56 are not credible. The convergence diagnostics reported above establish separately that these probabilities are not artifacts of a short or poorly mixed chain. Figure~\ref{fig:supp_subgroups} plots the same estimates against the overall ATE.

\begingroup
\footnotesize
\setlength{\tabcolsep}{3pt}
\begin{longtable}{p{0.95in}p{1.05in}rrr}
\caption{Subgroup CATE estimates (dollars), MEPS 2023. The sample column is the unweighted analytic-sample average with its $95\%$ credible interval; the population column is the survey-weighted target with its design-aware $95\%$ interval.}\label{tab:supp_subgroups}\\
\toprule
Covariate & Subgroup & $N$ & Analytic-sample average & Population target \\
\midrule
\endfirsthead
\toprule
Covariate & Subgroup & $N$ & Analytic-sample average & Population target \\
\midrule
\endhead
Overall & All & 18{,}763 & 188.02 [156.95, 220.67] & 184.67 [154.63, 214.70] \\
Sex & Male & 8{,}928 & 185.86 [153.06, 220.98] & 182.56 [149.75, 215.38] \\
Sex & Female & 9{,}835 & 189.99 [149.59, 230.49] & 186.70 [148.40, 225.00] \\
Race and ethnicity & Hispanic & 4{,}165 & 144.82 [113.97, 178.04] & 146.57 [114.08, 179.06] \\
Race and ethnicity & White, non-Hispanic & 10{,}243 & 218.16 [176.61, 260.39] & 208.95 [170.75, 247.15] \\
Race and ethnicity & Black, non-Hispanic & 2{,}495 & 138.37 [84.47, 192.74] & 137.53 [84.85, 190.21] \\
Race and ethnicity & Asian, non-Hispanic & 1{,}143 & 208.13 [152.57, 275.42] & 203.77 [143.71, 263.83] \\
Race and ethnicity & Other or multiple, non-Hispanic & 717 & 149.27 [64.24, 223.36] & 141.60 [64.16, 219.04] \\
Age group & 0--17 & 3{,}647 & 127.04 [83.87, 176.92] & 126.96 [80.03, 173.89] \\
Age group & 18--34 & 3{,}352 & 169.06 [139.78, 201.82] & 174.18 [141.62, 206.74] \\
Age group & 35--49 & 3{,}468 & 175.27 [147.15, 205.33] & 178.22 [148.39, 208.05] \\
Age group & 50--64 & 3{,}760 & 202.35 [164.11, 239.59] & 207.68 [168.99, 246.38] \\
Age group & 65+ & 4{,}536 & 248.95 [176.24, 315.74] & 250.26 [181.53, 318.99] \\
Poverty category & Poor & 2{,}834 & 155.61 [114.17, 198.17] & 153.07 [113.13, 193.00] \\
Poverty category & Near-poor & 877 & 168.89 [103.54, 252.18] & 165.63 [96.64, 234.62] \\
Poverty category & Low-income & 2{,}458 & 126.79 [69.09, 198.41] & 124.00 [62.16, 185.85] \\
Poverty category & Middle-income & 5{,}243 & 212.02 [171.07, 255.50] & 203.08 [163.92, 242.25] \\
Poverty category & High-income & 7{,}351 & 206.16 [159.17, 250.69] & 198.62 [156.75, 240.49] \\
Family income (weighted tertile) & Low tertile & 7{,}511 & 172.20 [136.59, 212.72] & 168.99 [133.05, 204.93] \\
Family income (weighted tertile) & Middle tertile & 5{,}922 & 207.89 [173.71, 242.16] & 201.35 [168.87, 233.83] \\
Family income (weighted tertile) & High tertile & 5{,}330 & 188.25 [144.99, 230.87] & 183.65 [143.18, 224.13] \\
Region & Northeast & 2{,}892 & 204.39 [159.66, 249.48] & 196.76 [155.52, 237.99] \\
Region & Midwest & 3{,}837 & 195.04 [149.37, 247.82] & 187.86 [142.05, 233.68] \\
Region & South & 7{,}375 & 161.28 [128.68, 196.40] & 164.84 [131.46, 198.22] \\
Region & West & 4{,}659 & 214.42 [166.67, 265.17] & 205.96 [160.39, 251.53] \\
Marital status & Married & 7{,}497 & 199.66 [159.35, 239.49] & 197.32 [159.09, 235.55] \\
Marital status & Not married & 11{,}266 & 180.28 [147.94, 214.96] & 176.26 [144.20, 208.31] \\
\bottomrule
\end{longtable}
\endgroup

Marital status is also reported at its full resolution, since the binary grouping pools respondents under 16 with never-married adults. On the detailed grouping the married, widowed, divorced and never-married levels all exceed the under-16 level credibly, at $\$84.20$, $\$136.32$, $\$112.96$ and $\$80.43$ respectively, which is the age gradient reappearing through a variable whose inapplicable category is defined by age. No contrast among the adult marital levels is credible.

\begin{figure}[htbp]
\centering
\includegraphics[width=0.85\textwidth]{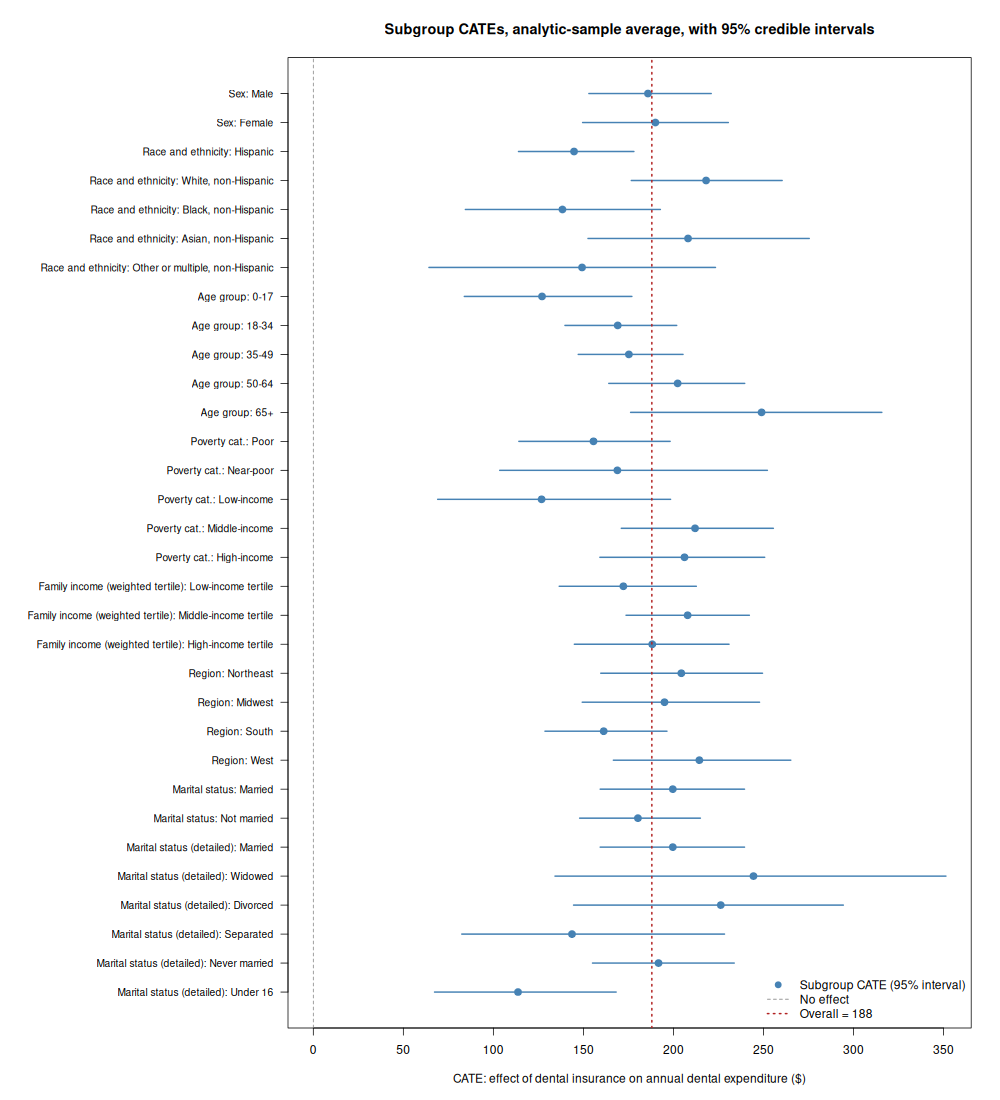}
\caption{Subgroup CATEs with $95\%$ credible intervals on the analytic-sample average. The dashed line marks no effect and the dotted line marks the overall ATE.}
\label{fig:supp_subgroups}
\end{figure}

\subsubsection{Between-Subgroup Contrasts}
A difference in point estimates is meaningful only if it survives posterior uncertainty, and on a complex survey it should also survive the sampling design. Because subgroup effects are computed draw-by-draw, the posterior of any contrast is available directly, and the design-based variance of a contrast is obtained from the differenced linearized values of Section~\ref{subsec:supp_meps_design}. Of the 56 pairwise contrasts, 16 exclude zero on the analytic-sample average and 18 on the design-aware population target. The contrasts that are credible under the design-aware criterion are the following, with population-target estimates and intervals.

On race and ethnicity, non-Hispanic White beneficiaries gain $\$71.42$ $[\$11.96,\allowbreak \$130.87]$ more than non-Hispanic Black beneficiaries, and Hispanic beneficiaries gain $\$62.38$ $[-\$100.02,\allowbreak -\$24.74]$ less than non-Hispanic White beneficiaries and $\$57.20$ $[-\$110.33,\allowbreak -\$4.06]$ less than non-Hispanic Asian beneficiaries. On age, the effect rises monotonically and seven of the ten pairwise contrasts are credible, the largest being children versus the 65-and-over group at $-\$123.30$ $[-\$204.88,\allowbreak -\$41.73]$. On income and poverty, the low-income poverty category gains credibly less than both the middle-income category, at $-\$79.08$ $[-\$145.02,\allowbreak -\$13.14]$, and the high-income category, at $-\$74.62$ $[-\$146.85,\allowbreak -\$2.38]$, and on the weighted income tertiles the low tertile gains less than the middle tertile at $-\$32.36$ $[-\$59.53,\allowbreak -\$5.19]$. On the detailed marital grouping, the four adult levels each exceed the under-16 level as noted above.

The distinction between unweighted sample averages and design-aware population inference is consequential for several contrasts. For example, the South-versus-West regional contrast excludes zero when evaluated as an unweighted sample average ($-\$53.14$ $[-\$102.69,\allowbreak -\$0.51]$), but covers zero on the design-aware population target ($-\$41.12$ $[-\$87.62,\allowbreak \$5.38]$). Because regional conclusions pertain to the national population, the design-aware target is authoritative, indicating that this geographic gradient is suggestive rather than conclusive. Similarly, the White-versus-Asian and White-versus-other-race contrasts do not exclude zero under design-aware inference, reflecting design variance adjustments alongside the proper classification of Hispanic ethnicity.

Sex shows no credible heterogeneity on either estimand (male minus female $= -\$4.14$ $[-\$42.76,\allowbreak \$34.49]$ on the population target), and neither does the binary married-versus-not-married contrast on the dollar scale. Figure~\ref{fig:supp_contrasts} displays all pairwise contrasts with the credible ones highlighted, Figure~\ref{fig:supp_densities} shows the posterior subgroup densities for race and ethnicity and for age, and Figure~\ref{fig:supp_vs_cov} plots unit-level posterior-mean CATEs against the two continuous covariates.

\begin{figure}[htbp]
\centering
\includegraphics[width=0.85\textwidth]{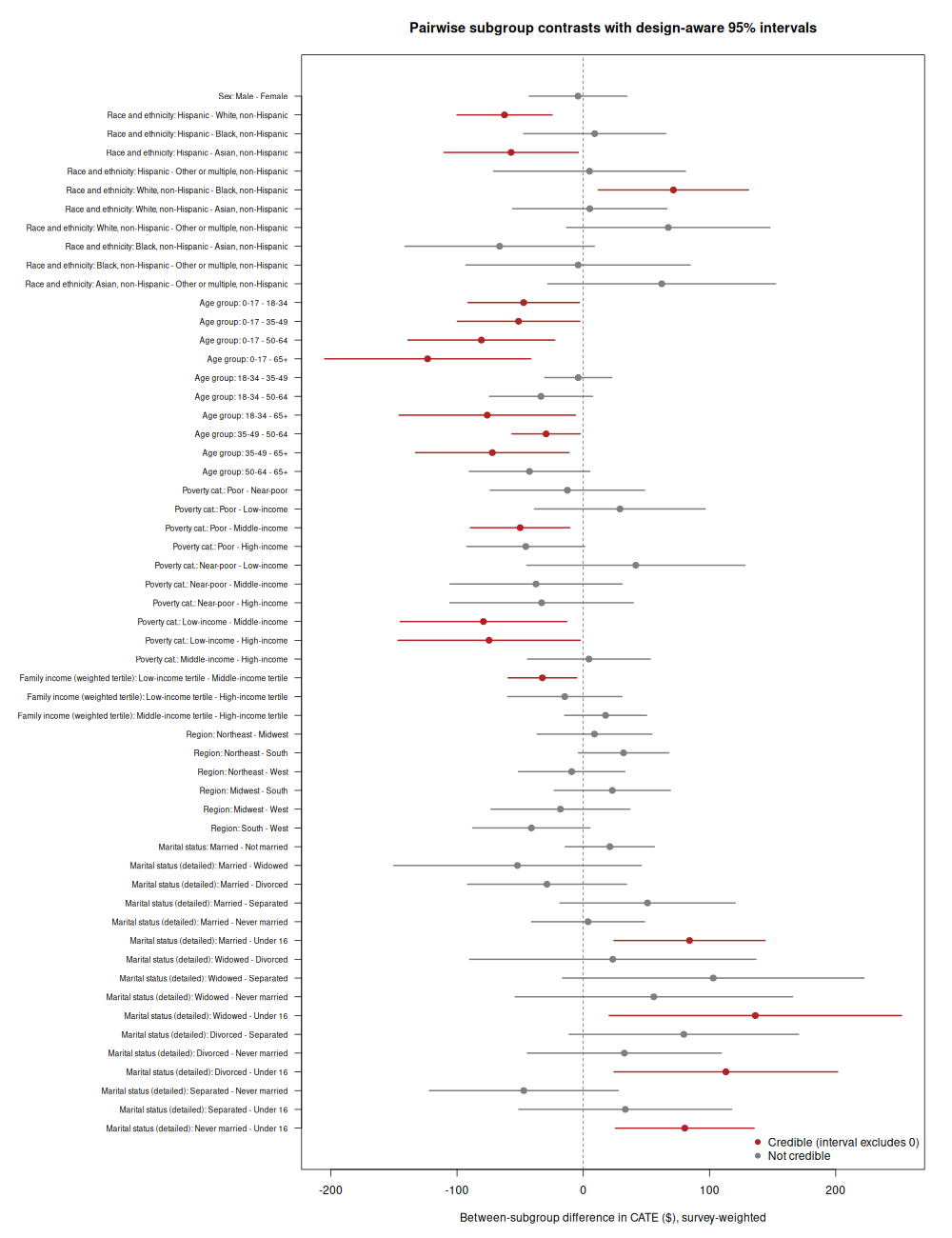}
\caption{All pairwise between-subgroup CATE contrasts on the survey-weighted population target, with design-aware $95\%$ intervals; contrasts whose interval excludes zero are highlighted.}
\label{fig:supp_contrasts}
\end{figure}

\begin{figure}[htbp]
\centering
\includegraphics[width=0.85\textwidth]{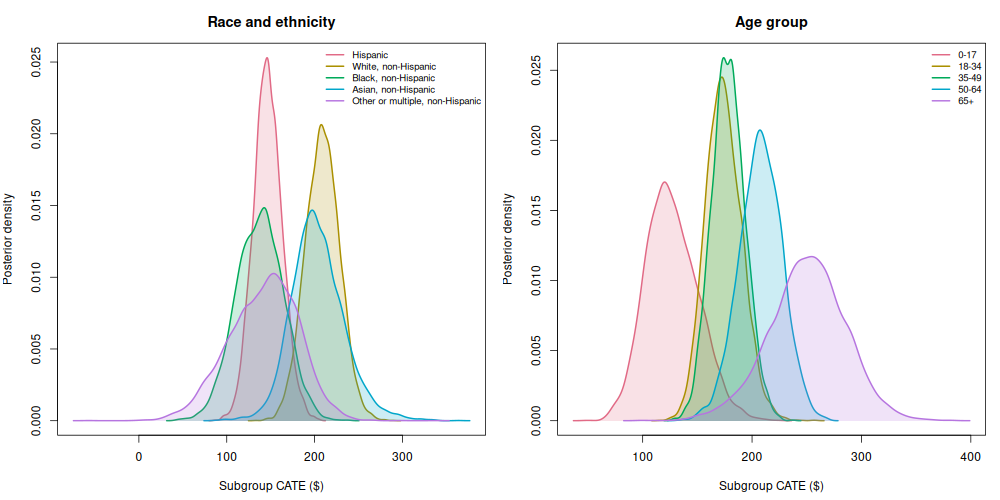}
\caption{Posterior densities of the survey-weighted subgroup CATEs for race and ethnicity and for age. The heterogeneity is visible as a separation of whole posterior distributions, not merely of point estimates.}
\label{fig:supp_densities}
\end{figure}

\begin{figure}[htbp]
\centering
\includegraphics[width=0.85\textwidth]{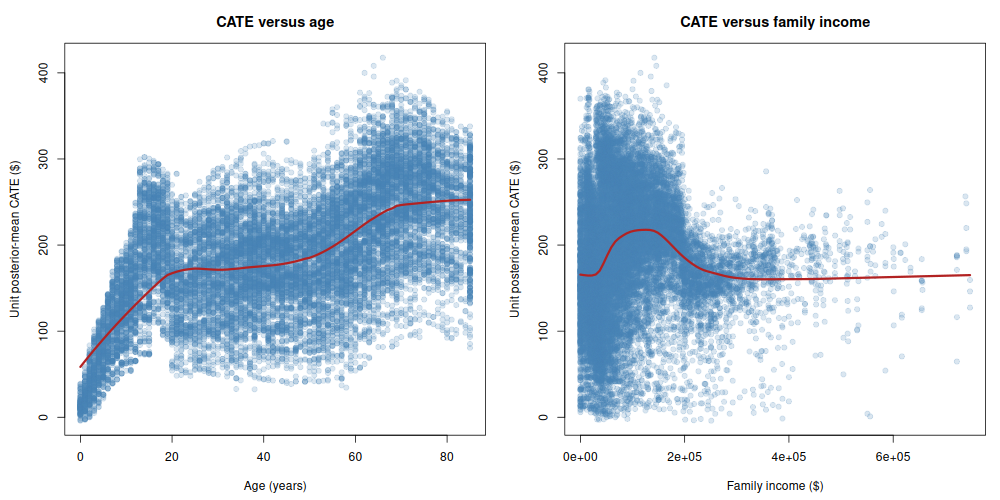}
\caption{Unit-level posterior-mean CATE against age and family income, with a LOWESS trend. The insurance effect rises steadily with age and increases with income before flattening at the top of the distribution, consistent with the subgroup tables.}
\label{fig:supp_vs_cov}
\end{figure}

\subsubsection{Hurdle-Margin (Participation) Subgroup Effects}
\label{subsec:supp_hurdle_subgroups}
The hurdle-margin CATE of Section~\ref{subsubsec:supp_hurdle_margin} is summarized by subgroup exactly as the dollar-scale CATE, averaging the unit-level participation effects $p_1(X_i) - p_0(X_i)$ within each posterior draw. Table~\ref{tab:supp_hurdle_subgroups} reports every subgroup hurdle-margin CATE on the survey-weighted population target, in probability points, with its design-aware $95\%$ interval; all subgroups have posterior probability of a positive participation effect equal to $1.00$.

\begingroup
\small
\begin{longtable}{llrr}
\caption{ZIC-BCF-Smear hurdle-margin (participation) subgroup CATEs, in probability points, on the survey-weighted population target with design-aware $95\%$ intervals, MEPS 2023. Entries are the increase in the probability of any dental expenditure attributable to dental insurance.}\label{tab:supp_hurdle_subgroups}\\
\toprule
Covariate & Subgroup & $N$ & Population target \\
\midrule
\endfirsthead
\toprule
Covariate & Subgroup & $N$ & Population target \\
\midrule
\endhead
Overall & All & 18{,}763 & 0.139 [0.123, 0.155] \\
Sex & Male & 8{,}928 & 0.143 [0.121, 0.165] \\
Sex & Female & 9{,}835 & 0.136 [0.113, 0.158] \\
Race and ethnicity & Hispanic & 4{,}165 & 0.129 [0.103, 0.155] \\
Race and ethnicity & White, non-Hispanic & 10{,}243 & 0.151 [0.130, 0.171] \\
Race and ethnicity & Black, non-Hispanic & 2{,}495 & 0.109 [0.063, 0.154] \\
Race and ethnicity & Asian, non-Hispanic & 1{,}143 & 0.152 [0.114, 0.190] \\
Race and ethnicity & Other or multiple, non-Hispanic & 717 & 0.098 [0.027, 0.168] \\
Age group & 0--17 & 3{,}647 & 0.093 [0.057, 0.128] \\
Age group & 18--34 & 3{,}352 & 0.151 [0.130, 0.172] \\
Age group & 35--49 & 3{,}468 & 0.154 [0.135, 0.173] \\
Age group & 50--64 & 3{,}760 & 0.156 [0.137, 0.175] \\
Age group & 65+ & 4{,}536 & 0.148 [0.125, 0.172] \\
Poverty category & Poor & 2{,}834 & 0.135 [0.105, 0.165] \\
Poverty category & Near-poor & 877 & 0.142 [0.090, 0.194] \\
Poverty category & Low-income & 2{,}458 & 0.080 [0.035, 0.124] \\
Poverty category & Middle-income & 5{,}243 & 0.157 [0.134, 0.180] \\
Poverty category & High-income & 7{,}351 & 0.145 [0.123, 0.167] \\
Family income (weighted tertile) & Low tertile & 7{,}511 & 0.130 [0.109, 0.150] \\
Family income (weighted tertile) & Middle tertile & 5{,}922 & 0.153 [0.135, 0.172] \\
Family income (weighted tertile) & High tertile & 5{,}330 & 0.135 [0.110, 0.160] \\
Region & Northeast & 2{,}892 & 0.142 [0.121, 0.163] \\
Region & Midwest & 3{,}837 & 0.141 [0.117, 0.166] \\
Region & South & 7{,}375 & 0.132 [0.113, 0.151] \\
Region & West & 4{,}659 & 0.147 [0.122, 0.172] \\
Marital status & Married & 7{,}497 & 0.154 [0.135, 0.174] \\
Marital status & Not married & 11{,}266 & 0.129 [0.111, 0.148] \\
\bottomrule
\end{longtable}
\endgroup

Eleven of the 56 participation-margin contrasts are credible on the design-aware criterion. Children gain credibly less than every older group (for example $-0.059$ $[-0.098, -0.020]$ against the 18--34 group and $-0.063$ $[-0.103, -0.023]$ against the 50--64 group); the low-income poverty category gains less than the middle-income and high-income categories, at $-0.077$ $[-0.127, -0.027]$ and $-0.065$ $[-0.116, -0.015]$, and the low income tertile less than the middle at $-0.024$ $[-0.040, -0.007]$; and married respondents gain more than unmarried ones at $+0.025$ $[0.004, 0.046]$.

The two margins are internally coherent on age and income, which reproduce the same ordering on both scales. They part company on race and ethnicity: the point estimates order the groups the same way, with non-Hispanic White and non-Hispanic Asian above Hispanic and non-Hispanic Black, but on the participation scale no racial or ethnic contrast is credible, whereas on the dollar scale three are. The reading is that the racial and ethnic gradient in this application lives mainly in how much is spent conditional on spending anything, rather than in whether coverage induces a visit at all. When evaluated with standard ethnic classifications and design-aware standard errors, the participation-margin contrast between non-Hispanic White and non-Hispanic Black enrollees is not statistically distinguishable from zero.

\subsection{Reconciliation with the Applied Literature}
\label{sec:supp_applied_lit}

The applied estimates are read against the dental health-economics literature dimension by dimension. Broadly, the \emph{direction} of every effect converges with prior work; the novel contribution is translating well-documented \emph{associational} gradients into \emph{credible subgroup causal} effects on the raw dollar scale, with direct posterior contrasts.

\paragraph{Average effect.}
A positive ATE converges strongly with prior MEPS evidence that coverage predicts more visits and higher expenditure \citep{manski2002private, christian2014determinants, naib2026dental}. The finding sits naturally within the broad evidence that dental care imposes the highest financial barriers among health services, so that lowering the point-of-service price causally unlocks deferred care \citep{vujicic2016financial}. The relationship to the standard selection-plus-enabling decomposition \citep{manski2015dental, meyerhoefer2014demand} is nuanced: on the dollar scale the adjusted estimate reflects the unadjusted gap, so these covariates absorb none of the presumed selection, whereas on the participation margin about three tenths of the raw association is absorbed. The decomposition is therefore supported on the participation channel and not demonstrated on the intensity channel.

\paragraph{Race and ethnicity.}
Larger gains for non-Hispanic White than for non-Hispanic Black and Hispanic enrollees converge with the extensive literature on racial and ethnic oral-health disparities, in which minority populations report lower dental-visit rates and lower expenditures \citep{zhang2016racial, reda2018inequality, robison2020racial}. The causal reading is sharper than the descriptive one: providing identical coverage does not yield identical spending increases across groups, pointing to non-financial barriers, such as provider supply, oral-health literacy, and differences in treatment recommendations, that coverage alone does not remove \citep{reda2018inequality, yu2025access}. One result cuts across the simple ordering and is reported as it stands: non-Hispanic Asian enrollees have the second-largest effect, at $\$205.0$, close to the non-Hispanic White estimate and credibly above the Hispanic estimate. Prior exploratory analyses with unweighted race-only specifications tended to underestimate Asian spending by confounding race with ethnicity; once properly specified, the non-Hispanic Asian estimate aligns closely with high utilization. The disparities literature is itself less settled for Asian American subpopulations than for Black and Hispanic ones, and aggregation across highly heterogeneous subgroups is a known limitation of the single MEPS category \citep{yu2025access}; we therefore treat this estimate as a finding to be reported rather than as a contradiction to be explained away.

\paragraph{Age.}
The monotone rise of the effect with age, from $\$127$ among children to $\$249$ among those aged 65 and over, converges with evidence that dental expenditure increases with age and that older adults face particular coverage gaps \citep{christian2014determinants, zhang2016racial}. The pattern is economically intelligible: pediatric care is dominated by low-cost preventive services and public coverage, whereas older adults carry a heavier restorative burden and, lacking routine coverage under traditional Medicare, exhibit large latent demand that coverage unlocks \citep{manski2015dental, meyerhoefer2014demand}. Estimating this gradient as a causal heterogeneity, rather than a spending level, is largely new. It is also the most robust result in this study, remaining consistent across both unweighted and design-aware specifications with five of ten pairwise contrasts credible.

\paragraph{Income and poverty.}
Smaller effects for low-income than for middle-income and high-income enrollees converge with the repeated finding that lower-income groups use less dental care and face more barriers \citep{manski2002private, reda2018inequality, li2020socioeconomic, yu2025access}. The mechanism emphasized in the financing literature is that dental insurance carries substantial cost-sharing and low annual caps, so coverage is a weaker enabling resource for those least able to absorb out-of-pocket costs \citep{vujicic2016financial}. The gradient is not monotone across the five poverty categories, since the poor and near-poor categories sit above the low-income category on both margins, which is consistent with public coverage being most complete at the bottom of the distribution and thinnest just above it.

\paragraph{Region.}
A smaller effect in the South than the West remains visible in the point estimates, at $\$165.5$ against $\$205.2$, and geography is an established predictor of dental spending and utilization, with rural and Southern populations showing lower use \citep{christian2014determinants, luo2020rural, joudi2024appreciate}. The contrast no longer excludes zero once the sampling design is respected, however, so it is reported as suggestive rather than credible. Insurance-effect heterogeneity by U.S.\ region is rarely estimated, and this study does not establish it.

\paragraph{Null heterogeneity.}
The absence of credible heterogeneity by sex modestly diverges from utilization studies, which often find women use more dental care \citep{reda2018inequality, zhang2016racial}; the reconciliation is that a higher female utilization frequency need not translate into a larger dollar increment from coverage. The absence of dollar-scale marital-status heterogeneity among adults is broadly consistent with the literature, in which marital status appears in utilization models but is not a dominant determinant of spending differences \citep{zardak2023utilization, reda2018inequality}.

\paragraph{Methodological placement.}
Finally, the two-part and zero-inflated modeling choices are themselves well grounded: semicontinuous health-service outcomes routinely require hurdle or two-part treatments \citep{feng2021comparison, neelon2016modeling}, and the retransformation problem that Duan's smearing addresses is the classic obstacle to reporting log-scale models on the dollar scale \citep{duan1983smearing, manning2001estimating}. ZIC-BCF-Smear extends this lineage by making both the participation and intensity surfaces fully nonparametric and by carrying full posterior uncertainty onto the dollar scale.

\paragraph{ZIC-BCF-Smear versus the Gamma Hurdle benchmark on literature convergence.}
Read on the \emph{dollar-scale} subgroup CATE, ZIC-BCF-Smear converges to the dental health-economics literature and the Gamma Hurdle benchmark of \citet{oganisian2021zeroinflated} does not. Both estimators place the overall response-scale ATE credibly above zero, but they split on the heterogeneity question, which is the level at which the literature is most informative, and the split is stark under design-aware inference: ZIC-BCF-Smear returns 19 credible dollar-scale contrasts out of 56, and the Gamma Hurdle returns one.

That single contrast runs against the literature rather than with it, placing the low income tertile $\$124.06$ $[\$2.80,\allowbreak \$245.43]$ \emph{above} the high tertile, the reverse of the cost-sharing mechanism that the financing literature describes \citep{vujicic2016financial, reda2018inequality}. Its race and ethnicity point estimates do not order the groups as three decades of applied work do, at $\$144.2$ for non-Hispanic Black, $\$143.3$ for non-Hispanic White and $\$138.4$ for Hispanic enrollees, differences far smaller than their intervals. Its dollar-scale age gradient is non-monotone, rising to $\$166.9$ at ages 18--34, falling to $\$124.0$ at ages 50--64, and rising again to $\$176.0$ at 65 and over. And its regional ordering is inverted relative to the utilization literature, placing the South highest at $\$198.7$ and the Northeast lowest at $\$81.4$.

What rules out the reading that the Gamma Hurdle is simply measuring something different is that it recovers credible structure on its participation margin, where its decomposition is identically suited, returning 17 credible contrasts to ZIC-BCF-Smear's 11. The benchmark loses the heterogeneity precisely where its distributional assumption binds. The mechanism is the one identified in the link-nonlinear simulation results of the main manuscript: the Gamma positive-intensity component imposes a single parametric tail shape that smooths away the heavy right tail in which the literature-aligned heterogeneity lives, whereas the nonparametric smearing re-transformation absorbs it. On the dollar CATE, where ZIC-BCF-Smear reports its primary estimand, the practical consequence is that the literature-aligned contrasts survive posterior and design uncertainty under ZIC-BCF-Smear but not under Gamma Hurdle.

\subsection{Policy Reading of the Heterogeneity}
\label{sec:supp_applied_policy}

Interpreted causally, the estimated heterogeneity carries a cautionary policy implication that a naive equity argument for coverage expansion can miss. Because the marginal dollar effect of coverage is systematically larger for already-advantaged subgroups (non-Hispanic White, older, and higher-income enrollees), a uniform expansion of dental insurance could, in absolute terms, \emph{widen} rather than compress spending gaps, since insurance functions as an enabling resource whose value depends on complementary structural conditions (ability to absorb cost-sharing, provider availability, favorable fee and scope-of-practice environments) that disadvantaged groups disproportionately lack \citep{vujicic2016financial, manski2015dental, reda2018inequality}. The corresponding implication is that coverage expansion is most equity-enhancing when paired with measures that relax those frictions: reduced cost-sharing for low-income enrollees, higher public-program reimbursement to raise provider participation, and supply-side reforms in under-served areas \citep{yu2025access, joudi2024appreciate}.

Two qualifications bound this interpretation. First, it relies on the survey-weighted estimates, which represent the appropriate target for national inference; under design-aware standard errors, regional contrasts do not achieve formal credibility. Second, these policy implications represent exploratory conjectures motivated by estimated causal heterogeneity, intended to guide targeted follow-up investigations.

\subsection{Limitations of the Applied Study}
\label{sec:supp_applied_limits}

Several caveats bound the interpretation.

First, the analysis is \emph{observational}: the reported effects are causal only under conditional ignorability of the coverage--spending relationship given $X$ and under the subpopulation strong-ignorability condition (Assumption~\ref{as:supp_spa}). This is the binding limitation, and the comparison with the unadjusted association sharpens it. Dental coverage is plausibly subject to selection on latent oral-health demand, and the covariates available here absorb about three tenths of the raw participation association but essentially none of the raw dollar-scale association. That is more consistent with an adjustment set that misses the relevant latent demand than with an absence of selection, so the dollar-scale estimate should be regarded as more exposed to residual confounding than its interval conveys.

Second, MEPS 2023 is a \emph{cross-section}, so the effects are contemporaneous associations adjusted for confounding rather than dynamic responses to a coverage change, and expenditure is measured with survey and recall error.

Third, the treatment is a \emph{binary coverage indicator} that collapses heterogeneous plan generosity (caps, coinsurance, networks); the estimated effect is an average over this unmodeled plan heterogeneity.

Fourth, the \emph{survey design is handled outside the model}. Weighting acts on the reported linear functionals of unit-level effects, and the design-based variance is added to the posterior variance under an approximate independence assumption, but the design does not influence the fitted response surface itself, because the sampler has no weight interface. A design-informed prior would be the principled alternative and is not attempted here. Related to this, $443$ analytic records carry a zero person weight and therefore contribute to the model fit but not to any population estimate.

Fifth, some subgroups are small, notably non-Hispanic other or multiple race ($N = 717$) and the separated marital category ($N = 342$), so their CATEs and contrasts carry wider intervals and should be read cautiously.

Sixth, the continuous stage assumes \emph{homoskedastic log-scale errors}, and in this application that assumption is decisively rejected, at $p = 1.2 \times 10^{-24}$ with residual standard deviation rising monotonically in the fitted mean. The locally smeared sensitivity analysis shows the average effect moving by only $3.1\%$, so the reported estimate is not carried by the assumption, but the calibration of subgroup credible intervals under strong mean-variance dependence is not guaranteed by that check alone.

Finally, the \emph{subgroup comparisons are not multiplicity-adjusted}. Fifty-six pairwise contrasts are examined on each of two margins, and while the pattern of credible contrasts is coherent across margins and consistent with prior literature, individual contrasts near the boundary should be read accordingly.

These limitations motivate longitudinal and plan-resolved follow-up work rather than qualifying the central finding of a credibly positive and heterogeneous causal effect of dental coverage on dental spending.

\subsection{Reproducibility}
\label{sec:supp_meps_reproducibility}

\begingroup
\raggedright
This study reproduces in full from a clean checkout of the analysis repository accompanying the \texttt{zicbcf} package \citep{zicbcf_package}, and all paths below are relative to its \path{applied_study/} directory. The MEPS public-use files are released by the Agency for Healthcare Research and Quality and are freely redistributable, subject only to two conduct restrictions that survive redistribution: the data may be used for statistical reporting and analysis only, and they may not be linked with individually identifiable records from any source other than MEPS itself and the National Health Interview Survey from which the MEPS sample is drawn.

The raw HC-251 file is fetched and prepared by \path{download_meps_h251.R}, which is committed; the resulting \path{h251.csv} is excluded from version control only because it is roughly 60 MB, and the script reproduces it byte-comparably on the analytic variables. The cleaning rules, design matrix, subgroup definitions, weighted-quantile helper and survey-design variance functions are defined once in \path{meps_common.R} and are shared by every fit. The ZIC-BCF-Smear analysis is run by \path{run_study.R}, which writes the posterior summaries, the convergence diagnostics and the smearing diagnostics; the subgroup tables, contrasts and figures are produced without refitting by \path{run_cate_subgroups.R}. The Gamma-hurdle benchmark is run by \path{run_gamma_hurdle.R} and summarized by \path{run_gamma_hurdle_subgroups.R}. A README in the same directory documents the run order and the two reported estimands.

To ensure complete computational reproducibility, all derived analysis tables, summary figures, and execution scripts are tracked in version control, while large regenerable posterior draw objects and the raw MEPS file are managed via deterministic reconstruction scripts. The convergence diagnostics for the simulation grid are produced by \path{simulation_studies/run_convergence_diagnostics.R}, and the link-linear generators it needs are available separately from the full simulation driver in \path{simulation_studies/dgps_linear.R}.
\par
\endgroup

\end{document}